\documentclass[11pt]{article}
\usepackage[sort,round]{natbib}
\usepackage{microtype}
\usepackage{graphicx}
\usepackage{subfigure}
\usepackage{booktabs}
\usepackage{bm}
\usepackage{multicol,multirow}
\usepackage{comment}
\usepackage{booktabs}
\newcommand{\ie}{\textit{i.e.}}
\newcommand{\eg}{\textit{e.g.}}

\usepackage{enumerate}   
\usepackage{amsmath}
\usepackage{tikz}
\usetikzlibrary{positioning,arrows}
\usepackage{wrapfig}
\usepackage{subfigure}

\DeclareMathOperator*{\argmin}{arg\,min}

\usepackage{amsmath}
\usepackage{appendix}
\usepackage{amssymb}

\newcommand{\epol}{\pi^\mathrm{e}}
\newcommand{\var}{\mathrm{var}}
\newcommand{\bpol}{\pi^\mathrm{b}}
\usepackage{amsthm}
\usepackage{hyperref}
\usepackage[capitalise]{cleveref}
\Crefname{assumption}{Assumption}{Assumptions}

\theoremstyle{plain}
\newtheorem{theorem}{Theorem}
\newtheorem{lemma}[theorem]{Lemma}
\newtheorem{corollary}[theorem]{Corollary}
\theoremstyle{definition}
\newtheorem{assumption}{Assumption}
\newtheorem{definition}{Definition}
\newtheorem{remark}{Remark}

\newcommand{\bG}{\mathbb{G}}
\def\Holder{{H\"{o}lder}}

\newcommand{\cm}{\mathcal{M}}
\newcommand{\ch}{\mathcal{H}}
\newcommand{\bigO}{\mathrm{O}} %
\newcommand{\smallo}{\mathrm{o}} %

\renewcommand{\P}{\mathbb{P}}
\newcommand{\op}{\smallo_{p}}
\newcommand{\E}{\mathbb{E}}

\newcommand{\rE}{\mathrm{E}}

\newcommand{\prns}[1]{\left(#1\right)}
\newcommand{\braces}[1]{\left\{#1\right\}}
\newcommand{\bracks}[1]{\left[#1\right]}
\newcommand{\sumT}{\sum_{t=0}^T}
\newcommand{\abs}[1]{\left|#1\right|}
\newcommand{\Rl}{\mathbb{R}}

\newcommand{\RN}[1]{%
  \textup{\uppercase\expandafter{\romannumeral#1}}%
}

\usepackage{color,graphicx,lscape,longtable}

\def\boxit#1{\vbox{\hrule\hbox{\vrule\kern6pt\vbox{\kern6pt#1\kern6pt}\kern6pt\vrule}\hrule}}

\def\edit{}
\def\blockedit{}

\begin{document}

\title{
Double Reinforcement Learning for
Efficient\\Off-Policy Evaluation in Markov Decision Processes
}

\author{Nathan Kallus\\
       Department of Operations Research and Information Engineering and Cornell Tech\\
       Cornell University\\
       New York, NY 10044, USA
       \and
        Masatoshi Uehara\thanks{\url{uehara_m@g.harvard.edu}}\\
       Department of Statistics\\
       Harvard University\\
       Cambridge, MA 02138, USA}

\date{}
\maketitle

\begin{abstract}
Off-policy evaluation (OPE) in reinforcement learning allows one to evaluate novel decision policies without needing to conduct exploration, which is often costly or otherwise infeasible. 
We consider for the first time the semiparametric efficiency limits of OPE in Markov decision processes (MDPs), where actions, rewards, and states are memoryless.
We show existing OPE estimators may fail to be efficient in this setting. We develop a new estimator based on cross-fold estimation of $q$-functions and marginalized density ratios, which we term double reinforcement learning (DRL). We show that DRL is efficient when both components are estimated at fourth-root rates and is also doubly robust when only one component is consistent.
We investigate these properties empirically and demonstrate the performance benefits due to harnessing memorylessness.

\vspace{\baselineskip}\noindent\textbf{Keywords}: Off-policy evaluation, Markov decision processes, Semiparametric efficiency, Double machine learning
\end{abstract}

\section{Introduction}

Off-policy evaluation (OPE) is the problem of estimating mean rewards of a given policy (target policy) for a sequential decision-making problem using data generated by the log of another policy (behavior policy). OPE is a key problem in reinforcement learning (RL) \citep{Precup2000,Mahmood20014,Li2015,thomas2016,jiang,Munos2016,Liu2018} and it finds applications as varied as healthcare \citep{MurphyS.A.2003Odtr} and education \citep{Mandel2014}.
Because data can be scarce, it is crucial to use all available data efficiently, while at the same time using flexible, nonparametric estimators that avoid misspecification error.

\begin{figure}[t!]%
\centering%
\begin{tikzpicture}[%
>=latex',node distance=2cm, minimum height=0.75cm, minimum width=0.75cm,
state/.style={draw, shape=circle, draw=green, fill=green!10, line width=0.5pt},
action/.style={draw, shape=rectangle, draw=red, fill=red!10, line width=0.5pt},
reward/.style={draw, shape=rectangle, draw=blue, fill=blue!10, line width=0.5pt}
]
\node[state] (S0) at (0,0) {$s_0$};
\node[action,right of=S0] (A0) {$a_0$};
\node[reward,right of=A0] (R0) {$r_0$};
\node[state,right of=R0] (S1) {$s_1$};
\node[action,right of=S1] (A1) {$a_1$};
\node[reward,right of=A1] (R1) {$r_1$};
\node[state,right of=R1] (S2) {$s_2$};
\draw[->] (S0) -- (A0);
\draw[->] (S0) edge[bend left=30] (R0);
\draw[->] (A0) -- (R0);
\draw[->] (S0) edge[bend left=30] (S1);
\draw[->] (A0) edge[bend left=30] (S1);
\draw[->] (S1) -- (A1);
\draw[->] (S1) edge[bend left=30] (R1);
\draw[->] (A1) -- (R1);
\draw[->] (S1) edge[bend left=30] (S2);
\draw[->] (A1) edge[bend left=30] (S2);
\draw[->] (S0) edge[bend left=30] (A1);
\draw[->] (S0) edge[bend left=30] (R1);
\draw[->] (A0) edge[bend left=30] (R1);
\draw[->] (S0) edge[bend left=30] (S2);
\draw[->] (A0) edge[bend left=30] (S2);
\end{tikzpicture}\caption{$\cm_1$: Non-Markov decision process (NMDP)}\label{fig:m1}%
\vspace{4em}%
\begin{tikzpicture}[%
>=latex',node distance=2cm, minimum height=0.75cm, minimum width=0.75cm,
state/.style={draw, shape=circle, draw=green, fill=green!10, line width=0.5pt},
action/.style={draw, shape=rectangle, draw=red, fill=red!10, line width=0.5pt},
reward/.style={draw, shape=rectangle, draw=blue, fill=blue!10, line width=0.5pt}
]
\node[state] (S0) at (0,0) {$s_0$};
\node[action,right of=S0] (A0) {$a_0$};
\node[reward,right of=A0] (R0) {$r_0$};
\node[state,right of=R0] (S1) {$s_1$};
\node[action,right of=S1] (A1) {$a_1$};
\node[reward,right of=A1] (R1) {$r_1$};
\node[state,right of=R1] (S2) {$s_2$};
\draw[->] (S0) -- (A0);
\draw[->] (S0) edge[bend left=30] (R0);
\draw[->] (A0) -- (R0);
\draw[->] (S0) edge[bend left=30] (S1);
\draw[->] (A0) edge[bend left=30] (S1);
\draw[->] (S1) -- (A1);
\draw[->] (S1) edge[bend left=30] (R1);
\draw[->] (A1) -- (R1);
\draw[->] (S1) edge[bend left=30] (S2);
\draw[->] (A1) edge[bend left=30] (S2);
\end{tikzpicture}\caption{$\cm_2$: Markov decision process (MDP)}\label{fig:m2}%
\vspace{4em}%
\centering%
\begin{tabular}{ccccccc}
$O(T^2)$&
$=$&
$\operatorname{EffBd}(\cm_2)$&
$=$&
$\operatorname{EffBd}(\cm_{2b})$&
$>$&
$\operatorname{EffBd}(\cm_{2q})$
\\[.25em]
&&\rotatebox[origin=c]{90}{$>$}&&\rotatebox[origin=c]{90}{$>$}
\\[.25em]
$2^{\Omega(T)}$&
$=$&
$\operatorname{EffBd}(\cm_1)$&
$=$&
$\operatorname{EffBd}(\cm_{1b})$&
$>$&
$\operatorname{EffBd}(\cm_{1q})$
\end{tabular}
\caption{Relationship between the semiparametric efficiency bounds in each model, which lower bound achievable mean-squared error. $\cm_1$, $\cm_2$ are, respectively, NMDP and MDP with unknown behavior policy. $\cm_{1b}$, $\cm_{2b}$ are with known behavior policy. $\cm_{1q}$, $\cm_{2q}$ are with parametric assumptions on $q$-functions. Inequalities are generically strict (\cref{cor:jensen}), and the MDP bound is generally polynomial in horizon length $T$ while the NMDP bound is generally exponential (see \cref{thm:horizon}).}\label{fig:order}%
\end{figure}

In this paper, our goal is to obtain an estimator for policy value with minimal asymptotic mean squared error under nonparametric models for the sequential decision process and behavior policy, that is, achieving the semiparametric efficiency bound \citep{bickel98}. Toward that end, we explore the efficiency bound and efficient influence function of the target policy value under two models: non-Markov decision processes (NMDP) and Markov decision processes (MDP).
The two models are illustrated in \cref{fig:m1,fig:m2} and defined precisely in \cref{sec:setup}.
While much work has studied efficient estimation under $\cm_1$ \citep{jiang,thomas2016,DudikMiroslav2014DRPE,Kallus2019IntrinsicallyES}, work on $\cm_2$ has been restricted to the parametric, finite-state-finite-action case \citep{jiang} and no globally efficient estimators have been proposed.
The two models are clearly nested and indeed we obtain that the efficiency bounds are generally strictly ordered (see \cref{fig:order}). In other words, if we correctly leverage the Markov property, we can obtain OPE estimators that are \emph{more efficient} than existing ones. This is quite important, given the practical difficulty of evaluation in long horizons \citep[see, \eg,][]{gottesman2019guidelines} and given that many RL problems are Markovian. \edit{In particular, our results show the NMDP efficiency bound is generally exponential in horizon length so that estimators that target the NMDP model \emph{necessarily} suffer from the curse of horizon, a phenomenon previously identified only for specific estimators. In contrast, the MDP efficiency bound, which we achieve, is generally polynomial in horizon.}

We propose the Double Reinforcement Learning (DRL) estimator, which is given by by cross-fold estimation and plug-in of the $q$- and density ratio functions into the efficient influence function for each model, which we derive for the first time here.
\edit{The name DRL is inspired by the Double Machine Learning estimation procedure of \citet{ChernozhukovVictor2018Dmlf}, which we leverage, and by our simultaneous use of two learning procedures: learning of $q$-functions and of density ratios.}
We show that DRL achieves the semiparametric efficiency bound globally even when these nuisances are estimated at slow fourth-root rates and without restricting to Donsker or bounded entropy classes, enabling the use of flexible machine learning method for the nuisance estimation in the spirit of \citet{ZhengWenjing2011CTME,ChernozhukovVictor2018Dmlf}.
Further, we show that DRL is consistent even if only some of the nuisances are consistently estimated, known as double robustness.
To the best of our knowledge, this is the first proposed estimator shown to be globally efficient for OPE in MDPs.

The organization of the paper is as follows. In \cref{sec:setup}, we define the OPE problem and our models. In \cref{sec:semiparam} we summarize semiparametric inference theory and in \cref{sec:litope} we review the literature on OPE. In \cref{sec:semiparambounds}, we calculate the efficient influence functions and efficiency bounds in each of our models. In \cref{sec:drl}, we propose the DRL estimator and prove its efficiency and robustness in each model, while also reviewing the inefficiency of other estimators. In \cref{sec:q-learning}, we discuss how to estimate $q$-functions in an off-policy manner to be used in DRL as well as the efficiency bound under parametric assumptions on the $q$-function. In \cref{sec:experiment}, we demonstrate the benefits of DRL empirically.

\subsection{Problem Setup}\label{sec:setup}

A (potentially) non-Markov decision process (NMDP) is given by a
sequence of state and action spaces $\mathcal S_t,\mathcal A_t$ for $t=0,1,\dots,T$, an initial state distribution $P_{s_0}(s_0)$, transition probabilities $P_{s_t}(s_t\mid \ch_{a_{t-1}})$ for $t=1,\dots,T$, and emission probabilities $P_{r_t}(r_t\mid \ch_{a_t})$ for $t=0,\dots,T$, where $\ch_{a_t}=(s_0,a_0,\dots,s_t,a_t)$ is the state-action history up to $a_t$.
A (non-anticipatory) policy is a sequence of action probabilities $\pi_t(a_t\mid \ch_{s_t})$, where $\ch_{s_t}=(s_0,a_0,\dots,a_{t-1},s_t)$ is the state-action history up to $s_t$. Together, an NMDP and a policy define a joint distribution over trajectories $\ch=(s_0,a_0,r_0,s_1,a_1,r_1,\dots,s_T,a_T,r_T)$, given by the product $P_{s_0}(s_0)\pi_0(a_0\mid \ch_{s_0})P_{r_0}(r_0\mid \ch_{a_0})P_{s_1}(s_1\mid \ch_{a_0})\cdots P_{r_T}(r_T\mid \ch_{a_T})$. The dependence structure of such a distribution is illustrated in \cref{fig:m1}. We denote this distribution by $P_\pi$ and expectations in this distribution by $\rE_\pi$ to highlight the dependence on $\pi$.

A \edit{(time-varying)} Markov decision process (MDP) is an NMDP where transitions and emissions \edit{depend only on the recent state and action and the time index $t$}, $P_{s_t}(s_t\mid \ch_{a_{t-1}})=P_{s_t}(s_t\mid s_{t-1},a_{t-1})$ and $P_{r_t}(r_t\mid \ch_{a_{t}})=P_{r_t}(r_t\mid s_{t},a_{t})$, and where we restrict to policies that depend only on the recent state, $\pi_t(a_t\mid \ch_{s_t})=\pi_t(a_t\mid s_t)$. MDPs have the important property that they are memoryless: given $s_t$, the trajectory starting at $s_t$ is independent of the past trajectory, so that $s_t$ fully captures the current state of the system. This imposes a stricter dependence structure, which is illustrated in \cref{fig:m2}. In particular, connections between variables with different time indices occurs only via $s_t$.

Our ultimate goal is to estimate the average cumulative reward of a policy,
$$
\rho^\pi=\rE_\pi\bracks{\sumT r_t}.
$$
The quality and value functions ($q$- and $v$-functions) are defined as the following conditional averages of the cumulative reward to go, respectively:
\begin{align*}
q_t(\ch_{a_t})=\rE_\pi\bracks{\sum_{k=t}^T r_k\mid \ch_{a_t}},\qquad
v_t(\ch_{s_t})=\rE_\pi\bracks{\sum_{k=t}^T r_k\mid \ch_{s_t}}=\rE_\pi\bracks{q_t(\ch_{a_t})\mid \ch_{s_t}}.
\end{align*}
Note that the very last expectation is taken only over $a_t\sim\pi_t(a_t\mid \ch_{s_t})$.
For MDPs, we have $q_t(\ch_{a_t})=q_t(s_t,a_t)$ and $v_t(\ch_{s_t})=v_t(s_t)=\rE_{\pi}\bracks{q_t(s_t,a_t)\mid s_t}$, where again the last expectation is taken only over $a_t\sim\pi_t(a_t\mid {s_t})$. For brevity, we define the random variables $q_t=q_t(\ch_{a_t})$, $v_t=v_t(\ch_{s_t})$.

The off-policy evaluation (OPE) problem is to estimate the average cumulative reward of a given (known) target evaluation policy, $\pi^e$, using $n$ observations of trajectories $\mathcal{D}=\{\ch^{(1)},\dots,\ch^{(n)}\}$ independently generated by the distribution $P_{\pi^b}$ induced by using another policy, $\pi^{b}$, in the same decision process. This latter policy, $\pi^{b}$, is called the behavior policy and it may be known or unknown. 

A model for the data generating process $P_{\pi}$ of $\mathcal{D}$ is given by the set of products $P_{s_0}(s_0)\pi^b_0(a_0\mid \ch_{s_0})P_{r_0}(r_0\mid \ch_{a_0})P_{s_1}(s_1\mid \ch_{a_0})\cdots P_{r_T}(r_T\mid \ch_{a_T})$ over some possible values for each probability distribution in the product. 
We let $\cm_1$ denote the nonparametric model where each distribution is unknown and free. We let $\cm_{1b}$ denote the submodel of $\cm_1$ where $\pi^b$ is known and fixed. We let $\cm_{1q}$ denote any submodel of $\cm_1$ where the functions $q_t(\ch_{a_t})$ are restricted parametrically for $t=0,\dots,T$. We let $\cm_2,\,\cm_{2b},\,\cm_{2q}$ denote the corresponding models where both the decision process and the behavior policy are restricted to be Markovian.
Since $\pi^e$ is given,
the parameter of interest, $\rho^{\pi^e}$, is a function of just the part that specifies the decision process (initial state, transition, and emission probabilities).

To streamline notation, when no subscript is denoted, all  expectations $\mathrm{E}[\cdot]$ and variances $\mathrm{var}[\cdot]$ are taken with respect to the behavior policy, $\pi^b$. 
At the same time, all $v$- and $q$-functions are for the target policy, $\pi^e$.
The $L^{p}$-norm is defined as $\|g\|_{p}=\rE[\abs{g}^{p}]^{1/p}$. 
For any function of trajectories, we define its empirical average as $$\textstyle\mathrm{E}_{n}[f(\mathcal \ch)]=n^{-1}\sum_{i=1}^nf(\mathcal \ch^{(i)}).$$

We denote the density ratio at time $t$ between the target and behavior policy by
$$
\eta_t(\ch_{a_{t}})=\frac{\pi^{e}_{t}(a_t\mid\ch_{s_t})}{\pi^{b}_{t}(a_t\mid\ch_{s_t})}.
$$
We denote the cumulative density ratio up to time $t$ and the marginal density ratio at time $t$ by, respectively,
\begin{align*}
\lambda_t(\ch_{a_{t}})&=\prod_{k=0}^t\eta_t(\ch_{a_{k}}),\qquad
\mu_t(s_t,a_t)=\frac{p_{\pi^{e}_t}(s_t,a_t)}{p_{\pi^{b}_t}(s_t,a_t)},
\end{align*}
where $p_{\pi_t}(s_t,a_t)$ denotes the \emph{marginal} distribution of $s_t,a_t$ under $P_\pi$. Note that under $\cm_2$, $\eta_t(\ch_{a_t})=\eta_t(a_t,s_t)$. Again, for brevity we define the variables $\eta_t=\eta_t(\ch_{a_{t}})$, $\lambda_t=\lambda_t(\ch_{a_{t}})$, $\mu_t=\mu_t(s_t,a_t)$.

We \edit{will often} assume the following:
\begin{assumption}[Sequential overlap]\label{asm:overlap}
\edit {The density ratios $\eta_t,\mu_t$ satisfy $0\leq\eta_t\leq C,0\leq \mu_t\leq C'$ for all $t=0,\dots,T$.}
\end{assumption}
\begin{assumption}[Bounded rewards]\label{asm:bddrewards}
The reward $r_t$ satisfies $0 \leq r_t \leq R_{\mathrm{max}}$ for all $t=0,\dots,T$. 
\end{assumption}
\edit{\Cref{asm:overlap} requires that every action supported by the evaluation policy is also supported by the behavior policy, else the evaluation policy may induce state-action combinations that we cannot possibly ever see in the data. The assumption is standard in causal inference. \Cref{asm:bddrewards} focuses on bounded rewards, which are common in reinforcement learning. Both assumptions can be relaxed to $L_p$-norm bounds on the above variables instead of boundedness (see \cref{remark:normassumptions}).}

\subsection{Summary of Semiparametric Inference}\label{sec:semiparam}

We briefly review semiparametric inference theory as it pertains to the relevance of our results. 
\edit{We provide a more complete review in \cref{appendix:semiparam}, while providing an accessible casual introduction here sufficient for the reader to understand the nature of our efficiency results.}
For a complete textbook presentation, we refer the reader to
\edit{\cite{bickel98,TsiatisAnastasiosA2006STaM,KosorokMichaelR2008ItEP,VaartA.W.vander1998As,LaanMarkJ.vanDer2003UMfC}}.

Suppose we have a model $\mathcal{M}$ for the generating process of the iid data $\ch^{(1)},\dots,\ch^{(n)}$, that is, a (potentially nonparametric) set of possible distributions for $\ch^{(i)}$ that also contains the true distribution $F\in\mathcal M$ that generated the data. Consider a 
(scalar) parameter of interest $R:\mathcal M\to\Rl$. 
\edit{Given an estimator $\hat R$ (or rather a sequence of estimators), its limiting law is the distribution limit of $\sqrt{n}(\hat R-R(F))$, and the \emph{asymptotic mean-squared error} (AMSE) is the second moment of the limiting law, which in turn lower bounds the scaled limit of the mean-squared error (MSE), $\lim n\E[(\hat R-R(F))^2]$, by the portmanteau lemma.}

\edit{Every gradient of $R(\cdot)$ at $F\in\mathcal M$ (for paths in the model $\mathcal M$) is an $F$-measurable (scalar) random variable, that is, $\phi(\ch)$ with $\ch\sim F$ for some function $\phi(\cdot)$.
Each such function is called an \emph{influence function}, and the influence function $\phi_{\mathrm{eff}}(\cdot)$ with the smallest $L^2$ norm is 
is called the \emph{efficient influence function} because
$$\operatorname{EffBd}(\mathcal M)=\rE_{\ch\sim F}\bracks{\phi_{\mathrm{eff}}^2(\ch)}$$ bounds below the AMSE of \emph{any} estimator that is regular with respect to the model $\cm$.\footnote{\edit{Note that $\operatorname{EffBd}(\mathcal M)$ depends on the estimand $R(\cdot)$, the model $\mathcal M$, and the instance $F$ in that model. We emphasize foremost the dependence on the model to highlight the differences when we change the model from NMDP to MDP, while our estimand is always the target policy value.}}
Regular estimators are roughly those that have risk that is invariant to vanishing perturbations to the data generating process $F$ (that remain inside the model $\mathcal M$), which is a desirable property else the estimator may be unreasonably sensitive to undetectable changes.\footnote{\edit{See \cref{appendix:semiparam} and \citealp[Ch.~25]{VaartA.W.vander1998As} for precise definitions of path derivatives and regular estimators.}}
Essentially, regular estimators with respect to $\mathcal M$ are those that would work for estimating $R(F)$ for \emph{any} instance $F\in\mathcal M$. Thus, this lower bound applies \emph{per-instance} for any estimator that, in a sense, works in the model $\mathcal M$. Note we have $\operatorname{EffBd}(\mathcal M')\leq \operatorname{EffBd}(\mathcal M)$ whenever $F\in\mathcal M'\subset\mathcal M$, and that these may be different even though $F\in\mathcal M'$, as the set of estimators that work in $\mathcal M'$ is potentially larger. For example, the lower bound in the NMDP model is still larger than (or equal to) the bound in MDP model, even if considered at a specific instance that happens to be an MDP.}

\edit{There are several further interpretations of this lower bound. By the portmanteau lemma, the lower bound on limiting law also means that $\operatorname{EffBd}(\mathcal M)$ lower bounds the limit of the MSE for any regular $\hat R$, namely $$\liminf_{n\to\infty}n\E[(\hat R-R(F))^2]\geq \operatorname{EffBd}(\mathcal M).$$ Moreover, standard results \citep[\eg,][Thm.~25.21]{VaartA.W.vander1998As} establish that the lower bound also applies to \emph{all} estimators (not just regular ones) in a local minimax fashion: for \emph{any} estimator, $n$ times the worst-case MSE in a $1/\sqrt{n}$-sized $\mathcal M$-neighborhood around $F$ has a limit infimum of at least $\operatorname{EffBd}(\mathcal M)$. Here the ambient model $\mathcal M$ is relevant in determining the bound as the local worst-case neighborhoods are restricted to remain inside the model. 
When $\mathcal M$ is a fully parametric model the semiparametric efficiency bound is actually the same as the Cram\'er-Rao bound. In fact, the semiparametric efficiency bound corresponds to the supremum of the Cram\'er-Rao bounds over all regular parametric submodels $F\in\mathcal M_{\text{para}}\subset \mathcal M$. Thus, it also describes the best-achievable behavior by nonparametric estimators that work in \emph{every} parametric submodel.}

In these senses, $\operatorname{EffBd}(\mathcal M)$, known as the semiparametric efficiency bound, lower bounds the achievable MSE in estimating $R$ on the model $\mathcal M$. If we can find an estimator whose limiting law has zero mean and variance $\operatorname{EffBd}(\mathcal M)$ then it must have the smallest-possible (asymptotic) MSE, and such estimators are known as (asymptotically) \emph{efficient}. 
\edit{Moreover, all efficient regular estimators must satisfy}
\begin{align*}
\sqrt{n}(\hat{R}-R(F))=\frac{1}{\sqrt{n}}\sum_{i=1}^{n}\phi_{\mathrm{eff}}(\ch^{(i)})+\smallo_{p}(1),
\end{align*}
that is, they are asymptotically linear with \edit{efficient} influence function $\phi_{\mathrm{eff}}$.
This suggests an estimation strategy: try to approximate $\hat\psi(\ch)\approx\phi_{\mathrm{eff}}(\ch)+R(F)$ and use $\hat R=\frac1n\sum_{i=1}^n\hat\psi(\ch^{(i)})$. Done appropriately, this can provide an efficient estimate. Therefore, deriving the efficient influence function is important both for computing the semiparametric efficiency bound and for coming up with good estimators.

\subsection{Summary of Literature on OPE}\label{sec:litope}

OPE is a central problem in both RL and in closely related problems such as dynamic treatment regimes \citep[DTR;][]{MurphySA2001MMMf}. While the NMDP model $\cm_1$ is commonly the one assumed in the causal inference literature in the context of marginal structural model estimation \citep{robins2000marginal,robins2000marginalb} and DTRs \citep{MurphySA2001MMMf,chakraborty2013statistical,ZhangBaqun2013Reoo},\footnote{\edit{OPE is equivalent to estimating the total treatment effect of a DTR in a causal inference setting. Although we do not explicitly use counterfactual notation (either potential outcomes or \emph{do}-calculus), if we assume the usual sequential ignorability conditions \citep{MurphySA2001MMMf,ErtefaieAshkan2014CDTR,LuckettDanielJ.2018EDTR}, the estimands we consider are the same and our results immediately apply.}} in RL one often assumes that the MDP model $\cm_2$ holds. Nonetheless, with some exceptions that we review below, OPE methods in RL have largely not leveraged the additional independence structure of $\cm_2$ to improve estimation, and in particular, the effect of this structure on efficiency has not previously been studied and no efficient evaluation method has been proposed.

Methods for OPE can be roughly categorized into three types. The first approach is the \emph{direct method} (DM), wherein we directly estimate the $q$-function and use it to directly estimate the value of the target evaluation policy. \edit{For example, one can use model-based estimates \citep{mannor2007bias} or estimate the $q$-function directly
using fitted LSTDQ or more general $q$-iteration \citep{LagoudakisMichail2004LPI,antos2008learning, le2019batch} (we further review estimation of $q$-functions in \cref{sec:q-learning}}. Once we have an estimate $\hat q_0$ of the first $q$-function, the DM estimate is simply $\hat\rho^{\pi^e}_{\mathrm{DM}}=\rE_n\bracks{\rE_{\pi^e}\bracks{\hat q_0(s_0,a_0)\mid s_0}}$, where the inner expectation is simply over $a_0\sim\pi^e(\cdot\mid s_0)$ and is thus computable as a sum or integral over a known measure and the outer expectation is simply an average over the $n$ observations of $s_0$. \edit{Recall we define all $q$-functions to be with respect to $\pi^e$.} For DM, we can leverage the structure of $\cm_2$ by simply restricting $q$-functions to be Markovian. However, DM can fail to be efficient even under $\cm_1$ unless $q$-functions are parametric (and correctly specified) or extremely smooth (as shown by \citealp{HahnJinyong1998OtRo} but only in the $T=0$ case). DM is also not robust in that, if $q$-functions are inconsistently estimated, the estimate will be inconsistent.

The second approach is \emph{importance sampling} (IS), which averages the data weighted by the density ratio of the evaluation and behavior policies. Given estimates $\hat{\lambda}_t$ of the cumulative density ratios (or, letting $\hat{\lambda}_t=\lambda_t$ if the behavior policy is known), the IS estimate is simply $\hat\rho^{\pi^e}_{\mathrm{IS}}=\rE_n\bracks{\sum_{t=0}^T\hat{\lambda}_tr_t}$. (An alternative but higher-variance IS estimator is $\rE_n\bracks{\hat{\lambda}_T\sum_{t=0}^Tr_t}$.) When behavior policy is known, IS is unbiased and consistent, but its variance tends to be large due to extreme weights. In particular, under $\cm_1$, IS with $\hat{\lambda}_t=\lambda_t$ is known to be inefficient \citep{hirano2003efficient}, which implies it must be inefficient under $\cm_2$ as well. A common variant of IS is the self-normalized estimate $\sum_{t=0}^T\frac{\rE_n\bracks{\hat{\lambda}_tr_t}}{\rE_n\bracks{\hat{\lambda}_t}}$ \citep{Swaminathan2015b}, which trades off some bias for variance but does not make IS efficient.

The third approach is the \emph{doubly robust} (DR) method, which combines DM and IS and is given by adding the estimated $q$-function as a control variate \citep{scharfstein99,DudikMiroslav2014DRPE,jiang}. 
The DR estimate has the form $\hat\rho^{\pi^e}_{\mathrm{DR}}=\rE_n\bracks{\sum_{t=0}^T\prns{\hat{\lambda}_t(r_t-\hat q_t)+\hat{\lambda}_{t-1}\rE_{\pi^e}\bracks{\hat q_t|s_t}}}$. 

DR is colloquially known to be efficient under $\cm_1$ but no precise result is available.
When state and action spaces are finite, the model $\cm_1$ is necessarily parametric, and, under this parametric model, \citet{jiang} study the Cram\'er-Rao lower bound and observe that an infeasible DR estimator that uses oracle nuisance values instead of estimates, $\hat q_t=q_t$ and $\hat \lambda_t=\lambda_t$, would achieve the bound. For completeness, we derive precisely the more general semiparametric efficiency bound under $\cm_1$ (\cref{thm:fin_nnonpara}) and show that two (feasible) variants of the standard DR estimate are semiparametrically efficient, either using sample splitting with a rate condition (\cref{thm:double_m1}) or without sample splitting with a Donsker condition (\cref{thm:double_m1_sam}). \citet{jiang} also study parametric Cram\'er-Rao lower bounds under finite action and state space in the MDP model $\cm_2$, but no efficient estimator, whether parametric or nonparametric, has been proposed.
See also \cref{remark:jianNMDPbound,remark:jianMDPbound}.
There is a significant gap to deriving the semiparametric bound, which generalizes these results to more general action and state spaces and nonparametric models. More importantly, our derivation yields the efficient influence function, which provides a way to construct an efficient estimator under $\cm_2$.

Many variations of DR have been proposed. \citet{thomas2016} propose both a self-normalized variant of DR and a variant blending DR with DM when density ratios are extreme. \citet{Chow2018} propose to optimize the choice of $\hat q_t$ to minimize a variance estimate for DR rather than use a plug-in value. \citet{Kallus2019IntrinsicallyES} propose a DR estimator that achieves local efficiency, has certain stability properties enjoyed by self-normalized IS, and at the same time is guaranteed to have asymptotic MSE that is never worse than both DR, IS, and self-normalized IS.

However, all of the aforementioned IS and DR estimators do not exploit MDP structure and, in particular, will \emph{fail} to be efficient under $\cm_2$. 
\edit{Recently, in the finite-state-space setting \citet{XieTengyang2019OOEf} studied an IS-type estimator that exploits MDP structure by replacing density ratios with marginalized density ratios, estimated by a recursive formula.}
However, this estimator is also {not} efficient, even in the finite tabular setting. Remark 4 of \citet{XieTengyang2019OOEf} points out the inefficiency of the estimator proposed therein.

\section{Semiparametric Inference for Off-Policy Evaluation}\label{sec:semiparambounds}

In this section, we derive the efficiency bounds and efficient influence functions for $\rho^{\pi^{e}}$ under the models $\cm_{1}$, $\cm_{1b}$, $\cm_{2}$, and $\cm_{2b}$. Recall that the former two models are NMDP and the latter two are MDP and that the second and fourth assume a known behavior policy.

\subsection{Semiparametric Efficiency in Non-Markov Decision Processes}
\label{sec:m1}

First, we consider the NMDP models $\cm_{1}$ and $\cm_{1b}$.
We do this mostly for the sake of completeness since, while the influence function we derive below for the NMDP model appears as a central object in the structure of various previously proposed doubly-robust OPE estimators for RL \citep[\eg, among others,][]{jiang,DudikMiroslav2014DRPE,Chow2018,Kallus2019IntrinsicallyES,thomas2016}, \edit{these do not establish it as the efficient influence function in the NMDP model or derive the semiparametric efficiency bound, with the exception of the concurrent \citet{pmlr-v97-bibaut19a}. We note that in contrast, the influence function we derive for the MDP model in the next section appears to be novel and leads to new, more efficient estimators.}
\begin{theorem}[Efficiency bound under $\cm_1$]
\label{thm:fin_nnonpara}
The efficient influence function of $\rho^{\pi^{e}}$ under $\cm_1$ is
\begin{align}
\label{eq:m1_nonpara}
\phi^{\cm_1}_{\mathrm{eff}}(\ch)=-\rho^{\pi^{e}}+
\sum_{t=0}^{T}\prns{\lambda_{t}\prns{r_t-q_{t}}+\lambda_{t-1} 
v_t
}.
\end{align}

The semiparametric efficiency bound under $\cm_1$ is 
\begin{align}
\label{eq:m1_bound}
\operatorname{EffBd}(\cm_1)=
    \mathrm{var}(v_0)+\sum_{t=0}^{T}\mathrm{E}\bracks{\lambda^{2}_{t} \mathrm{var}\prns{r_{t}+v_{t+1}\mid\ch_{a_{t}}}}
    , 
\end{align}
where $v_{T+1}=0$.

Under $\mathcal{M}_{1b}$, the efficient influence function and bound are the same. 
\end{theorem}
\edit{Note that we do not assume \cref{asm:overlap,asm:bddrewards} in the above. The quantity $\operatorname{EffBd}(\cm_1)$ may or may not be finite. An infinite efficiency bound would imply the impossibility of consistent $\sqrt{n}$ estimation. Below in \cref{thm:horizon} we show how to bound $\operatorname{EffBd}(\cm_1)$ under \cref{asm:overlap,asm:bddrewards}.}

\begin{remark}
The efficient influence function and bound are both the same whether we know the behavior policy or not. Intuitively, this happens because the estimand $\rho^{\pi_e}$ does not in fact depend on behavior policy part of the data generating distribution, $P_{\pi^b}$, but only on the decision process parts (initial state, transition, and emission probabilities). 
This phenomenon mirrors the situation with knowledge of the propensity score in average treatment effect estimation in causal inference noted by \citet{HahnJinyong1998OtRo}.
\end{remark}

\begin{remark}\label{remark:jianNMDPbound}
When the action and state spaces are discrete, $\cm_1$ is necessarily a parametric model. In this discrete-space parametric model and with $r_t=0$ for $t\leq T-1$, Theorem~2 of \citet{jiang} derives the Cram\'er-Rao lower bound for estimating $\rho^{\pi^{e}}$. Because the semiparametric efficiency bound is the same as the Cram\'er-Rao lower bound for parametric models, the bound coincides with ours in this special discrete setting. 
\edit{\Cref{thm:fin_nnonpara} and the related result in \citet{pmlr-v97-bibaut19a}} are more general, establishing the limit on estimation in non-discrete, nonparametric settings and, moreover, establishes that the efficient influence function coincides with the structure of many doubly-robust OPE estimators used in RL (see references above).
\end{remark}

\begin{remark}
The efficient influence function $\phi_{\mathrm{eff}}^{\cm_1}$ has the oft-noted doubly robust structure. Specifically, 
\begin{align*}
\rho^{\pi^{e}}+\rE\bracks{\phi^{\cm_1}_{\mathrm{eff}}(\ch)}
&=\underbrace{\rE\left [\sum_{t=0}^{T}\lambda_{t}r_t\right]}_{=\rho^{\pi^{e}}}+\underbrace{\rE\left[\sum_{t=0}^{T}\prns{-\lambda_{t} q_t+\lambda_{t-1}v_{t}} \right]}_{=0}\\
&=  \underbrace{\rE\bracks{v_0}}_{=\rho^{\pi^{e}}} +\underbrace{\rE\bracks{\sum_{t=0}^{T}\lambda_{t}(r_t- q_t+v_{t+1} )}}_{=0}. 
\end{align*}
The first term in each line corresponds to a sequential IS estimator and direct method (DM), respectively. The second term in each line is a control variate, which remain mean zero even if we plug in different (\ie, wrong) $q$- and $v$-functions or density ratios, respectively.
In this sense, it is sufficient to estimate only one part of these for consistent OPE. We will leverage this in \cref{thm:doublerobust_m1} to achieve double robustness for DRL.
\end{remark}

\subsection{Semiparametric Efficiency in Markov Decision Processes}
\label{sec:m2}

Next, we derive the efficiency bound and efficient influence function for $\rho^{\pi^{e}}$ under the models $\cm_{2}$ and $\cm_{2b}$, \ie, when restricting to MDP structure. To our knowledge, not only have these never before been derived, the influence function we derive has also not appeared in any existing OPE estimators in RL. We recall that under $\cm_2$, we have $q_t=q_t(s_t,a_t)$ and $v_t=v_t(s_t)$.

\begin{theorem}[Efficiency bound under $\cm_2$]\label{thm:fin_nnonpara2}
The efficient influence function of $\rho^{\pi^{e}}$ under $\cm_2$ is 
\begin{align}
\label{eq:m2_eff}
 \phi^{\cm_2}_{\mathrm{eff}}(\ch)=-\rho^{\pi^{e}}+
 \sum_{t=0}^{T}\prns{\mu_t\prns{r_t-q_{t}}+\mu_{t-1}
 v_t}.
\end{align}

The semiparametric efficiency bound under $\cm_2$ is 
\begin{align}
\label{eq:m2_nonpara}
\operatorname{EffBd}(\cm_2)=
    \mathrm{var}(v_0)+\sum_{t=0}^{T}\mathrm{E}\bracks{\mu^{2}_{t} \mathrm{var}\prns{r_{t}+v_{t+1}\mid s_t,a_t}}
    , 
\end{align}
where $v_{T+1}=0$.

Under $\mathcal{M}_{2b}$, the efficient influence function and bound are the same.
\end{theorem}
\edit{Again, we do not assume \cref{asm:overlap,asm:bddrewards} in the above. Below in \cref{thm:horizon} we show how to bound $\operatorname{EffBd}(\cm_2)$ under \cref{asm:overlap,asm:bddrewards}.}

\begin{remark}\label{remark:jianMDPbound}
Again, when the action and state spaces are discrete, $\cm_2$ is necessarily a parametric model. In this discrete-space parametric model and with $r_t=0$ for $t\leq T-1$, Theorem~3 of \citet{jiang} derives the Cram\'er-Rao lower bound, which must (and does) coincide with ours in this setting. Again, our result is more general, covering nonparametric models and estimators, and, importantly, derives the efficient influence function, which we will use to construct the first globally efficient estimator for $\rho^{\pi^{e}}$ under $\cm_2$.
\end{remark}

\begin{remark}\label{remark:diffinfluencefn}
The difference between the efficient influence functions in the NMDP and MDP models, $\phi_{\mathrm{eff}}^{\cm_1}$ and $\phi_{\mathrm{eff}}^{\cm_2}$, is that (a) the cumulative density ratio $\lambda_{t}$ is replaced with the marginalized density ratio $\mu_{t}$ and (b) that $q$- and $v$-functions only depend on recent state and action rather than full past trajectory. Note that the latter difference is slightly hidden in our notation: in $\phi_{\mathrm{eff}}^{\cm_1}$, $q_t$ refers to $q_t(\ch_{a_t})$, while in $\phi_{\mathrm{eff}}^{\cm_2}$, $q_t$ refers to the much simpler $q_t(s_t,a_t)$.

Although the efficient influence function in \cref{thm:fin_nnonpara2} is derived \textit{de-novo} in the proof, which is the most direct route to a rigorous derivation, we can also use the geometry of influence functions to understand the result relative to \cref{thm:fin_nnonpara}. The efficient influence function is always given by projecting the influence function of any regular asymptotic linear estimator onto the tangent space \citep[Thm.~4.3]{TsiatisAnastasiosA2006STaM}. Under $\cm_2$, the function $\phi^{\cm_1}_{\mathrm{eff}}(\ch)$ from \cref{thm:fin_nnonpara} can be shown to still be an influence function of some regular asymptotic linear estimator in $\cm_2$. Projecting it onto the tangent space in $\cm_2$, where we have imposed the independence of past and future trajectories given intermediate state, can be seen to exactly correspond to the above marginalization over the past trajectory, explaining this structure of $\phi^{\cm_2}_{\mathrm{eff}}(\ch)$.
\end{remark}

\begin{remark}
The efficient influence function $\phi^{\cm_2}_{\mathrm{eff}}(\ch)$ also has a doubly robust structure. Specifically,
\begin{align*}
\rho^{\pi^{e}}+\rE\bracks{\phi^{\cm_2}_{\mathrm{eff}}(\ch)}
&=\underbrace{\rE\left [\sum_{t=0}^{T}\mu_{t}r_t\right]}_{=\rho^{\pi^{e}}}+\underbrace{\rE\left[\sum_{t=0}^{T}\prns{-\mu_{t} q_t+\mu_{t-1}v_{t}} \right]}_{=0}\\
&=  \underbrace{\rE\bracks{v_0}}_{=\rho^{\pi^{e}}} +\underbrace{\rE\bracks{\sum_{t=0}^{T}\mu_{t}(r_t- q_t+v_{t+1} )}}_{=0}. 
\end{align*}
The first term on the first line corresponds to the MIS estimator \citep{XieTengyang2019OOEf}. The first term on the second line corresponds to the DM estimator. The second term on each line corresponds to control variate terms. 
We will leverage this in \cref{thm:doublerobust_m2} to achieve double robustness for DRL.
\end{remark}

By comparing the efficiency bounds of \cref{thm:fin_nnonpara} and \cref{thm:fin_nnonpara2} and using Jensen's inequality, we can see that the Markov assumption reduces the efficiency bound, usually strictly so. 
\begin{theorem}\label{cor:jensen}
If $P_{\pi^b}\in\cm_2$ (\ie, the underlying distribution is an MDP), then $$\operatorname{EffBd}(\cm_2)\leq \operatorname{EffBd}(\cm_1).$$ Moreover, the inequality is strict if there exists $t\leq T$ such that both $\lambda_{t-1}$ and $r_{t-1}+v_{t}$ are not constant given ${s_{t-1},a_{t-1}}$.
\end{theorem}

{\blockedit
Beyond being sorted, we can actually see that $\operatorname{EffBd}(\mathrm{MDP})$ is generally polynomial in $T$ while $\operatorname{EffBd}(\mathrm{NMDP})$ is generally exponential in $T$. This shows that the curse of horizon is \emph{inevitable} in NMDP. While previously it was just shown to be a limitation specifically of IS and DR estimators \citep{Liu2018,XieTengyang2019OOEf}, this result shows it \emph{must} plague any estimator that targets the NMDP model and that it is insurmountable without leveraging additional structure that further narrows the model. That $\operatorname{EffBd}(\mathrm{MDP})$ is generally polynomial in $T$ shows that we can potentially overcome this by efficiently leveraging MDP structure, which is exactly what our novel DRL estimator will do.
\begin{theorem}
Under \cref{asm:bddrewards,asm:overlap},
\label{thm:horizon}
\begin{align*}
\operatorname{EffBd}(\mathrm{MDP})&\leq {C'}R^2_{\max}(T+1)^2,\\\operatorname{EffBd}(\mathrm{NMDP})&\leq C^{T+1}R^2_{\max}(T+1)^2.
\end{align*}
If $\E_{\epol}[\log(\eta_t)] \geq C_{\min}$ and $\E_{\epol}\bracks{\log(\var(r_t+v_{t+1}\mid \ch_{a_t}))}\geq \log(V^2_{\min})$ then 
\begin{align*}
\operatorname{EffBd}(\mathrm{NMDP})&\geq C_{\min}^{T+1} V^2_{\min}.
\end{align*}
Note that $\E_{\epol}[\log(\eta_t)]=\E_{\epol}\bracks{{\operatorname{KL}}(\epol_t(\cdot\mid s_t)\,\vert\vert\,\bpol_t(\cdot\mid s_t))}$ is the KL divergence between the distributions over actions induced by $\epol$ and $\bpol$, averaged over the states visited by $\epol$, and that $\E_{\epol}\bracks{\log(\var(r_t+v_{t+1}\mid \ch_{a_t}))}=\E\bracks{\log(\var(r_t+v_{t+1}\mid \ch_{a_t}))}$.
\end{theorem}
The lower bound on $\operatorname{EffBd}(\mathrm{NMDP})$ shows that the curse of horizon is inevitable. The condition on $\eta_t$ simply means that the evaluation and behavior policies are not becoming arbitrarily similar as $t$ grows (on-policy evaluation does not suffer from curse of horizon). The condition on $r_t+v_{t+1}$ essentially ensures that rewards are not trivially constant. The upper bound on $\operatorname{EffBd}(\mathrm{MDP})$ shows that, in contrast, variance that is polynomial in $T$ is possible in the MDP model.
}
{\blockedit
\begin{remark}[Consistency of $\operatorname{EffBd}(\mathrm{NMDP})$ and $\operatorname{EffBd}(\mathrm{NMDP})$]
NMDP models may be trivially transformed into MDP models by letting the state variable be the whole history $\ch_{s_t}$. Then, the trajectory becomes $\{\ch_{s_0},a_0,r_0,\ch_{s_1},a_1,r_1,\cdots\}$ and the efficiency bound under this transformed MDP matches the efficiency bound under the original NMDP since 
$$\mu_t(\ch_{s_t},a_t)=\frac{p_{\epol_t}(\ch_{s_t},a_t)}{p_{\bpol_t}(\ch_{s_t},a_t)}=\prod_{k=0}^{t}\eta_k(\ch_{a_k})=\lambda_t(\ch_{a_t}).$$
\end{remark}
}

\section{Efficient Estimation Using Double Reinforcement Learning}\label{sec:drl}

In this section, we construct the DRL estimator and then study its properties in the various models. In particular, we show that DRL is globally efficient under very mild assumptions. In the NMDP model, these assumptions are generally weaker than needed for efficiency of previous estimators. In the MDP model, this provides the first globally efficiency estimator for OPE. We further show that DRL enjoys certain double robustness properties when some nuisances are inconsistently estimated.

DRL is a meta-estimator; it takes in as input estimators for $q$-functions and density ratios and combines them in a particular manner that ensures efficiency even when the input estimators may not be well behaved. This is achieved by following the cross-fold sample-splitting strategy developed by \edit{\citet{1987CEot,ZhengWenjing2011CTME,ChernozhukovVictor2018Dmlf}}. We proceed by presenting DRL and its properties in each setting (NMDP and MDP). \edit{In the NMDP setting, DRL amounts to the cross-fold version of the RL OPE doubly robust estimator, which was proposed in the experiments of \citet[Section 6.1]{jiang} but not analyzed.} In the MDP setting, DRL is the first semiparametrically efficient and doubly robust estimator.

\edit{Throughout this section we assume that \cref{asm:overlap,asm:bddrewards} hold.}

\subsection{Double Reinforcement Learning for NMDPs}

Given a learning algorithm to estimate the $q$-function $q(\ch_{a_t})$ and cumulative density ratio function $\lambda_t(\ch_{a_t})$,
DRL with \edit{$K$-fold sample splitting ($K\geq2$)}
for NMDPs proceeds as follows:
\edit{\begin{enumerate}
\item Randomly permute the data indices and let $\mathcal D_j=\{\lceil{(j-1)n/K}\rceil+1,\dots,\lceil{jn/K}\rceil\}$ for $j=1,\dots,K$. Let $j_i$ be the fold containing observation $i$ so that $i\in\mathcal D_{j_i}$ (namely, $j_i=1+\lfloor(i-1)K/n\rfloor$).
\item For $j=1,\dots,K$,
construct estimators $\hat{\lambda}^{(j)}_t(\ch_{a_t})$ and $\hat q^{(j)}_t(\ch_{a_t})$ based on the training data given by all trajectories excluding those in $\mathcal D_{j}$, that is, $\{1,\dots,n\}\backslash\mathcal D_{j}$.
\item Let 
\begin{align*}
\hat\rho^{\pi^{e}}_{\mathrm{DRL}(\cm_1)}=
\frac1n\sum_{i=1}^n
\sum_{t=0}^{T}\biggl(&\hat \lambda^{(j_i)}_t(\ch_{a_t}^{(i)})\prns{r_t^{(i)}-\hat q^{(j_i)}_{t}(\ch_{a_t}^{(i)})}\\&+\hat \lambda^{(j_i)}_{t-1}(\ch_{a_{t-1}}^{(i)})\int_{a'_t} \hat q^{(j_i)}_{t}((\ch_{s_t}^{(i)},a'_t)) d\pi^e_t(a'_t\mid \ch_{s_t}^{(i)})
\biggr)
.
\end{align*}
Here $(\ch_{s_t}^{(i)},a'_t)$ represents the trajectory given by appending $a'_t$ to $\ch_{s_t}^{(i)}$; note this differs from $\ch_{a_t}^{(i)}$.
Note further that the integral becomes a simple sum when $\pi^e$ has finite support over actions (\eg, if $\epol$ is deterministic or if there are finitely many actions).
\end{enumerate}}

In other words, we approximate the efficient influence function $\phi^{\cm_1}_{\mathrm{eff}}(\ch)+\rho^{\pi^e}$ from \cref{thm:fin_nnonpara} by replacing the unknown $q$- and density ratio functions with estimates thereof and we take empirical averages of this approximation over the data, where for each data point we use $q$- and density ratio function estimates based only on the half-sample that does \emph{not} contain the data point.
\edit{While $K=2$ is sufficient to achieve efficiency, larger $K$ allows nuisances to be fit on more data and may prove practically successful.}
Note also that we take only \edit{a single random $K$-fold split} and that is enough to achieve our results below. In practice, repeating the above process over several splits of the data and taking the average of resulting DRL estimates can only reduce the variance without increasing bias.

This estimator has several desirable properties. To state them, we assume the following 
conditions for the estimators, reflecting \cref{asm:overlap,asm:bddrewards}:
\begin{assumption}[Bounded estimators]\label{asm:boundedestim1}
$0\leq \hat{\lambda}^{(j)}_{t}\leq C^{t+1}$, $0\leq \hat{q}^{(j)}_{t}\leq (T+1-t)R_{\mathrm{max}}$ for all $0\leq t \leq T$, $1\leq j\leq K$. 
\end{assumption}
\edit{\Cref{asm:boundedestim1} provides that the same bounds that apply to the true $\lambda_t,q_t$ due to \cref{asm:bddrewards,asm:overlap} also apply to their estimates, as will necessarily be the case for many non-parametric estimators such as random forests and kernel regression. For the rest of this subsection we will assume \cref{asm:boundedestim1} hold.}

We first prove that DRL achieves the semiparametric efficiency bound, even if each nuisance estimator has a slow, nonparametric convergence rate (sub-$\sqrt{n}$).
{\blockedit
\begin{theorem}[Efficiency of $\hat{\rho}_{\mathrm{DRL}(\cm_1)}$ under $\cm_1$]
\label{thm:double_m1}
Suppose %
$\|\hat{\lambda}^{(j)}_{t}-\lambda_{t}\|_2\|\hat{q}^{(j)}_t-q_t\|_2=\op(n^{-1/2}),\|\hat{\lambda}^{(j)}_{t}-\lambda_{t}\|_2=\op(1),\,\|\hat{q}^{(j)}_t-q_t\|_2=\op(1)$ for $0\leq t \leq T,1\leq j\leq K$. Then, the estimator $\hat{\rho}_{\mathrm{DRL}(\cm_1)}$ achieves the semiparametric efficiency bound under $\cm_1$.
\end{theorem}
\begin{remark}\label{remark:normassumptions}
\Cref{asm:bddrewards,asm:overlap,asm:boundedestim1} posited $L^\infty$ bounds on density ratios, rewards, and their estimates. These assumptions are standard in both reinforcement learning and causal inference. It is possible to relax these to $L^p$ bounds at the cost of requiring stronger convergence on nuisance estimates above, requiring $L^{2/(1-1/p)}$ convergence instead of the $L^2$ convergence above. Since $L^2$ convergence in estimation is usually the standard convergence mode considered and standard results can be invoked to ensure such rates, and similarly $L^\infty$ bounds on density ratios and rewards are also standard, we focus our analysis on this most common case of the assumptions in order to avoid cumbersome presentation.
\end{remark}}
\begin{remark}
There are two important points to make about this result. First, we have not assumed a Donsker condition \citep{VaartA.W.vander1998As} on the class of estimators $\hat{\lambda}_{t}$ and $\hat{q}_{t}$. This is why this type of sample splitting estimator is called a double machine learning: the only required condition is a convergence rate condition at a nonparametric rate, allowing the use of complex machine learning estimators, for which one cannot verify the Donsker condition \citep{ChernozhukovVictor2018Dmlf}. \edit{In fact, many adaptive or high-dimensional estimators fail to satisfy Donsker conditions \citep{DiazIvan2019Mlit}. Eschewing such conditions allows us to use such estimators as the highly adaptive LASSO \citep{BenkeserDavid2016THAL}, c\`adl\`ag function estimators in very high dimensions \citep{bibaut2019fast}, and random forests \citep{WagerStefan2016ACoR} as long as their convergence rates are ensured under certain conditions. \citet{BenkeserDavid2016THAL,bibaut2019fast} in particular establish $\smallo_p(n^{-1/4})$ rates, which are compatible with our assumptions.}
Second, relative to the efficient influence function, which is defined in terms of the true $q$-function and cumulative density ratio, there is no inflation in DRL's asymptotic variance due to plugging in \emph{estimated} nuisance functions. This is due to the doubly robust structure of efficient influence function so that the estimation errors multiply and drop out of the first-order variance terms. 
This is in contrast to inefficient importance sampling estimators, as we will see in \cref{thm:ipw_m2}.
\end{remark}

In addition to efficiency, we can also establish finite-sample guarantees for DRL, where the leading term is controlled by the efficient variance.

\begin{theorem}\label{thm:effcor_m1}
Suppose that for some $C_1,C_2$, for every $n$, with probability at least $1-\delta$, we simultaneously have that $\|\hat\lambda^{(j)}_t-\lambda_t\|^2_2\leq \kappa_1=C_1 (n^{-2\alpha_{1}}+\log (2KT/\delta)/n)$, $\|\hat q^{(j)}_t-q_t\|^2_2\leq\kappa_2=C_2 (n^{-2\alpha_{2}}+\log (2KT/\delta)/n)$ $\forall t\leq T,\,\forall j\leq K$. Then, for every $n$, with probability at least $1-7\delta$, we have 
\begin{align*}
      \abs{\hat{\rho}^{\pi_e}_{\operatorname{DRL}(\cm_1)}-\rho^{\pi_e}} \leq&~\sqrt{\frac{2\log(2/\delta)\mathrm{Effbd}(\cm_1)}n}\\ &~+Q_1\sqrt{\frac{\log(2/\delta)T^{2}(TR_{\mathrm{max}}C^{T+1}\sqrt{\kappa_1\kappa_2}+\kappa_1T^2R^2_{\mathrm{max}}+\kappa_2C^{2(T+1)})}{n}}\\ 
      &~+Q_2\frac{\log(2/\delta)TR_{\mathrm{max}}C^{T+1}}{n}+Q_3 T\sqrt{\kappa_1\kappa_2},
\end{align*}
where $Q_1,Q_2,Q_3$ are constants not depending on $\delta,T,R_{\mathrm{max}},C,n,C_1,C_2$. 
\end{theorem}
Notice that, if $\alpha_1>0,\alpha_2>0,\alpha_1+\alpha_2> 1/2$ as in \cref{thm:double_m1}, then the leading term in the above is exactly
$
    \sqrt{\frac{2\log(2/\delta)\mathrm{Effbd}(\cm_1)}n}$
    (with no additional constant factor),
while the other terms are of strictly smaller order in $n$. 

The rate assumptions in \cref{thm:effcor_m1} are standard finite-sample estimation guarantees.
For example, if estimators for nuisances are obtained by empirical risk minimization methods based on $L^{2}$-loss, then the results of \citet{BartlettPeterL.2005LRc} apply.
In this cases, the rates $\alpha_1$ and $\alpha_2$ would be determined by the local Rademacher complexity of the posited function classes. 
The number $2KT$ in $\log(2KT/\delta)$ comes from the fact that there are $2KT$ nuisance estimators. 

Often in RL, the behavior policy is known and need not be estimated. That is, we can let $\hat{\lambda}^{(j)}_t=\lambda$. In this case, as an immediate corollary of \cref{thm:double_m1}, we have a much weaker condition for semiparametric efficiency: just that we estimate the $q$-function consistently, \emph{without a rate}.
\begin{corollary}[Efficiency of $\hat{\rho}_{\mathrm{DRL}(\cm_1)}$ under $\cm_{1b}$]
\label{cor:double_m1_2}
Suppose $\hat{\lambda}_t=\lambda_t$ and $\|\hat{q}^{(j)}_t-q_t\|_2=\smallo_{p}(1)$ for $0\leq t \leq T,1\leq j \leq K$. Then, the estimator $\hat{\rho}_{\mathrm{DRL}(\cm_1)}$ achieves the semiparametric efficiency bound under $\cm_{1b}$.
\end{corollary}

Without sample splitting, we have to assume a Donsker condition for the class of estimators in order to control a stochastic equicontinutiy term \citep[see, \eg,][Lemma 19.24]{VaartA.W.vander1998As}. Although this is more restrictive,
for completeness, we also include a theorem establishing the semiparametric efficiency of the standard plug-in doubly robust estimator for NMDPs \citep{jiang} when assuming the Donsker condition for in-sample-estimated nuisance functions, since this result was never precisely established before. \edit{We do not recommend this estimator due to its restrictive requirements on nuisance estimators.}

{\blockedit
\begin{theorem}[Efficiency without sample splitting]\label{thm:double_m1_sam}
Let $\hat{\lambda}_t,\hat q_t$ be estimators based on $\mathcal D$ and let 
$\hat\rho^{\pi^{e}}_{\mathrm{DR}}=
\rE_{n}\bracks{
\sum_{t=0}^{T}\prns{\hat \lambda_t\prns{r_t-\hat q_{t}}-\hat \lambda_{t-1}\rE_{\pi^e}\bracks{\hat q_{t}\mid \ch_{s_t}}}}$.
Suppose $\|\hat{\lambda}_{t}-\lambda_{t}\|_2\|\hat{q}_t-q_t\|_2=\op(n^{-1/2}),\|\hat{\lambda}_{t}-\lambda_{t}\|_2=\op(1),\,\|\hat{q}_t-q_t\|_2=\op(1)$ for $0\leq t \leq T$ and that $\hat{q}_t,\,\hat{\lambda}_t$ belong to a Donsker class. Then, the estimator $\hat{\rho}^{\pi^e}_{\mathrm{DR}}$ achieves the semiparametric efficiency bound under $\cm_1$. 
\end{theorem}
}
Thus, in $\cm_1$, in comparison to the standard doubly robust estimator, DRL enjoys efficiency under milder conditions. To our knowledge, \cref{thm:double_m1,thm:double_m1_sam} are the first results precisely showing semiparametric efficiency for any OPE estimator.

In addition to efficiency, DRL enjoys a double robustness guarantee \citep{Rotnitzky14,RotnitzkyAndrea2019Copw}. Specifically, if at least just one model is correctly specified, then the DRL is estimator is still %
\edit{$\sqrt{n}$-consistent.}
{\blockedit
\begin{theorem}[Double robustness ($\sqrt{n}$-consistency)]
\label{thm:doublerobust_m1}
Suppose $\|\hat{\lambda}^{(j)}_{t}-\lambda^\dagger_{t}\|_2=\bigO_{p}(n^{-\alpha_{1}})$ and $\|\hat{q}^{(j)}_t-q^\dagger_t\|_2=\bigO_{p}(n^{-\alpha_{2}})$ for $0\leq t \leq T,1\leq j\leq K$. If, for each $0\leq t\leq T$, either $\lambda^{\dagger}_{t}=\lambda_t$ and $\alpha_{1}\geq 1/2,\,\alpha_{2}> 0$ or $q^{\dagger}_t=q_t$ and $\alpha_{2}\geq 1/2,\,\alpha_{1}> 0$, then the estimator $\hat{\rho}_{\mathrm{DRL}(\cm_1)}$ is {$\sqrt{n}$-consistent} around $\rho^{\pi^{e}}$. 
\end{theorem}
In particular, if the behavior policy is known so that $\hat{\lambda}^{(j)}_t=\lambda_t$, we can always ensure the estimator is $\sqrt{n}$-consistent (an example is the IS estimator, which has $\hat q^{(j)}_t=q^\dagger_t=0$).

For consistency without a rate, it is sufficient for one nuisance to be consistent without a rate.
\begin{corollary}[Double robustness (Consistency)]
\label{cor:db1}
\label{}
Suppose $\|\hat{\lambda}^{(j)}_{t}-\lambda^\dagger_{t}\|_2=\smallo_{p}(1)$ and $\|\hat{q}^{(j)}_t-q^\dagger_t\|_2=\smallo_{p}(1)$ for $0\leq t \leq T,1\leq j\leq K$. If, for each $0\leq t\leq T$, either $\lambda^{\dagger}_{t}=\lambda_t$ or $q^{\dagger}_t=q_t$, then the estimator $\hat{\rho}_{\mathrm{DRL}(\cm_1)}$ is consistent around $\rho^{\pi^{e}}$. 
\end{corollary}
}

A remaining question is when can we get nonparametric estimators achieving the necessary rates for the $q$- and density ratio functions. We discuss estimating $q$-functions in \cref{sec:q-learning}. \edit{Regarding the density ratio, $\lambda_k$, if the behavior policy is known then the density ratio is known and if the behavior policy is unknown it must be estimated. \emph{Any} estimator satisfying the slow rate conditions would suffice.}

\edit{For example, we may let $\hat{\lambda}^{(j)}_{k}=\prod_{t=0}^{k} \pi^{e}_{t}/\hat{\pi}^{b,(j)}_{t}$, where $\hat{\pi}^{b,(j)}_{t}$ is some nonparametric regression estimator. Then $\hat{\lambda}^{(j)}_{k}$ would enjoy the same rates as $\hat{\pi}^{b,(j)}_{t}$:
\begin{lemma}\label{lem:denstiy_estimation}
Suppose $\hat{\pi}^{b,(j)}_{t}$ and $\pi^{b}_{t}$ are uniformly bounded by some constant below and that $\|\hat{\pi}^{b,(j)}_{t}-\pi^{b}_{t}\|_2=\op(n^{-\alpha})$.
Then, $\|\hat{\lambda}^{(j)}_{t}-\lambda_{t}\|_{2}=\op(n^{-\alpha})$. 
\end{lemma}
Rates for $\hat{\pi}^{b,(j)}$ can be obtained from standard results for nonparametric regression, such as for kernel and sieve estimators \citep{newey94,StoneCharlesJ.1994RTUo} or nonparametric estimators suited for high dimensions and non-smooth models \citep{KhosraviKhashayar2019NIAt,BibautAurelien2019Frfe,ImaizumiMasaaki2018DNNL}.}

Alternatively, parametric models can be used for $q_{t}$ and (if behavior policy is unknown) $\lambda_{t}$. Then, under standard regularity conditions, using MLE and other parametric regression estimators for behavior policy would yield $\|\hat{\lambda}^{(j)}_{t}- \lambda^\dagger_{t} \|_2=\bigO_{p}(n^{-1/2})$, where $\lambda^\dagger_{t}=\lambda_{t}$ if the model is well-specified. Similarly, in \cref{sec:q-learning}, we discuss how using parametric $q$-models yields $\|\hat{q}^{(j)}_{t}- q^\dagger_{t}\|_2=\bigO_{p}(n^{-1/2})$. If both models are correctly specified then \cref{thm:double_m1} immediately implies DRL achieves the efficiency bound. When using parametric models, this is sometimes termed \emph{local efficiency} (\ie, local to the specific parametric model). If only one model is correctly specified then \cref{thm:doublerobust_m1} ensures the estimator is still \edit{$\sqrt{n}$-consistent}.

\subsection{Double Reinforcement Learning for MDPs}

Leveraging our derivation of the efficient influence function for $\cm_2$ in \cref{thm:fin_nnonpara2}, we can similarly construct our DRL estimator for MDPs.
Given a learning algorithm to estimate the $q$-function $q_t(s_t,a_t)$ and marginal density ratio function $\mu_t(s_t,a_t)$,
DRL with \edit{$K$-fold sample splitting ($K\geq2$)}
for MDPs proceeds as follows:
\edit{\begin{enumerate}
\item Randomly permute the data indices and let $\mathcal D_j=\{\lceil{(j-1)n/K}\rceil+1,\dots,\lceil{jn/K}\rceil\}$ for $j=1,\dots,K$. Let $j_i$ be the fold containing observation $i$ so that $i\in\mathcal D_{j_i}$ (namely, $j_i=1+\lfloor(i-1)K/n\rfloor$).
\item For $j=1,\dots,K$,
construct estimators $\hat{\mu}^{(j)}_t(s_t,a_t)$ and $\hat q^{(j)}_t(s_t,a_t)$ based on the training data given by all trajectories excluding those in $\mathcal D_{j}$, that is, $\{1,\dots,n\}\backslash\mathcal D_{j}$.
\item Let 
\begin{align*}
\hat\rho^{\pi^{e}}_{\mathrm{DRL}(\cm_2)}=
\frac1n\sum_{i=1}^n
\sum_{t=0}^{T}\biggl(&\hat \mu^{(j_i)}_t({s_t}^{(i)},{a_t}^{(i)})\prns{r_t^{(i)}-\hat q^{(j_i)}_{t}({s_t}^{(i)},{a_t}^{(i)})}\\&+\hat \mu^{(j_i)}_{t-1}({s_{t-1}}^{(i)},{a_{t-1}}^{(i)})\int_{a'_t} \hat q^{(j_i)}_{t}({s_t}^{(i)},a'_t) d\pi^e_t(a'_t\mid {s_t}^{(i)})
\biggr)
.
\end{align*}
Again, note the difference between the dummy $a'_t$ and the data $a_t^{(i)}$, and that the integral becomes a simple sum when $\pi^e$ has finite support over actions (\eg, if $\epol$ is deterministic or if there are finitely many actions).
\end{enumerate}}

Again, what we have done is approximate the efficient influence function $\phi^{\cm_2}_{\mathrm{eff}}(\ch)+\rho^{\pi^e}$ from \cref{thm:fin_nnonpara2} and taken its empirical average over the data, where for each data point we use $q$- and marginal density ratio function estimates based only on the half-sample that does \emph{not} contain the data point. Again, taking one split suffices for our results, but one can repeat the above over many splits and take averages without deterioration.

Since both estimators are approximating their respective influence function as we derive in \cref{thm:fin_nnonpara,thm:fin_nnonpara2}, the differences between $\hat\rho^{\pi^{e}}_{\mathrm{DRL}(\cm_1)}$ and $\hat\rho^{\pi^{e}}_{\mathrm{DRL}(\cm_2)}$, as noted in \cref{remark:diffinfluencefn}, is (a) $\lambda_{t}$ is replaced with $\mu_{t}$ and (b) $q$- and $v$-functions only depend on recent state and action rather than full past trajectory. Again notice that in $\hat\rho^{\pi^{e}}_{\mathrm{DRL}(\cm_1)}$, $\hat q^{(j)}_t$ refers to $\hat q^{(j)}_t(\ch_{a_t})$, while in $\hat\rho^{\pi^{e}}_{\mathrm{DRL}(\cm_2)}$, $\hat q^{(j)}_t$ refers to the much simpler $\hat q^{(j)}_t(s_t,a_t)$.

\edit{Again, to establish the properties of DRL for MDPs,
we assume the following conditions
 reflecting \cref{asm:overlap,asm:bddrewards}:
\begin{assumption}[Bounded estimators]\label{asm:boundedestim2}
$0\leq \hat{\mu}^{(j)}_{t}\leq C'$, $0\leq \hat{q}^{(j)}_{t}\leq (T+1-t)R_{\mathrm{max}}$ for all $0\leq t \leq T$, $1\leq j\leq K$. 
\end{assumption}
For the rest of this subsection we will assume \cref{asm:boundedestim2} hold.}

The following result establishes that DRL is the first efficient OPE estimator for MDPs. In fact, it is efficient even if each nuisance estimator has a slow, nonparametric convergence rate (sub-$\sqrt{n}$).
Moreover, as before, we make no restrictive Donsker assumption; the only required condition is the convergence rate condition..
This result leverages our novel derivation of the efficient influence function in \cref{thm:fin_nnonpara2} and the structure of the influence function, which ensures no variance inflation due to estimating the nuisance functions.  

{\blockedit
\begin{theorem}[Efficiency of $\hat{\rho}_{\mathrm{DRL}(\cm_2)}$ under $\cm_2$]
\label{thm:double_m2}
Suppose $\|\hat{\mu}^{(j)}_{t}-\mu_{t}\|_2\|\hat{q}^{(j)}_t-q_t\|_2=\op(n^{-1/2}),\|\hat{\mu}^{(j)}_{t}-\mu_{t}\|_2=\op(1),\,\|\hat{q}^{(j)}_t-q_t\|_2=\op(1)$ for $0\leq t \leq T,1\leq j\leq K$. Then, the estimator $\hat{\rho}_{\mathrm{DRL}(\cm_2)}$ achieves the semiparametric efficiency bound under $\cm_2$.
\end{theorem}}
\edit{\begin{remark}[Example cases]\label{remark:cases}
We note a few specific cases of \cref{thm:double_m2}.
\begin{itemize}
\item \emph{Tabular case}: Suppose the state and action spaces are finite: $\abs{\mathcal S_t},\abs{\mathcal A_t}<\infty$. Then both $\mu_t$ and $q_t$ are \emph{parametric} functions with parameters given by their values at each $(s_t,a_t)$ pair. They can therefore be easily estimated at $\bigO_p(n^{-1/2})$ rates, ensuring the above rate conditions are easily satisfied. 
For example, we can use simple frequency estimators that simply take sample averages within each $(s_t,a_t)$ bin \citep[Chapter 3]{LiQi2007Ne:t}. Other examples and additional detail are given in \cref{sec:mis,sec:q-learning}.
\item \emph{Finite state space, known behavior policy}: Suppose now only $\abs{\mathcal S_t}<\infty$ while $\mathcal A_t$ can be continuous and that $\eta_t$ is known. Then $\mu_t(s_t,a_t)=\eta_t(s_t,a_t)w_t(s_t)$ and $w_t(s_t)=\E[\lambda_{t-1}\mid s_t]$ is a parametric function easily estimated at $\bigO_p(n^{-1/2})$ rates using a frequency estimator or using a recursive estimator as in \citet{XieTengyang2019OOEf} (more detail in \cref{sec:mis}). It therefore suffices for $q$-function estimators to have errors that are $\smallo_p(1)$, \ie, only consistency is needed without a rate. Since $\sum_{k=t}^Tr_k$ has finite variance (it is bounded) and $q_t$ is its regression on $(s_t,a_t)$, Theorems 4.2, 5.1, 6.1, 10.3, and 16.1 of \citet{gyorfi2006distribution} establish that this can be done with any of
histogram estimates, kernel regression, $k$-nearest neighbor, sieve regression, or neural networks of growing width, respectively. (These provide $L^2$ convergence of $L^2$ errors, which is stronger than the in-probability convergence of $L^2$ errors we require.)
\item \emph{Nonparametric case}: In the fully nonparametric case, our nuisance estimators may converge more slowly than $\bigO_p(n^{-1/2})$. Our result nonetheless accommodates such lower rates and, crucially, does not impose strong metric entropy conditions that would exclude flexible machine learning estimators. We discuss in greater detail how one might estimate the nuisance functions in \cref{sec:mis,sec:q-learning}.
\end{itemize}
\end{remark}}

As before, we can also obtain a finite-sample guarantee for DRL in $\cm_2$ with a leading constant controlled by the asymptotic variance, which in this case is efficient. \edit{If we can say $C_1,C_2$ do not depend on $C^{T+1}$, this gives a finite sample result with a sample complexity, which depends only $C'$ not on $C^{T+1}$, noting $\mathrm{Effbd}(\cm_2)$ is bounded by $C'R^2_{\max}T^2$. }
{\blockedit
\begin{theorem}\label{thm:effcor_m2}
Suppose that for some $C_1,C_2$, for every $n$, with probability at least $1-\delta$, we simultaneously have that $\|\mu^{(j)}_t-\mu_t\|^2_2\leq \kappa_1=C_1 (n^{-2\alpha_{1}}+\log (2KT/\delta)/n)$, $\|q^{(j)}_t-\lambda_t\|^2_2\leq\kappa_2=C_2 (n^{-2\alpha_{2}}+\log (2KT/\delta)/n)$ $\forall t\leq T,\,1\leq j\leq K$. Then, for every $n$, with probability at least $1-7\delta$, we have 
\begin{align*}
      \abs{\hat{\rho}^{\pi_e}_{\operatorname{DRL}(\cm_2)}-\rho^{\pi_e}} \leq&~\sqrt{\frac{2\log(2/\delta)\mathrm{Effbd}(\cm_2)}n}\\ &~+Q_1\sqrt{\frac{\log(2/\delta)T^{2}(TR_{\mathrm{max}}C'\sqrt{\kappa_1\kappa_2}+\kappa_1T^2R^2_{\mathrm{max}}+\kappa_2\{C'\}^2)}{n}}\\ 
      &~+Q_2\frac{\log(2/\delta)TR_{\mathrm{max}}C'}{n}+Q_3 T\sqrt{\kappa_1\kappa_2},
\end{align*}
where $Q_1,Q_2,Q_3$ are constants not depending on $\delta,T,R_{\mathrm{max}},C',n,C_1,C_2,C$. 
\end{theorem}
}
As before, if $\alpha_1>0,\alpha_2>0,\alpha_1+\alpha_2>1/2$, then the leading term in $n$ in the above bound is exactly $\sqrt{\frac{2\log(2/\delta)\mathrm{Effbd}(\cm_2)}n}$ (with no constant factor).

The DRL estimator in $\cm_2$ can also achieve efficiency without sample splitting (\ie, with adaptive in-sample estimation of nuisances) if we impose an Donsker condition on the estimated nuisances.
{\blockedit
\begin{theorem}[Efficiency without sample splitting]\label{thm:double_m2_sam}
Let $\hat{\mu}_t,\hat q_t$ be estimators based on $\mathcal D$ and let 
$\hat\rho^{\pi^{e}}_{\mathrm{DRL}(\cm_2),\,\mathrm{adaptive}}=
\rE_{n}\bracks{
\sum_{t=0}^{T}\prns{\hat \mu_t\prns{r_t-\hat q_{t}}-\hat \mu_{t-1}\rE_{\pi^e}\bracks{\hat q_{t}\mid \ch_{s_t}}}}$.
Suppose $\|\hat{\mu}_{t}-\mu_{t}\|_2\|\hat{q}_t-q_t\|_2=\op(n^{-1/2}),\|\hat{\mu}_{t}-\mu_{t}\|_2=\op(1),\,\|\hat{q}_t-q_t\|_2=\op(1)$ for $0\leq t \leq T$ and that $\hat{q}_t,\,\hat{\mu}_t$ belong to a Donsker class. Then, the estimator $\hat\rho^{\pi^{e}}_{\mathrm{DRL}(\cm_2),\,\mathrm{adaptive}}$ achieves the semiparametric efficiency bound under $\cm_2$. 
\end{theorem}
}

In addition to efficiency, DRL enjoys a double robustness guarantee in $\cm_2$.

{\blockedit
\begin{theorem}[Double robustness ($\sqrt{n}$-consistency)]
\label{thm:doublerobust_m2}
Suppose $\|\hat{\mu}^{(j)}_{t}-\mu^\dagger_{t}\|_2=\bigO_{p}(n^{-\alpha_{1}})$ and $\|\hat{q}^{(j)}_t-q^\dagger_t\|_2=\bigO_{p}(n^{-\alpha_{2}})$ for $0\leq t \leq T,1\leq j\leq K$. If, for each $0\leq t\leq T$, either $\mu^{\dagger}_{t}=\mu_t$ and $\alpha_{1}\geq 1/2,\,\alpha_{2}>0$ or $q^{\dagger}_t=q_t$ and $\alpha_{2}\geq 1/2,\,\alpha_{1}>0$, then the estimator $\hat{\rho}_{\mathrm{DRL}(\cm_2)}$ is $\sqrt{n}$-consistent around $\rho^{\pi^{e}}$. 
\end{theorem}
Again, we obtain consistency without a rate even if just one nuisance is consistent without a rate.
\begin{corollary}[Double robustness (consistency)]
\label{cor:db2}
Suppose $\|\hat{\mu}^{(j)}_{t}-\mu^\dagger_{t}\|_2=\smallo_{p}(1)$ and $\|\hat{q}^{(j)}_t-q^\dagger_t\|_2=\smallo_{p}(1)$ for $0\leq t \leq T,1\leq j\leq K$. If, for each $0\leq t\leq T$, either $\mu^{\dagger}_{t}=\mu_t$ or $q^{\dagger}_t=q_t$, then the estimator $\hat{\rho}_{\mathrm{DRL}(\cm_2)}$ is consistent around $\rho^{\pi^{e}}$. 
\end{corollary}
}

\begin{remark}\label{rem:dangerous}
When the behavior policy is known, the estimator $\hat{\rho}_{\mathrm{DRL}(\cm_1)}$ is still \edit{$\sqrt{n}$-consistent} %
under $\cm_2$ even without smoothness conditions on $\mu_{t}$ because $\cm_2$ is included in $\cm_1$ so that \cref{thm:doublerobust_m1} applies.
On the other hand, \edit{in the nonparametric setting,} the estimator $\hat{\rho}_{\mathrm{DRL}(\cm_2)}$ requires some smoothness conditions even if the behavior policy is known because $\mu_t$ must still be estimated.
In this sense, when the behavior policy is known, $\hat{\rho}_{\mathrm{DRL}(\cm_1)}$ is more robust than  $\hat{\rho}_{\mathrm{DRL}(\cm_2)}$ under $\cm_2$ but its asymptotic variance is bigger, and generally strictly so.
\end{remark}

A reaming question is how to estimate the nuisances at the necessary rates. We discuss $q$-function estimation in \cref{sec:q-learning}. 
For estimating $\mu_k$, one can leverage the following relationship to reduce it to a regression problem:
\begin{align}
\label{eq:key_relation}
\mu_t(s_t,a_t)=\eta_{t}(s_t,a_t)w_{t}(s_t),\quad\text{where}~
w_{t}(s_t)=\rE[\lambda_{t-1}\mid s_t].
\end{align}
Thus, for example, when the behavior policy is known, we need only estimate $w_t$, which amounts to regressing $\lambda_{t-1}$ on $s_t$. So, in particular, if $w_t(s_t)$ belongs to a {\Holder} class with smoothness $\alpha$ and $s_t$ has dimension $d_s$, estimating $w_t$ with a sieve-type estimator $\hat w_t$ based on the loss function $\prns{\lambda_{t-1}-w_t(s_t)}^{2}$ and letting $\hat\mu^{(j)}_t(s_t,a_t)=\eta_{t}(s_t,a_t)\hat w^{(j)}_{t}(s_t)$ will give a convergence rate $\|\hat\mu^{(j)}_t(s_t,a_t)-\mu_t(s_t,a_t)\|_2=\bigO_{p}(n^{-\alpha/(\alpha+d_{s_t})})$ \citep{ChenXiaohong2007C7LS}. When the behavior policy is unknown, it can be first estimated to construct $\hat{\lambda}_{t}$ and we can repeat the above replacing $\lambda_{t}$ with $\hat{\lambda}_{t}$. In particular, there will be no deterioration in rate if $\pi^b_t$ also belongs to a {\Holder} class with smoothness $\alpha$ and if we further split each $\mathcal D_j$, estimate $\pi_b^t$ as in \cref{lem:denstiy_estimation} on one half, and plug it in to estimate $w_t$ on the other half. Further strategies for estimating $\mu_t$ are discussed in \cref{sec:mis} below.

In the special case where we use parametric models for $\mu_{t}$ and $q_{t}$, under some regularity conditions, parametric estimators will generally satisfy $\|\hat{\mu}_{t}- \mu_{t}^\dagger \|_2=\bigO_{p}(n^{-1/2})$ and $\|\hat{q}_{t}-q_{t}^\dagger\|_2=\bigO_{p}(n^{-1/2})$, where $q_t^\dagger=q_t$ and $\mu_t^\dagger=\mu_t$ if the models are well-specified. (See \cref{sec:q-learning} regarding estimating the $q$-function). Thus, if both models are correctly specified, then \cref{thm:double_m2} yields local efficiency. If only one model is correctly specified, \cref{thm:doublerobust_m2} yields double robustness. 

\subsection{Estimating Marginalized Density Ratios and the Inefficiency of Marginalized Importance Sampling}\label{sec:mis}

In this section we discuss strategies for estimating $\mu_t$ and also show that doing OPE estimation using only marginalized density ratios, as recently proposed, leads to inefficient evaluation in $\cm_2$.

\edit{Specifically, the Marginalized Importance Sampling (MIS) estimator is given by DRL without sample splitting when we let $\hat q_t=0$}:
\begin{align}
\label{eq:m2_ipw_1}
\hat{\rho}^{\pi^e}_{\mathrm{MIS}}=\rE_{n}\bracks{\sum_{t=0}^{T}
\hat\mu_t
r_t}.
\end{align}
\edit{Note that since $\mu_t=\eta_tw_t$, when behavior policy is known we can estimate $\mu_t$ as $\hat \mu_t=\eta_t\hat w_t$.}
We focus on two cases: when $\hat w$ is estimated using a histogram by averaging $\lambda_{t-1}$ by state over a finite state space and a nonparametric extension.

\begin{theorem}[Asymptotic variance of $\hat{\rho}^{\pi^e}_{\mathrm{MIS}}$ with finite state space]\label{thm:ipw_m2_finite}
Suppose $\abs{\mathcal S_t}<\infty$ for $0\leq t\leq T$.
Let 
\begin{align}\label{eq:histogram}
  \hat w_t(s_t)=\frac{\sum_{i=1}^n\mathbb I[{s_t^{(i)}=s_t}]\lambda_{t-1}}{\sum_{i=1}^n\mathbb I[{s_t^{(i)}=s_t}]}. 
\end{align}
Then $\hat{\rho}^{\pi^e}_{\mathrm{MIS}}$ is consistent and asymptotically normal (CAN) around $\rho^{\pi^e}$ and its asymptotic MSE is
\begin{align}\label{eq:ipw_m2_asympmse}
\mathrm{var}\bracks{\sum_{t=0}^{T}\mu_{t}r_t+(\lambda_{t-1}-w_{t})\rE_{\pi_e}[r_t\mid s_t]}.
\end{align}
\end{theorem}
For the proof of Theorem \ref{thm:ipw_m2_finite}, we use an argument based on the theory of $U$-statistics \citep[Ch. 12]{VaartA.W.vander1998As} in order to rephrase the MIS estimator with $\hat w$ as in \cref{eq:histogram} in an asymptotically linear form:
$\hat{\rho}^{\pi^e}_{\mathrm{MIS}}=\rE_n[\sum_{t=0}^T\mu_tr_t+(\lambda_{t-1}-w_t)\rE_{\pi_e}[r_t|s_t]]+\op(n^{-1/2})$.

This influence function is different from the efficient influence function; therefore, $\hat{\rho}^{\pi^e}_{\mathrm{MIS}}$ \edit{with histogram nuisance estimators} is not efficient (the efficient influence function is unique). In fact, we can confirm this fact by calculating and comparing the variances. 
\begin{theorem}\label{thm:ipw_inefficient}
If $P_{\pi^b}\in\cm_2$ (\ie, the underlying distribution is an MDP), \cref{eq:ipw_m2_asympmse} is greater than or equal to $\operatorname{EffBd}(\cm_2)$. The difference is 
\begin{align*}
\mathrm{var}[v_{0}]+\sum^{T-1}_{t=0}\rE \left[(w_{t}-\lambda_{t-1})^{2} \mathrm{var}\left(\eta_{t}v_t(s_{t+1})|s_{t}\right)\right].
\end{align*}
\end{theorem}

We now turn to the nonparametric case, where we first consider a sieve-type extension of the $\hat w$ estimator.
\begin{theorem}[Asymptotic variance of $\hat{\rho}^{\pi^e}_{\mathrm{MIS}}$ with nonparametric $w_t$ estimate]
\label{thm:ipw_m2}
Suppose\break $\rE[(\lambda_{t}-\mu_t)^q]<\infty$ for some $q>1$.
Let
\begin{align}
\label{eq:sieve_opt}
    \hat{w}_{t}(s_t)=\argmin_{w_{t}(s_t) \in \Lambda^{\alpha}_{d_{s_t}}}\rE_{n}[\prns{w_t(s_t)-\lambda_{t-1}}^{2}],
\end{align}
where $\Lambda^{\alpha}_{d_{s_t}}$ is the space of {\Holder} functions with smoothness $\alpha$ and the dimension $d_{s_t}$. Assume $w_t\in\Lambda^{\alpha}_{d_{s_t}}$, $\alpha/(2\alpha+d_{s_t})>1/4$. Then the estimator $\hat{\rho}^{\pi^e}_{\mathrm{MIS}}$ is CAN around $\rho^{\pi^{e}}$ and its asymptotic MSE is equal to \cref{eq:ipw_m2_asympmse}.
\end{theorem}

\begin{remark}
The estimator \cref{eq:sieve_opt} is over an infinite-dimensional function space. It can be replaced with a finite-dimensional approximation $\Lambda^{\alpha}_n$ such that $\Lambda^{\alpha}_n \to \Lambda^{\alpha}_{d_s}$.
Following Example 1(b) in Section 8 \citep{ShenXiaotong1997OMoS}, it can be shown that this will lead to the same asymptotic MSE as in \cref{eq:ipw_m2_asympmse} and not change the conclusion of \cref{thm:ipw_m2}. 
\end{remark}

\begin{remark}

When the action and sample space is continuous, the histogram estimator in \cref{eq:histogram} can also easily be extended to a kernel estimator:
\begin{align}\label{eq:kernel}
    \hat{w}_t(s_t)=\frac{\sum_{i=1}^n\mathrm {K}_{h}({s_t^{(i)}-s_t})\lambda_{t-1}}{\sum_{i=1}^n\mathrm {K}_{h}({s_t^{(i)}-s_t})},
\end{align}
where $\mathrm{K}_h$ is a kernel with a bandwidth $h$. 

The smoothness condition in \cref{thm:ipw_m2} ensures we can estimate $w_t$ at fourth-root rates using \cref{eq:sieve_opt}. Following \citet{newey94} and utilizing a high-order kernel, we can obtain similar fourth-root rates for \cref{eq:kernel} and a similar variance result for MIS. Unlike \cref{eq:sieve_opt}, we cannot invoke a Donsker condition to prove a stochastic equicontinuity condition. However, it is still possible to show this directly based on a V-statistics theory (see Chapter 8 of \citealp{newey94}). 

\end{remark}

Finally, we also consider estimating $\mu_t$ directly and nonparametrically using the relation
\begin{align}\label{eq:muonnuregression}
\mu_t(s_t,a_t)=\rE[\lambda_{t}\mid s_t,a_t].
\end{align}
A sieve-type regression estimator for $\mu_t$ is then constructed as 
\begin{align}\label{eq:sieve_2}
    \hat{\mu}_{t}(s_t,a_t)=\argmin_{\mu_{t}(s_t,a_t) \in \Lambda^{\alpha}_{d_{s_t}+d_{a_t}}}\rE_{n}[\prns{\mu_t(s_t,a_t)-\lambda_{t}}^{2}]. 
\end{align}

\begin{theorem}[Asymptotic variance of $\hat{\rho}^{\pi^e}_{\mathrm{MIS}}$ with nonparametric $\mu_t$ estimate]
\label{thm:ipw_m2_2}
Suppose $\rE[(\lambda_{t}-\mu_t)^q]<\infty$ for some $q>1$. Let $\hat{\mu}_{t}$ be as in \cref{eq:sieve_2}. Assume $\mu_t\in\Lambda^{\alpha}_{d_{s_t}+d_{a_t}}$, $\alpha/(2\alpha+d_{s_t}+d_{a_t})>1/4$. Then the estimator $\rE_{n}\bracks{\sum_{t=0}^{T}\hat{\mu}_{t} r_t}$  is CAN around $\rho^{\pi^{e}}$ and its asymptotic MSE is equal to \begin{align}\label{eq:m2_ipw_2}
    \mathrm{var}\left[\sum_{t=0}^{T}\mu_t r_t+(\lambda_{t}-\mu_t)\rE[r_t\mid s_t,a_t] \right]. 
\end{align}
\end{theorem}
\begin{remark}
While both estimators for $\mu_t$ in \cref{thm:ipw_m2,thm:ipw_m2_2} achieve fourth-root rates under the respective conditions,
the resulting asymptotic variances in \cref{eq:ipw_m2_asympmse,eq:m2_ipw_2} are different and generally incomparable. Both are inefficient, but which is larger is problem-dependent. Note that, in contrast, the asymptotic variance of DRL (\cref{thm:double_m2}) is the same (and is efficient) regardless of which way is used to estimate $\mu_t$ as long as we have the necessary rate. When the behavior policy is known, using \cref{eq:sieve_opt} may be better than \cref{eq:sieve_2} when estimating $\mu_t$ nonparametrically because the smoothness condition is weaker and the convergence rate is faster (since $d_{s_t}<d_{a_t}+d_{s_t}$).
However, when using parametric models, the rates are the same (under correct specification) and sometimes it is easier to model $\mu_t(s_t,a_t)$ rather than $w_t(s_t)$, as we do in \cref{sec:exp-toy}. 
\end{remark}

{\blockedit
\begin{remark}[The Inefficiency of MIS]\label{remark:oracleMIS}
\Cref{thm:ipw_m2_finite,thm:ipw_m2,thm:ipw_m2_2} each study the MIS estimator $\hat{\rho}^{\pi^e}_{\mathrm{MIS}}$ in \cref{eq:m2_ipw_1} but with different estimator for the nuisance $\mu_t=\eta_tw_t$. As noted above, unlike DRL, the variance of the MIS estimator actually depends on the way this nuisance is estimated. And, in each case, the MIS estimator was inefficient. In the finite-state-space setting with known behavior policy, \citet{XieTengyang2019OOEf} propose another, different MIS estimator based on estimating the MDP transition kernel; but per their Remark 4 it is also inefficient.
(In contrast, \cref{remark:cases} shows that DRL is efficient in the finite-state-space setting \emph{without} requiring any smoothness conditions.)
This does not immediately imply the MIS estimator is always inefficient, as it may depend on how $\mu_t$ is estimated, but semiparametric theory strongly suggests there is reason to believe that MIS would in general be inefficient.

One natural question that sheds light on this is how would a hypothetical MIS estimator perform with oracle values for $\mu_t$. In fact, the variance of $\sum_{t=0}^{T}\mu_t r_t$ is in general \emph{incomparable} to $\operatorname{EffBd}(\mathcal M_2)$, that is, it may be smaller \emph{or} larger depending on the particular instance. This may surprising but is not contradictory since one can in fact prove that no regular estimator (let alone an efficient one) in either $\cm_2$ or $\cm_{2b}$ could ever have the form $\E_n[\sum_{t=0}^{T}\mu_t r_t]+\op(n^{-1/2})$, that is, asymptotically linear with influence function $\sum_{t=0}^{T}\mu_t r_t$. This is because $\sum_{t=0}^{T}\mu_t r_t$ is \emph{not} a gradient of $\rho^{\epol}$ under either $\cm_2$ or $\cm_{2b}$ (see \cref{thm:gradients}). This is in stark contrast to IS: $\sum_{t=0}^{T}\lambda_t r_t$ is always a valid influence function under either $\cm_{1b}$ or $\cm_{2b}$ since we know its empirical average always gives an unbiased linear estimator (not just asymptotically). Indeed, we similarly have that the variance of $\sum_{t=0}^{T}\mu_t r_t$ is \emph{also} incomparable to $\sum_{t=0}^{T}\lambda_t r_t$. The function $\sum_{t=0}^{T}\mu_t r_t$ is an influence function (a gradient) under the MDP model with \emph{known} transition kernel, but that is a very restrictive and unrealistic model.

One interesting specific case is the fully tabular setting (finite state \emph{and} action spaces).
Since our paper was posted, the more recent \citet{Yin2020} considered a ``modified'' version of the estimator of \citet{XieTengyang2019OOEf} in order to obtain efficiency under the tabular case of $\cm_{2b}$.
By simple algebra, the estimator of \citet{Yin2020}, which is defined as
\begin{align*}
 &\textstyle\frac{1}{n}\sum_{i=1}^n\sum_{t=0}^T\hat  w_t(s^{(i)}_t){\int r_t\hat P_{r_t}(r_t|s^{(i)}_t,a_t)\epol(a_t|s^{(i)}_t)\mathrm{d}(r_t,a_t)},\\
 &\textstyle\text{where}~\hat  w_t(s_t)=\frac1{\hat p_{\bpol_t}(s_t)}\int \hat P_{s_{t}}(s_{t}|s_{t-1},a_{t-1})\prod_{k=0}^{t-1}\prns{ \epol_k(a_k|s_k)\hat P_{s_{k}}(s_{k}|s_{k-1},a_{k-1})}\mathrm{d}(\ch_{a_{t-1}}), 
\end{align*}
and where $\hat P_{s_{t}},\hat P_{r_{t}},\hat p_{\bpol_t}$ are each an empirical frequency (histogram) estimator,
can in fact be rewritten simply as
\begin{align*}
 \sum^{T}_{t=0}\int r_t \hat P_{r_t}(r_t|s_t,a_t)\prod_{k=0}^{t} \prns{ \epol_k(a_k|s_k)\hat P_{s_{k}}(s_{k}|s_{k-1},a_{k-1})}\mathrm{d}(\ch_{a_{t}},r_t).
\end{align*}
This is essentially a model-based OPE estimator, where we first fit \emph{all} MDP parameters and then explicitly integrate with respect to the resulting estimated trajectory density function in order to compute the expectation $\rho^{\epol}$. This is also often called the $G$-formula in the causal inference literature \citep{robins1986new,hernan2019}. In the tabular setting, the efficiency of this estimator is immediate since it is exactly the (parametric) MLE estimate for this setting, which is well known to achieve the Cram\'er-Rao lower bound, and the Cram\'er-Rao lower bound is the efficiency bound in the tabular setting since the model is parametric.
In the case of continuous state and/or action spaces, the simple extension of replacing $\hat P_{s_{k+1}}(s_{k+1}|s_{k},a_{k}),\hat P_{r_t}(r_t|s_t,a_t)$ with some nonparametric conditional density estimators would have poor performance since the nonparametric density estimation is unstable and would significantly inflate the variance.
Alternatively, if we extend the estimator by instead estimating $\mu_t$ or $w_t$ nonparametrically, the above already argues why we expect this would generally be inefficient.

Similarly to the model-based approach, in the tabular case, our results in the next section have shown that also simple DM estimates based on $q$-function estimation are also efficient, since in the tabular case $q$-functions are parametric.
Moreover, DRL was already shown to be efficient in the the tabular MDP setting with \emph{any} parametric $\mu_t$ estimator as they all have $\bigO_p(1/\sqrt{n})$ convergence in this setting, and the particular choice of estimator does not affect this (see also \cref{remark:cases}).
Hence, DRL was the \emph{first} efficient OPE estimator, both in general and in the tabular MDP setting in particular.
\end{remark}
\begin{remark}[Other estimators for $\mu_t$]
\citet{XieTengyang2019OOEf,Yin2020} may in fact both offer alternative estimators for $w_t$ and hences $\mu_t$ (in their respective settings) that may be used in DRL and either will ensure efficiency for DRL (see \cref{remark:cases}).
In particular, $\lambda_t$ may have variance growing exponentially in $T$, this may affect the variance of estimates based on the regression of it on $s_t$ or on $s_t,a_t$, as studied above. Although this will not appear in the leading term of the variance of DRL and will not affect efficiency, it may still be a concern. 
Developing and analyzing alternative estimators for $w_t$ and/or $\mu_t$ may be fruitful future work. 
For example, still other possible estimation approaches for $w_t$ include a fitted $w$-iteration: start with $\hat w_0\equiv1$, regress $\hat w_{t-1}\eta_t$ on $s_t$ using any supervised regression method to obtain $\hat w_t$, and repeat.
\end{remark}}

\section{Estimating the $q$-function and Efficiency Under $\cm_{1q},\,\cm_{2q}$}
\label{sec:q-learning}

In this section, we discuss the estimation of $q$-functions in an off-policy manner, parametrically or nonparametrically, which can be plugged into our estimators, $\hat{\rho}_{\mathrm{DRL}(\cm_1)},\hat{\rho}_{\mathrm{DRL}(\cm_2)}$. On the way, we also derive the semiparametric efficiency bound when we impose parametric restrictions on $q$-functions, \ie, the models $\cm_{1q},\,\cm_{2q}$. 
To do this, we will leverage a recursive definition of the $q$-functions \citep{BertsekasDimitriP2012Dpao}.
Under $\cm_1$, the following recursion equation holds:
\begin{align}\label{eq:q-recursion}
    q_{t} &=\mathrm{E}\bracks{r_t+ \mathrm{E}_{\pi^{e}}\bracks{q_{t+1}\mid\ch_{s_{t+1}}}\mid \ch_{a_t}}. 
\end{align}
Under $\cm_2$, we can further replace $\ch_{s_{t+1}}$ with $s_{t+1}$ and $\ch_{a_t}$ with $(s_t,a_t)$ in the above.

The recursion in \cref{eq:q-recursion} can equivalently be written as a set of conditional moment equations satisfied by the $q$-functions:
\begin{align}\label{eq:recursive-m}
&m_t(\ch_{a_t};\braces{q_1,\dots,q_T})=0\quad\forall t\leq T,\\
&\notag\text{where}\quad m_t(\ch_{a_t};\braces{q'_1,\dots,q'_T})=
\mathrm{E}\bracks{r_t+ \mathrm{E}_{\pi^{e}}\bracks{q'_{t+1}(\ch_{a_t+1})\mid\ch_{s_{t+1}}}-q'_{t}(\ch_{a_t})\mid \ch_{a_t}}.
\end{align}
This formulation of the $q$-function in terms of conditional moment equations, along with the observation that $\rho^{\pi^e}=\rE\bracks{\mathrm{E}_{\pi^{e}}\bracks{q_{0}(s_0,a_0)\mid s_0}}$ is determined by the $q$-function, allows us both to estimate the $q$-function efficiently, either parametrically and nonparametrically, and to characterize the efficiency bounds under $\cm_{1q}$ and $\cm_{2q}$. We start with the latter.

\subsection{Efficiency Bounds Under $\cm_{1q},\,\cm_{2q}$}

In this section we consider the models where we restrict $q$-functions parametrically:
\begin{align*}
\cm_{1q}&=\braces{P_{\pi^b}\in \cm_1:\exists \beta^*_t\in\Theta_{\beta_t},\,q_t(\ch_{a_t})=q_t(\ch_{a_t};\beta^*_t)\ \ \forall t\leq T},\\
\cm_{2q}&=\braces{P_{\pi^b}\in \cm_2:\exists \beta^*_t\in\Theta_{\beta_t},\,q_t(s_t,a_t)=q_t(s_t,a_t;\beta^*_t)\ \ \forall t\leq T},
\end{align*}
where $q_t(\ch_{a_t};\beta_t)$ or $q_t(s_t,a_t;\beta_t)$ is some parametric model for the $q$-function at time $t$ that is continuously differentiable with respect to the parameter $\beta_t$, $\Theta_{\beta_t}$ is some compact parameter space, and $\beta_t^*$ is the true parameter, which is assumed to lie in the interior of $\Theta_{\beta_t}$. For brevity we define $v_t(\ch_{s_t};\beta_t)=\rE_{\pi^e}\bracks{q_t(\ch_{a_t};\beta_t)\mid \ch_{s_t}}$ and similarly $v_t({s_t};\beta_t)$.

Under $\cm_{1q},\,\cm_{2q}$, \cref{eq:q-recursion} can be rephrased as as a set of conditional moment restrictions on the parameter $\beta$ defined by $\beta=(\beta^{\top}_1,\cdots\beta^{\top}_T)^{\top}$. In particular, overloading notation and letting $m_t(\ch_{a_t};\beta)=m_t(\ch_{a_t};\braces{q_1(\cdot,\beta_1),\dots,q_T(\cdot,\beta_T)})$, we have that $\beta$ is defined by the set of conditional moment equations $m_t(\ch_{a_t};\beta)=0\ \forall t\leq T$.
This observation is key in establishing the following result.

\begin{theorem}[Efficiency bound under $\cm_{1q},\,\cm_{2q}$]\label{thm:bound_mq_rl}
Define $e_{q,t}=r_t+v_{t+1}-q_t$
\begin{align*}
    A_t  &=C^{-1}_t+C^{-1}_tB_t A_{t+1} B^{\top}_t {C^{-1}_t},\\
    B_t &=\rE\bracks{\nabla_{\beta_t} q_{t}(\ch_{a_{t}};\beta^*_{t})\mathrm{var}(e_{q,t}\mid\ch_{a_t})^{-1}\nabla^{\top}_{\beta_{t+1}}v_{t+1}(\ch_{s_{t+1}};\beta_{t+1}^*)},\\
    C_t &= \rE\bracks{\nabla_{\beta_t} q_t(\ch_{a_{t}};\beta_t^*) \mathrm{var}(e_{q,t}\mid\ch_{a_t})^{-1}\nabla^{\top}_{\beta_{t}} q_t(\ch_{a_{t}};\beta^*_t)},\\
    A_{T}&=\rE\bracks{\nabla_{\beta_T}q_{T}(\ch_{a_{T}};\beta^*_T)\mathrm{var}(e_{q,T}\mid\ch_{a_T})^{-1}\nabla^\top_{\beta_T}q_{T}(\ch_{a_{T}};\beta^*_T)}^{-1},\\
    B_{-1}&=\rE[\eta_{0}(s_0,a_0)\nabla^{\top}_{\beta_0} q_{0}(s_0,a_0;\beta_0^*)].
    \end{align*}
Then $$\operatorname{EffBd}(\cm_{1q})=\mathrm{var}\prns{v_{0}}+B_{-1}A_0B^{\top}_{-1}.$$
Moreover, the efficiency bound for estimating $\beta_t$ is $A_t$.

Finally, the corresponding efficiency bounds under $\cm_{2q}$ are given by replacing $\ch_{s_{t+1}}$ with $s_{t+1}$ and $\ch_{a_t}$ with $(s_t,a_t)$ everywhere in the above.
\end{theorem}

\begin{remark}
When $T=1$, $\operatorname{EffBd}(\cm_{1q})$ above is equal to 
\begin{align*}
&\mathrm{var}[v_{0}(s_0)]+B_{-1}A_{0}B^{\top}_{-1},\\\text{where}\quad& A_0 = \rE[\nabla_{\beta_0}q_{0}(s_0,a_0;\beta_0^*)\mathrm{var}[r_0\mid s_0,a_0]^{-1} \nabla_{\beta}^{\top}q_{0}(s_0,a_0;\beta_0^*)]^{-1},
\end{align*}
The Matrix Cauchy-Schwarz Inequality \citep{matrix} immediately shows that this is upper bounded by $\operatorname{EffBd}(\cm_1)$, as is also implied by $\cm_{1q}\subset\cm_1$ albeit less directly. 
\end{remark}

\subsection{Parametric Estimation of $q$-functions}\label{sec:parametric-q-est}

Next, we consider an estimation method for $\beta_t$ and $\rho^{\pi^e}$. Given the above observations, a natural way to estimate $\beta$ is by solving the following set of conditional moment equations given by $m_t(\ch_{a_t};\beta)=0\ \forall t\leq T$.
For example, one approach when the $q$-model is a linear model specified as $\beta^{\top}_t \phi_t(\ch_{a_t})$ for some $d_{\phi_t}$-dimensional feature expansion $\phi_t$ is to choose $\hat\beta$ to minimize $\sum_{t=0}^T\sum_{i=1}^{d_{\phi_t}}\prns{\rE_n\bracks{\prns{r_t + v_{t+1}(\ch_{s_{t+1}};\beta_{t+1})-q_{t}(\ch_{a_t};\beta_t)} \phi_{ti}(\ch_{a_t})}}^2$, which
corresponds exactly to backward-recursive ordinary least squares. That is, first $r_T$ is regressed on $\phi_t(\ch_{a_T})$ to obtain $\hat\beta_T$, then $q_T(\ch_{a_T};\hat\beta_T)$ is averaged over $\pi_T^e(a_T\mid \ch_{s_T})$ to obtain $\hat v_T$, then $r_{T-1}+\hat v_T$ is regressed on $\phi_t(\ch_{a_{T-1}})$ to obtain $\hat\beta_{T-1}$, and so on.

Although such an estimator can achieve the rate $\bigO_{p}(n^{-1/2})$ under correct specification and standard conditions for $M$-estimators, it might not yield an efficient estimator for $\beta$ or for $\rho^{\pi^e}$. 
When the $q$-model is linear as above, this can be easily solved by instead applying any efficient variant of the generalized method of moments (GMM), such as two-step GMM \citep{hansen1996finite,Hansen82}, to the set of moment equations given by $m_t(\ch_{a_t};\beta)\phi_{ti}(\ch_{a_t})=0\ \forall t\leq T,\,i\leq d_{\phi_t}$. This is almost the same as the above backward-recursive ordinary least squares but with an optimal weighting of the different moment conditions in the sum above.

When the $q$-model may be nonlinear, we can obtain an efficient estimator by instead applying the method of \citet{HahnJinyong1997Eeop} to our set of conditional moment equations. Specifically, we can consider the set of $Tm_n$ moment equations
$\rE\bracks{m_t(\ch_{a_t};\beta)\phi_{ti}(\ch_{a_t})}=0\ \forall t\leq T,\,i\leq m_n$,
where $\phi_{t1}(\ch_{a_t}),\phi_{t2}(\ch_{a_t}),\dots$ is a basis expansion of the $L^{2}$-space and $m_{n}\to \infty$ as $n\to\infty$. 
Then, applying any efficient variant of GMM to this set of moment conditions will yield an efficient estimator $\hat\beta$ of $\beta$. 

In all of the above, replacing $\ch_{s_{t+1}}$ with $s_{t+1}$ and $\ch_{a_t}$ with $(s_t,a_t)$, the same techniques can be applied in $\cm_2$.
In either case, once we have an efficient estimate $\hat\beta$ of $\beta$, an efficient estimate for $\rho^{\pi^{e}}$, achieving the semiparametric efficiency bound in the appropriate model, is given by $\hat{\rho}_{DM}=\rE_{n}\bracks{v_{0}(s_0;\hat\beta_0)}$. 

{\blockedit
\begin{remark}[Tabular setting]
Consider a tabular case. Then, by treating $q_t(s_t,a_t)=\beta^{\top}\phi(a_t,s_t)$, where $\phi(a,s)=(\mathrm{I}(a=a^{\dagger}_1,s=s^{\dagger}_1),\cdots,\mathrm{I}(a=a^{\dagger}_{|A_t|},s=s^{\dagger}_{|S_t|}))$ and $a^{\dagger}_i$ and $s^{\dagger}_i$ are the elements of the finite $\mathcal{A}_t$ and $\mathcal{S}_t$, we can observe that $\mathrm{Effbd}(\cm_{2q})=\mathrm{Effbd}(\cm_{2})$ by some algebra. This result is natural since $\cm_{2q}=\cm_{2}$ in the tabular setting. 
\end{remark}
}

\subsection{Nonparametric Estimation of $q$-functions}

The above observation in \cref{eq:q-recursion} that $q$-functions satisfy a set of conditional moment equations also lends itself to nonparametric estimation of the $q$-functions. In this section we briefly review how one approach to this, following the application of the method of \citet{AiChunrong2009Sebf} to this set of conditional moment equation, can obtain the necessary fourth-root rates for use in DRL.

The estimator $\{\hat{q}_t\}_{t=0}^{T}$ is constructed as the following sieve minimum distance estimator:
\begin{align}
\notag
\{\hat{q}_t\}_{t=0}^{T}\in
   \argmin_{q_t \in\Lambda_{t,n}\,\forall t\leq T}\ \sum_{t=0}^{T}\rE_{n}\bracks{\hat{m}_t(\ch_{a_t};q_t)\hat{\Sigma}_{t}^{-1}\hat{m}_{t}(\ch_{a_t};q_t)},
\end{align}
where $\hat{m}_t(\ch_{a_t};q_t)$ is a nonparameric estimator for $m_t(\ch_{a_t};q_t)$, $\hat{\Sigma}_{t}$ is a nonparametric estimator for $\mathrm{var}\prns{e_{q,t}\mid\ch_{a_t}}$, and $\Lambda_{k,n}$ is a sequence of approximation space whose union $\cup_{n=1}^{\infty}\Lambda_{t,n}$ is dense in some infinite dimensional space $\Lambda_{t}$. Alternatively, in $\cm_2$, we replace $\ch_{a_t}$ with $(s_t,a_t)$ in the above.

\citet{AiChunrong2003EEoM} prove that applying the above with appropriate nonparametric estimators, under some smoothness conditions, we can obtain $\|\hat{q}_t-q_t\|_{F,t}=\smallo_{p}(n^{-1/4})$, where $\|\cdot\|_{F,t}$ is the Fisher metric, which in our setting of \cref{eq:recursive-m} is defined as 
\begin{align*}
\|g(\ch_{a_t})\|^{2}_{F,t}=
\rE[\mathrm{var}(e_{q,t}\mid\ch_{a_t})g^{2}+\mathrm{var}(e_{q,t-1}\mid\ch_{{a}_{t-1}})\rE_{\pi^{e}}\bracks{g(\ch_{a_t})\mid\ch_{s_t}}^{2}].
\end{align*}
We omit the details and refer the interested reader to \citet{AiChunrong2003EEoM}. We only prove that this norm is in fact equivalent to the $L^2$-norm under mild conditions.
\begin{lemma}
\label{lem:distance}
Suppose $\mathrm{var}[e_{q,t}\mid\ch_{a_t}]$ and $\mathrm{var}[e_{q,t-1}\mid\ch_{a_{t-1}}]$ are bounded away from zero. Then, $\|\cdot\|_{F,k}$ and $\|\cdot\|_{2}$ are equivalent norms. 
\end{lemma}
This means that, under the appropriate conditions, the estimator $\hat q$ obtains the rate $\smallo_{p}(n^{-1/4})$ in terms of $L^{2}$-norm, as necessary for \cref{thm:double_m1,thm:doublerobust_m1,thm:double_m2,thm:doublerobust_m2}.

\section{Experiments}\label{sec:experiment}

We now turn to an empirical study of OPE and DRL. First, we construct a simulation to investigate the effect of using memorylessness on estimation variance as well as the effect of double robustness on model specification sensitivity. Then, we study comparative performance of different OPE estimators in two standard OpenAI Gym tasks.

Replication code for all experiments is available at \url{http://github.com/CausalML/DoubleReinforcementLearningMDP}.

\subsection{The Effects of Leveraging Memorylessness and of Double Robustness}\label{sec:exp-toy}

\begin{table}[t!]
    \centering
        \caption{Experiment from \cref{sec:exp-toy}: RMSE (and standard errors).
        }
    \begin{tabular}{ccllllll}\toprule
     Setting & $n$  & \multicolumn{1}{c}{$\hat{\rho}_{\mathrm{IS}}$} & \multicolumn{1}{c}{$\hat{\rho}_{\mathrm{DRL}(\cm_1)}$} & \multicolumn{1}{c}{$\hat{\rho}_{\mathrm{DM}}$} & \multicolumn{1}{c}{$\hat{\rho}_{\mathrm{MIS}}$} & \multicolumn{1}{c}{$\hat{\rho}_{\mathrm{DRL}(\cm_2)}$} \\\midrule
     \multirow{3}{*}{(1)} & 1500  & 42.4 (12.4) & 36.1 (16.8) & 0.70 (0.002) & 40.8 (12.5)  & 0.70 (0.002) \\
      & 3000 &  20.4 (3.1) & 7.8 (0.8) & 0.50 (0.001) & 20.8 (2.8)  & 0.50 (0.001)  \\
      & 4500 & 20.2 (3.1) & 6.6 (0.75) & 0.43 (0.001) & 21.5 (3.5) &  0.43 (0.001) \\\midrule
     \multirow{3}{*}{(2)}  & 1500 & 42.4 (12.4)& 77.6 (29.1) & 10.8 (0.002) & 40.8 (12.5) & 10.3 (3.5)  \\
      & 3000 & 20.4 (2.5) & 36.6 (6.9) & 10.8 (0.001) & 20.8 (2.8) &  6.0 (0.6)  \\
        & 4500 & 20.2 (3.1) & 34.4 (9.6) & 10.8 (0.001) &  21.5 (3.5) & 5.5 (2.0)  \\\midrule
    \multirow{3}{*}{(3)}  & 1500 & 42.4 (12.4) & 36.1 (16.8) & 0.70 (0.002)  & 87.7 (25.5)  &  0.73 (0.03) \\  
      & 3000 &  20.4 (3.1) & 7.8 (0.8) & 0.50 (0.001)  &  37.3 (3.2) & 0.51 (0.002)\\  
      & 4500 & 20.2 (3.1)& 6.6 (0.75) & 0.43 (0.001)  & 53.5 (15.1) & 0.44 (0.005)\\  
      \bottomrule
    \end{tabular}
    \label{tab:mse}
\end{table}

In this section we consider an MDP with a horizon of $T=30$, binary actions, univariate continuous state, initial state distribution $p(s_0)\sim \mathcal{N}(0.5,0.2)$, transition probabilities \edit{$P_{t}(s_{t+1}\mid s_t,a_t)\sim \mathcal{N}(s+0.3a-0.15,0.2)$}. The target and behavior policies we consider are, respectively,
\begin{align*}
    \pi^{e}(a\mid s) & \sim \mathrm{Bernoulli}(p_e),\ p_e=0.2/(1+\exp(-0.1s))+0.2U,\ U\sim \mathrm{Uniform}[0,1] \\
    \pi^{b}(a\mid s) & \sim \mathrm{Bernoulli}(p_b),\ p_b=0.9/(1+\exp(-0.1s))+0.1U,\ U\sim \mathrm{Uniform}[0,1].  
\end{align*}
We assume the behavior policy is known. Note that this setting is an MDP and belongs to $\cm_2$.

We compare five estimators: $\hat{\rho}_{\mathrm{IS}},\ \hat{\rho}_{\mathrm{DRL}(\cm_1)},\ \hat{\rho}_{\mathrm{DM}},\ \hat{\rho}_{\mathrm{MIS}},\ \hat{\rho}_{\mathrm{DRL}(\cm_2)}$ when nuisance functions $q_t(s,a)$ and $\mu_t(s)$ are estimated parametrically. 
We consider three settings:
\begin{enumerate}[(1)]
\item Both models correct: $q_t(s_t,a_t)=\beta_{1t} s_t+\beta_{2t}s_ta_t+\beta_{3t},\,\mu_t(s_t,a_t)=\beta_{4t} s_t+\beta_{5t} s_ta_t+\beta_{6t}$.
\item Only $\mu$-model correct: $q_t(s_t,a_t)=\beta_{1t} s^{2}_t+\beta_{2t} s^{2}_ta_t+\beta_{3t},\,\mu_t(s_t,a_t)=\beta_{4t} s_t+\beta_{5t} s_ta_t+\beta_{6t}$.
\item Only $q$-model correct: $q_t(s_t,a_t)=\beta_{1t} s_t+\beta_{2t} s_ta_t+\beta_{3t},\,\mu_t(s_t,a_t)=\beta_{4t} s_t^{2}+\beta_{5t} s_t^{2}a_t+\beta_{6t}$.
\end{enumerate} 
Note that in the above, the ``correct'' models are in fact not exactly correct because $\rE_{\pi^e}[a_t\mid s_t]$ is actually nonlinear in $s_t$, but it is very nearly linear in the space of observed $s_t$ values (for example, best linear fit for $\rE_{\pi^e}[a_t\mid s_t]$ has an $L^2$ distance $3\times10^{-5}$ on $[0,1]$, which spans $\pm2.5$ standard deviations for $s_0$). We therefore treat them as correctly specified.

In all cases, to estimate $q$-models we use backward-recursive ordinary least squares \edit{as in \cref{sec:parametric-q-est}}. 
To estimate $\mu$-models we use ordinary least squares regression on $\lambda_t$ (which is assumed known) as in \cref{eq:muonnuregression}.

For each $n=1500,3000,4500$, we consider $50000$ Monte Carlo replications. In each replication, we estimate the $q$- and $\mu$-models as above and compute, for each setting, each of $\hat{\rho}_{\mathrm{IS}},\,\hat{\rho}_{\mathrm{DRL}(\cm_1)},\,\hat{\rho}_{\mathrm{DM}},\,\hat{\rho}_{\mathrm{MIS}},\,\hat{\rho}_{\mathrm{DRL}(\cm_2)}$. We report the RMSE of each estimator in each setting (and the standard error) in \cref{tab:mse}.

Our first immediate observation is that $\hat{\rho}_{\mathrm{DRL}(\cm_2)}$ nearly dominates all other estimators, achieving similar or better performance in every setting and sample size. In particular, in settings (1) and (3), where the $q$-model is correct, it has performance similar to $\hat{\rho}_{\mathrm{DM}}$. 
Note that in settings (1) and (3), $\hat{\rho}_{\mathrm{DM}}$ is efficient for $\cm_{2q}$ per \cref{sec:parametric-q-est} (or almost so; it would be efficient if we used efficient GMM instead of one-step GMM). In setting (1), $\hat{\rho}_{\mathrm{DRL}(\cm_2)}$ is locally efficient, while in setting (3), it is only doubly robust and performs almost imperceptibly worse than the efficient $\hat{\rho}_{\mathrm{DM}}$.

In setting (2), where the $q$-model is incorrect, $\hat{\rho}_{\mathrm{DM}}$ is inconsistent and $\hat{\rho}_{\mathrm{DRL}(\cm_2)}$ handily outperforms it. In the same setting (2), the consistent $\hat{\rho}_{\mathrm{IS}}$ and $\hat{\rho}_{\mathrm{MIS}}$ also outperform the inconsistent $\hat{\rho}_{\mathrm{DM}}$ but not by as much as $\hat{\rho}_{\mathrm{DRL}(\cm_2)}$. While $\hat{\rho}_{\mathrm{DRL}(\cm_1)}$ is doubly robust in setting (2) guaranteeing consistency, unlike the case of $\hat{\rho}_{\mathrm{DRL}(\cm_2)}$, the combination of large (unmarginalized) cumulative density ratios and a misspecified $q$-model leads to still worse performance in the sample sizes tested.

Generally, $\hat{\rho}_{\mathrm{IS}}$, $\hat{\rho}_{\mathrm{MIS}}$, and $\hat{\rho}_{\mathrm{DRL}(\cm_1)}$ all have high RMSE due to the significant mismatch between the behavior and target policies so that cumulative density ratios are very large and only marginalizing them without also using a $q$-model helps only a little. 
In settings (1) and (2), where the $\mu$-model is correct, $\hat{\rho}_{\mathrm{MIS}}$ improves on $\hat{\rho}_{\mathrm{IS}}$ only slightly, while in setting (3), where $\mu$-model is incorrect, it performs significantly worse. This highlights the potential danger of misspecifying $\mu$-models compared to the robustness of importance sampling with known behavior policy (see also \cref{rem:dangerous}). 

While both $\hat{\rho}_{\mathrm{IS}}$ and $\hat{\rho}_{\mathrm{DRL}(\cm_1)}$ remain consistent throughout all settings, they are outperformed by the also-consistent $\hat{\rho}_{\mathrm{DRL}(\cm_2)}$, which leverages the MDP structure of $\cm_2$ and exhibits local efficiency  in setting (1) and doubly robustness in settings (2) and (3).

\subsection{Investigating Performance in RL Tasks: Cliff Walking and Mountain Car} \label{sec:exp-aigym}

We next compare the same OPE estimators using nonparametric nuisance estimation in two standard RL settings included in OpenAI Gym \citep{gym}: Cliff Walking and Mountain Car. For further detail on each setting, see \cref{sec:experiment-ape}.

First, we used $q$-learning to learn an optimal policy for the MDP and define it as $\pi^{d}$. 
Then we generate the dataset from the behavior policy $\pi^b=(1-\alpha)\pi^{d}+\alpha \pi^u$ where $\pi^u$ is a uniform random policy and $\alpha=0.8$. We define the target policy similarly but with $\alpha=0.9$. Again, we assume the behavior policy is known. Note that this $\pi_d$ is fixed in each setting.

We estimate all $\mu$-functions by first estimating $w$-functions and using \cref{eq:key_relation}. For Cliff Walking, we use a histogram estimator for $w$ as in \cref{eq:histogram}. For Mountain Car, we use a kernel estimator for $w$ as in \cref{eq:kernel}. We use the Epanechnikov kernel and choose an optimal bandwidth based on an $L^{2}$-risk criterion for $t=1$; we then use this bandwidth for all other $t$ values as well for simplicity.  
For $q$-functions, we use backward-recursive regression \edit{as in \cref{sec:parametric-q-est}}. For Cliff-Walking, we use a histogram model, $q(s,a;\beta)=\sum_{s_j,a_k\in \mathcal S,\mathcal A}\beta_{jk} \mathbb{I}[s_j=s,a_k=a]$. For Mountain-Car, we use the mode $q(s,a;\beta)=\beta^{\top}\phi(s,a)$ where $\phi(s,a)$ is a $400$-dimensional feature vector based on a radial basis function, generated using the \texttt{RBFSampler} method of \texttt{scikit-learn} based on \citet{Rahimi}.

We again compare $\hat{\rho}_{\mathrm{IS}},\ \hat{\rho}_{\mathrm{DRL}(\cm_1)},\ \hat{\rho}_{\mathrm{DM}},\ \hat{\rho}_{\mathrm{MIS}},\ \hat{\rho}_{\mathrm{DRL}(\cm_2)}$. In each setting we consider varying evaluation dataset sizes and for each consider 1000 replications. 
We report the RMSE of each estimator in each setting (and the standard error) in \cref{tab:cliff,tab:mou}.

We again find that the performance of $\hat{\rho}_{\mathrm{DRL}(\cm_2)}$ is superior to all other estimators in either setting. This is especially true in Cliff Walking. The estimator $\hat{\rho}_{\mathrm{DRL}(\cm_2)}$ also improves upon $\hat{\rho}_{\mathrm{IS}}$ and $\hat{\rho}_{\mathrm{DM}}$ but not as much as $\hat{\rho}_{\mathrm{DRL}(\cm_2)}$. The estimator $\hat{\rho}_{\mathrm{MIS}}$ offers a slight improvement over $\hat{\rho}_{\mathrm{IS}}$, but is still outperformed by $\hat{\rho}_{\mathrm{DRL}(\cm_2)}$, $\hat{\rho}_{\mathrm{DRL}(\cm_1)}$, and $\hat{\rho}_{\mathrm{DM}}$. That the improvement of $\hat{\rho}_{\mathrm{MIS}}$ over $\hat{\rho}_{\mathrm{IS}}$ and the overall improvements of $\hat{\rho}_{\mathrm{DRL}(\cm_2)}$ is starker in Cliff Walking than in Mountain Car may be attributable to the difficulty of learning $w_t$ nonparametrically in a continuous state space.

\begin{table}[!]
\setlength{\tabcolsep}{0.75em}%
    
    \centering
    {
    \caption{Cliff Walking: RMSE (and standard errors)}
    \begin{tabular}{clllll}\toprule
     Size & \multicolumn{1}{c}{$\hat{\rho}_{\mathrm{IS}}$} & \multicolumn{1}{c}{$\hat{\rho}_{\mathrm{DRL}(\cm_1)}$} & \multicolumn{1}{c}{$\hat{\rho}_{\mathrm{DM}}$} & \multicolumn{1}{c}{$\hat{\rho}_{\mathrm{MIS}}$} & \multicolumn{1}{c}{$\hat{\rho}_{\mathrm{DRL}(\cm_2)}$}  \\ \midrule
    500 &  18.8 (7.67) & 3.78(1.14)  & 2.63 (0.01) &  12.8 (4.96)  & 1.44 (0.29) \\ 
    1000 & 7.99 (0.89) &   0.28 (0.026)  & 1.27 (0.002) & 5.92 (0.78) &   0.22 (0.34) \\
    1500 &  7.64 (1.63) & 0.098 (0.013)  & 1.01 (0.001)  & 5.55 (1.10) & 0.075 (0.008)
    \\\bottomrule
    \end{tabular}
    \label{tab:cliff}
    }\vspace{0.4em}
    \centering{
    \caption{Mountain Car: RMSE (and standard errors)}
    \begin{tabular}{clllll}\toprule
     $n$ & \multicolumn{1}{c}{$\hat{\rho}_{\mathrm{IS}}$} & \multicolumn{1}{c}{$\hat{\rho}_{\mathrm{DRL}(\cm_1)}$} & \multicolumn{1}{c}{$\hat{\rho}_{\mathrm{DM}}$} & \multicolumn{1}{c}{$\hat{\rho}_{\mathrm{MIS}}$} & \multicolumn{1}{c}{$\hat{\rho}_{\mathrm{DRL}(\cm_2)}$}  \\ \midrule
    500 &  6.85 (0.13) &  3.72 (0.08)  & 4.30 (0.05)   & 6.82 (0.12) & 3.53 (0.12)  \\ 
    1000 & 4.73 (0.07) &  2.12 (0.04)  & 3.40 (0.008)   & 4.83 (0.06) & 2.07 (0.04)  \\ 
    1500 &  3.41 (0.04) &  1.82 (0.02)  & 3.30 (0.008)   & 3.40 (0.05) & 1.69 (0.03)  
    \\\bottomrule
    \end{tabular}
    \label{tab:mou}
    }
    \vspace{-0.9em}
\end{table}

\section{Conclusions}

We established the semiparametric efficiency bounds and efficient influence functions for OPE under either NMDP or MDP model, which quantify how fast one could hope to estimate policy value. 
While in the NMDP case, the influence function we derived has appeared frequently in OPE estimators, in the MDP case, the influence function is novel and has not appeared in existing estimators. Our results also suggested how one could construct efficient estimators. We used this to develop DRL, which used our newly derived efficient influence function, with nuisances estimated in a cross-fold manner. This ensured efficiency under very weak and mostly agnostic conditions on the nuisance estimation method used. Notably, DRL is the \emph{first efficient OPE estimator for MDPs}. In addition, DRL enjoyed double robustness properties. This efficiency and robustness translated to better performance in experiments.

\section*{Acknowledgments}

\edit{The authors thank Nan Jiang and Yu-Xiang Wang for helpful discussions.}
This material is based upon work supported by the National Science Foundation under Grant No. 1846210.

\bibliography{pfi}

\newpage 

\onecolumn 

\appendix 

\section{Notation}

We first summarize the notation we use in \cref{tab:pre} and the abbreviations we use in \cref{tab:abbr}. Notice in particular that, following empirical process theory literature, in the proofs we also use $\mathbb P$ to denote expectations (interchangeably with $\rE$).

\begin{table}[h!]
    \centering
      \caption{Notation} \vspace{0.3cm}
    \begin{tabular}{l|l}
    $\nabla_{\beta}$ & Differentiation with respect to $\beta$ \\
    $r_t,s_t,a_t,$  &  Reward, state, action at $t$  \\
    $\mathcal{J}_{r_{t}}$,  $\mathcal{J}_{s_{t}}$, $\mathcal{J}_{a_{t}}$ & History up to time $r_t,s_t,a_t$, including reward variables \\
 $\ch_{s_t}$, $\ch_{a_t}$ & History up to time $s_t,a_t$, excluding reward variables \\
    $\pi_t(a_t|\ch_{s_t}), \pi_t(a_t|s_t)$  & Policy in NMDP and MDP case, respectively \\
    $\pi^{e}_{t}$, $\pi^{b}_{t}$ &  Target and behavior policies  at $t$, respectively \\
    $\rho^\pi$  & Policy value, $\rE_{\pi}[\sum_{t=0}^{T}r_t]$\\
    $v_t=v_t(\ch_{s_t})$,  $v_t(s_t)$ & Value function at $t$, in $\cm_1,\cm_2$ respectively \\
    $q_t=q_t(\ch_{a_t})$,  $q_t(s_t,a_t)$ & $q$-function at $t$, in $\cm_1,\cm_2$ respectively \\ 
    $\lambda_{t} $  &  Cumulative density ratio $\prod_{k=0}^{t}\pi^{e}_t/\pi^{b}_t$ \\ 
    $\mu_{t}$  &  Marginal density ratio $\rE[\lambda_t\mid s_t,a_t]$ \\
    $\eta_{t}$ & Instantaneous density ratio $\pi^{e}_{t}/\pi^{b}_{t}$ \\ 
    $\Lambda$ & Tangent space \\
    $\mathcal{M}$ & A model for the data generating distribution \\
    $\cm_1,\cm_{1b},\cm_{1q}$ & NMDP model with unknown behavior policy, \\
    & known behavior policy, and parametric $q$-function, respectively \\
    $\cm_2,\cm_{2b},\cm_{2q}$ & MDP model with unknown behavior policy, \\
    & known behavior policy, and parametric $q$-function, respectively \\
    $C, R_{\mathrm{max}}$ & Upper bound of density ratio and reward, respectively \\
    $\prod(A|B)$ & Projection of A onto B \\
    $\bigoplus$ & Direct sum \\
    $\|\cdot \|_{p}$ & $L^{p}$-norm $\rE[f^{p}]^{1/p}$\\
    $\lnapprox$ & Inequality up to constant  \\
    $\rE_{\pi}[\cdot],\P_\pi$ & Expectation with respect to a sample from a policy $\pi$ \\
    $\rE[\cdot],\P$ & Same as above for $\pi=\pi^b$   \\
    $\rE_{n}[\cdot],\P_n$ & Empirical expectation (based on sample from a behavior policy) \\
    $n_j$& The size of $\mathcal D_j$\\
    $\rE_{n_j},\P_{n_j}$& Empirical expectation on $\mathcal D_j$\\
    $\bG_n$ & Empirical process $\sqrt{n}(\P_n-\P)$  \\
    $\mathrm{Asmse}[\cdot]$, $\mathrm{var}[\cdot ]$ & Asymptotic variance, variance  \\
    $\mathcal{N}(a,b)$ &  Normal distribution with mean $a$ and variance $b$ \\
    $\mathrm{Uni}[a,b]$ & Uniform distribution on $[a,b]$ \\
    $A_n=\smallo_{p}(a_n)$ & The term $A_n/a_n$ converges to zero in probability \\ 
    $A_n=\bigO_{p}(a_n)$  & The term $A_n/a_n$ is bounded in probability \\
    $\Lambda^{\alpha}_{d}$ & \Holder\, space with smoothness $\alpha$ with a dimension $d$\\
    \end{tabular}
    \label{tab:pre}
\end{table}

\begin{table}[h!]
    \centering
       \caption{Abbreviations}\vspace{0.3cm}
    \begin{tabular}{l|l}
    NMDP     & Non-Markov Decision Process \\
    MDP     & Markov Decision Process\\
    RL  & Reinforcement Learning \\
    CB & Contextual Bandit \\
    OPE & Off policy Evaluation \\
    MLE & Maximum Likelihood Estimation \\
    RAL & Regular and Asymptotic Linear \\ 
    CAN & Consistent and Asymptotically Normal \\ 
    MSE & Mean Squared Error
    \end{tabular}
    \label{tab:abbr}
\end{table}

\newpage
\allowdisplaybreaks

\section{Proofs}
{\blockedit
Before going into details of the proof, we summarize definitions and proofs to derive a semiparametric lower bound. As we mentioned in \cref{sec:semiparam}, for a complete and rigorous treatment, refer to \citet{bickel98,LaanMarkJ.vanDer2003UMfC,BolthausenErwin2002LoPT}. Additional accessible treatments are also given in \citep{TsiatisAnastasiosA2006STaM,Karel2010,BibautAurelien2019Frfe}

\subsection{Semiparametric theory}\label{appendix:semiparam}

We overload notation on \cref{sec:semiparam}. We denote the all of the history $\{\ch^{(i)}\}_{i=1}^{n}$ as $\ch^{n}$, the estimand as $R(F):\cm \to \mathbb{R}$ and the estimator as $\hat R:\ch^{n} \to \mathbb{R}$. First, we introduce some definitions. 

\begin{definition}[One-dimensional submodel and its score function]
A one-dimensional submodel of $\cm$ that passes through $F$ at $0$ is a subset of $\cm$ of the form $\{F_{\epsilon}:\epsilon \in[-a,a]\}$ for some small $a>0$ s.t. $F_{\epsilon=0}=F$. The score of the submodel $F_{\epsilon}$ at $\theta=0$ is defined as 
\begin{align*}
    s(\ch)=\frac{\log (\mathrm{d}F_{\epsilon}/\mathrm{d}\mu)(\ch)}{\mathrm{d}\epsilon}\mid_{\epsilon=0}. 
\end{align*}

\end{definition}

\begin{definition}[Tangent space]
The tangent space of a model $\cm$ at $F$ denoted by $T_{\cm}(F)$ is the linear closure of the set of score functions of the all one-dimensional submodels regarding $\cm$ that pass through $F$. 
\end{definition}

\begin{definition}[Influence function of estimators]
An estimator $\hat R(\ch^{n})$ is asymptotically linear with influence function (IF) $\psi(\ch)$ if 
\begin{align*}
    \sqrt{n}(\hat R(\ch^{n})-R(F))=\frac{1}{\sqrt{n}}\sum_{i=1}^{n}\psi(\ch^{(i)})+\op(1/\sqrt{n}). 
\end{align*}
\end{definition}

\begin{definition}[Pathwise differentiability]
A functional $R(F)$ is pathwise differentiable at $F$ w.r.t the model $\cm$ (or w.r.t the tangent space $\mathcal{T}_{\cm}(F)$) if there exists a function $D_{F}(\ch)$ such that for all submodels $\{F_{\epsilon}:\epsilon\}$ in $\cm$ satisfying $F_{\epsilon=0}=F$ and 
\begin{align*}
    \frac{d R(F_{\epsilon})}{d\epsilon}\mid_{\epsilon=0}=\E[D_{F}(\ch)s(\ch)],
\end{align*}
where $s(\ch)$ is a corresponding score function for $F_{\epsilon}$. The function $D_{F}(\ch)$ is called a gradient of $R(F)$ at $F$ w.r.t the model $\cm$. The efficient IF (EIF) of $R(F)$ w.r.t the model $\cm$ is called a canonical gradient $\tilde D_F(\ch)$, which is the unique gradient of $R(F)$ at $F$ w.r.t the model $\cm$ that belongs to the tangent space $\mathcal{T}_{\cm}(F)$. 
\end{definition}

Next, we define regular estimators. Regular estimators means estimators whose limiting distribution is insensitive to local changes to the data generating process. It excludes a well-known Hodge estimator. 
Here, we denote a submodel with some score function $g$ in a given tangent space $T_{\cm}(F)$ as $\{F_{t,g}; t\in [-a,a]\}$. 

\begin{definition}[Regular estimators]
An estimator sequence $T_n$ is called regular at $F$ for $R(F)$ w.r.t the model $\cm$ (or w.r.t the tangent space $T_{\cm}(F)$), if there exists a probability measure $L$ such that 
\begin{align*}
    \sqrt{n}\{T_n-R(F_{1/\sqrt{n},g})\}\stackrel{d(F_{1/\sqrt{n},g})}{\rightarrow}L, \mathrm{for\, every}\,g\in T_{\cm}(F). 
\end{align*}
\end{definition}

The following three theorems imply that influence functions of the estimators $\hat R(F)$ for $R(F)$ and gradients of $R(F)$ correspond to each other, and how to construct an efficient estimator. These theorems are based on Theorem 3.1 \citep{vaart1991}. 

\begin{theorem}[Influence functions are gradients]
\label{thm:gradients}
Under certain regularity conditions, for $P\in \cm$, suppose $\hat R(\ch^n)$ is a regular estimator of $R(F)$ w.r.t the model $\cm$, and that it is asymptotically linear with influence function $D_F(\ch)$. Then, $R(F)$ is pathwise differentiable at $F$ w.r.t $\cm$ and $D_F(\ch)$ is a gradient of $R(F)$ at $F$ w.r.t $\cm$. 
\end{theorem}

\begin{theorem}[Gradients are influence functions] Under certain regularity conditions, if a $D_F(\ch)$ is a gradient of $R(F)$ at $F$ w.r.t the model $\cm$, there exists an asymptotically linear estimator of $R(F)$ with influence function $D_F(\ch)$, which is regular w.r.t the model $\cm$.  
\end{theorem}

\begin{corollary}[Characterization of efficient influence functions]
The efficient influence function is the projection of any gradient onto the tangent space $\mathcal{T}_{\cm}(F)$. 
\end{corollary}
Note that gradients w.r.t the model $\cm$ are not unique if the model $\cm$ is not a fully nonparametric model. If the underlying model is fully nonparametric model, the gradient is unique. 

\paragraph{Strategy to calculate the EIF}

With the abovementioned definitions and theorems in mind, our general strategy to compute efficient influence functions is as follows. 
\begin{enumerate}
    \item Calculate some gradient $D_F(\ch)$ (a candidate of EIF) of the target functional $R(F)$ w.r.t $\cm$
    \item Calculate the the tangent space $\mathcal{T}_{\cm}(F)$ at $F$
    \item Show that some candidate of EIF is orthogonal to the orthogonal tangent space, i.e., the candidate of EIF lies in the tangent space. Then, this implies that a candidate of EIF is actually the EIF.
\end{enumerate}
The other common strategy is calculating some gradient and projection it onto $\mathcal{T}_{\cm}(F)$. 

\paragraph{Optimalites}

The efficiency bound has the following interpretations. First, the efficiency bound is the lower bound in a local asymptotic minimax sense \citep[Thm.~25.20]{VaartA.W.vander1998As}.
\begin{theorem}[Local Asymptotic Minimax theorem]
\label{thm:lam}
Let $R(F)$ be pathwise diffentiable at $F$ w.r.t the model $\cm$ with EIF $\tilde D_F(\ch)$. If $T_{\cm}(F)$ is a convex cone, for any estimator sequence $\hat R(\ch^n)$, and subconvex loss function $l:\mathbb{R}\to [0,\infty)$,
\begin{align*}
    \sup_{I} \lim_{n\to \infty }\sup_{g\in I}\rE_{F_{1/\sqrt{n},g}}[l[\sqrt{n}\{\hat R(\ch^n)-R(F_{1/\sqrt{n},g})\}]]\geq \int l(u)\mathrm{d}N(0,\var_{F}[\tilde D_F(\ch)])(u), 
\end{align*}
where the first supremum is taken over all finite subsets $I$ of the tangent set. 
\end{theorem}
\begin{corollary}
Under the same assumptions of \cref{thm:lam}, 
\begin{align*}
    \inf_{\delta>0}\liminf_{n\to\infty}\sup_{\|Q-F\|_{T}\leq \delta}\rE_{Q}[l[\sqrt{n}\{\hat R(\ch^n)-R(Q)\}]]\geq \int l(u)\mathrm{d}\mathcal{N}(0,\var_{F}[\tilde D_F(\ch)])(u),  
\end{align*}
where $\|\cdot\|_{T}$ is a total variation distance. 
\end{corollary}

Other different type of optimality is seen in the following theorem. The following theorem state that an asymptotic variance of every regular estimator sequence $\hat R(\ch^n)$ with limiting distribution $L$ is bounded below $\rE[\tilde D^2_F(\ch)]$ \citep[Thm.~25.21]{VaartA.W.vander1998As}. 

\begin{theorem}[Convolution theorem] Let $R(F)$ be pathwise differentiable at $F$ w.r.t the model $\cm$ with EIF $\tilde D_F(\ch)$. Let $\hat R(\ch^n)$ be a regular estimator sequence at $F$ w.r.t the tangent space $\mathcal{T}_{\cm}(F)$ with limiting distribution $L$. Then, if the tangent space $T_{\cm}(F)$ is a cone, then, the term 
\begin{align*}
    \int u^2\mathrm{d}L(u)-\rE[\tilde D^2_F(\ch)] 
\end{align*}
is non-negative. 
\end{theorem}
}

\subsection{Proof}

\begin{proof}[Proof of \cref{thm:fin_nnonpara}]
~
\paragraph*{Efficient influence function under $\cm_1$.}

The entire regular (regular model as defined in Chapter 7 \citealp{VaartA.W.vander1998As}) parametric submodel under $\cm_1$ is 
\begin{align*}
    \{p_{\theta}(s_0)p_{\theta}(a_0|s_0)p_{\theta}(r_0|\ch_{a_0})p_{\theta}(s_1|\ch_{a_0})p_{\theta}(a_1|\ch_{s_1})p_{\theta}(r_1|\ch_{a_1})\cdots p_{\theta}(r_T|\ch_{a_T})\}, 
\end{align*}
where it matches with a true pdf at $\theta=0$. 

{\blockedit The score function of the model $\cm_1$ is decomposed as 
\begin{align*}
    g(\mathcal{J}_{s_T})&=\sum_{k=0}^{T} \nabla \log p_{\theta}(s_{k}|\ch_{a_{k-1}}) +\sum_{k=0}^{T} \nabla \log p_{\theta}(a_{k}|\ch_{s_k}) +\sum_{k=0}^{T} \nabla \log p_{\theta}(r_k|\ch_{a_k})\\ 
    &=\sum_{k=0}^{T}g_{s_{k}|\ch_{a_{k-1}}}+\sum_{k=0}^{T}g_{a_{k}|\ch_{s_k}}+\sum_{k=0}^{T}g_{r_k|\ch_{a_k}}. 
\end{align*}
}

\edit{We first calculate a derivative for the target functional w.r.t the model $\cm_1$}. Note that this derivative is not unique. We have 
{\blockedit
\begin{align*}
& \nabla_{\theta} \mathrm{E}_{\pi^{e}}\left [\sum_{t=0}^{T}r_t \right] \\
& =\nabla_{\theta} \left[\int \sum_{t=0}^{T}r_t \left\{\prod_{k=0}^{T}p_{\theta}(s_k|\ch_{r_{k-1}})p_{\pi^{e}}(a_k|\ch_{s_k})p_{\theta}(r_k|\ch_{a_k}) \right\}\mathrm{d}\mu(\mathcal{J}_{s_T})\right]   \\
    &=\sum_{c=0}^{T} \{\rE_{\pi^{e}}[\{\rE_{\pi^{e}}(r_c|s_0)-\rE_{\pi^{e}}(r_c)\}g_{s_0}]+\rE_{\pi^{e}}[\{r_c-\rE_{\pi^{e}}(r_c|\ch_{a_c})\} g_{r_c|\ch_{a_c}}] \\ 
    &+\rE_{\pi^{e}}\left[\left(\rE_{\pi^{e}}\left [\sum_{t=c+1}^{T} r_t|\ch_{s_{c+1}}\right]-\rE_{\pi^{e}}\left [\sum_{t=c+1}^{T} r_t|\ch_{a_c}\right ]\right)g_{s_{c+1}|\ch_{a_c}}\right]\} \\
    &=\sum_{c=0}^{T}\{\rE_{\pi^{e}}[\{\rE_{\pi^{e}}(r_c|s_0)-\rE_{\pi^{e}}(r_c)\}g(\mathcal{J}_{s_{T+1}})]+\rE_{\pi^{e}}[\{r_c-\rE_{\pi^{e}}(r_c|\ch_{a_c})\} g(\mathcal{J}_{s_T})] \\ 
    &+\rE_{\pi^{e}}\left[\left(\rE_{\pi^{e}}\left [\sum_{t=c+1}^{T} r_t|\ch_{s_c+1}\right]-\rE_{\pi^{e}}\left [\sum_{t=c+1}^{T} r_t|\ch_{a_c}\right ]\right)g(\mathcal{J}_{s_T})\right]\} \\
    &= \rE \left(\left[-\rho^{\pi^{e}}+\sum_{c=0}^{T}\left \{\lambda_{c}r_c- \lambda_{c}\sum_{t=c}^{T}\rE_{\pi_e}(r_t|\ch_{a_c})+\lambda_{c-1}\sum_{t=c}^{T}\rE_{\pi_e}(r_t|\ch_{s_c})\right \}\right]g(\mathcal{J}_{s_T})\right).
\end{align*}
}

This concludes that the following function is a derivative:
\begin{align}
\label{eq:eff_m1}
 \rho^{\cm_1}_{\mathrm{eff}}= -\rho^{\pi^{e}}+\sum_{c=0}^{T}\left \{\lambda_{c}r_c-\lambda_{c}\sum_{t=c}^{T}\rE_{\pi_e}(r_t|\ch_{a_c})-\lambda_{c-1}\sum_{t=c}^{T}\rE_{\pi_e}(r_t|\ch_{s_c})\right \}. 
\end{align}

Next, we show that this derivative is the efficient influence function. In order to show this, we calculate the tangent space of model $\cm_1$. The tangent space of the model $\cm_1$ is the product space:
\begin{align*}
   & \bigoplus_{0\leq t\leq T} (A_t \bigoplus B_t \bigoplus C_t),\\
A_t &= \{q(s_t,\ch_{a_{t-1}}); \rE[q(s_t,\ch_{a_{t-1}})|\ch_{a_{t-1}}]=0,\,q\in L^{2}\}, \\
B_t &= \{q(a_t,\ch_{s_t}); \rE[q(a_t,\ch_{s_t})|\ch_{s_t}]=0,\,q\in L^{2}\}, \\
C_t &= \{q(r_t,\ch_{a_t}); \rE[q(r_t,\ch_{a_t})|\ch_{a_t}]=0,\,q\in L^{2}\}.  \\
\end{align*}
The orthogonal space of the tangent space is the product of 
\begin{align}\label{eq:orth1}
 \bigoplus_{0\leq t\leq T}( A'_t \bigoplus B'_t \bigoplus C'_t)
\end{align}
such that 
\begin{align*}
 A'_t \bigoplus A_t &= A''_t  ,\, A''_t = \{q(\mathcal{J}_{s_{t}}); \rE[q(\mathcal{J}_{s_{t}})|\mathcal{J}_{r_{t-1}}]=0,\,q\in L^{2}\}, \\
B'_t \bigoplus B_t &= B''_t  ,\, B''_t = \{q(\mathcal{J}_{a_{t}}); \rE[q(\mathcal{J}_{a_{t}})|\mathcal{J}_{s_{t}}]=0,\,q\in L^{2}\}, \\ 
C'_t \bigoplus C_t &= C''_t  ,\, C''_t = \{q(\mathcal{J}_{r_{t}}); \rE[q(\mathcal{J}_{r_{t}})|\mathcal{J}_{a_{t}}]=0,\,q\in L^{2}\}. 
\end{align*}
More specifically, we have the following lemma.
\begin{lemma}
The orthogonal tangent space is represented as 
\begin{align*}
A'_t &= \left \{q(\mathcal{J}_{s_{t}})-\rE[ q(\mathcal{J}_{s_{t}})|\ch_{s_t}]; \rE[q(\mathcal{J}_{s_{t}})|\mathcal{J}_{r_{t-1}}]=0,\,q\in L^{2} \right\}, \\
B'_t &= \left \{q(\mathcal{J}_{a_{t}})-\rE[ q(\mathcal{J}_{a_{t}})|\ch_{a_t}]; \rE[q(\mathcal{J}_{a_{t}})|\mathcal{J}_{s_{t}}]=0,\,q\in L^{2} \right\},\\
C'_t &= \left \{q(\mathcal{J}_{r_{t}})-\rE[ q(\mathcal{J}_{r_{t}})|\ch_{a_t},r_t]; \rE[q(\mathcal{J}_{r_{t}})|\mathcal{J}_{a_{t}}]=0,\,q\in L^{2} \right\}.
\end{align*}
\end{lemma}
\begin{proof}
We give a proof for $A'_t$. Regarding the other cases, it is proved similarly.
First, from the definition of the conditional expectation, $A'_t$ and $A_t$ are orthogonal. Thus, what we have to prove is $\rE[ q(\mathcal{J}_{s_{t}})|\ch_{s_t}]$ is included in $A_t$. This is proved as follows:
\[
    \rE[\rE[ q(\mathcal{J}_{s_{t}})|\ch_{s_t}]|\ch_{a_{t-1}}]=\rE[ q(\mathcal{J}_{s_{t}})|\ch_{a_{t-1}}]=\rE[\rE[q(\mathcal{J}_{s_{t}})|\mathcal{J}_{r_{t-1}}]|\ch_{a_{t-1}}]=0.\qedhere
\]
\end{proof}

If we can prove that the influence function \cref{eq:eff_m1} is orthogonal to the orthogonal tangent space \cref{eq:orth1}, we can see that the above derivative is actually the efficient influence function under the model $\cm_1$. This fact is shown as follows. 

\begin{lemma}
\label{lem:orthgonal}
The derivative \cref{eq:eff_m1} is orthogonal to $\{A'_t\}_{t=0}^{T+1}$, $\{B''_t\}_{t=0}^{T}$, $\{C'_t\}_{t=0}^{T}$
\end{lemma}
\begin{proof}
The influence function is orthogonal to $A'_k$: for $t(\mathcal{J}_{s_k}) \in A'_k$
\begin{align*}
    &\rE \left[\left \{-\rho^{\pi^{e}}+\sum_{c=0}^{T}\lambda_{c}(\ch_{a_c})r_c-\left \{\lambda_{c}(\ch_{a_c})\sum_{t=c}^{T}\rE_{\pi^{e}}[r_t|\ch_{a_c}]-\lambda_{c-1}(\ch_{a_{c-1}})\sum_{t=c}^{T}\rE_{\pi^{e}}[r_t|\ch_{s_c}]\right \}\right\}t(\mathcal{J}_{s_k})\right] \\
    &=\rE \left[\left\{\sum_{c=k}^{T}\lambda_{c}(\ch_{a_c})r_c-\lambda_{k-1}\sum_{t=k}^{T}\rE_{\pi^{e}}[r_t|\ch_{s_k}]\right \} t(\mathcal{J}_{s_k})\right] \\
    &=0.
\end{align*}
The influence function is orthogonal to $B''_k$: for $t(\mathcal{J}_{a_k}) \in B''_k$;
\begin{align*}
    & \rE \left[\left\{-\rho^{\pi^{e}}+\sum_{c=0}^{T}\lambda_{c}r_c-\left \{\lambda_{c}\sum_{t=c}^{T}\rE_{\pi^{e}}[r_t|\ch_{a_c}]-\lambda_{c-1}\sum_{t=c}^{T}\rE_{\pi^{e}}[r_t|\ch_{s_c}]\right \}\right\}t(\mathcal{J}_{a_k}) \right] \\
    &= \rE \left[\left\{\sum_{c=k}^{T}\lambda_{c}r_c-\left \{\lambda_{c}\sum_{t=c}^{T}\rE_{\pi^{e}}[r_t|\ch_{a_c}]-\lambda_{c-1}\sum_{t=c}^{T}\rE_{\pi^{e}}[r_t|\ch_{s_c}]\right \}\right\}t(\mathcal{J}_{a_k}) \right]\\
    &= \rE \left[\left\{\left \{\sum_{c=k}^{T}\lambda_{c}r_c\right\}-\left \{\lambda_{k}\sum_{t=k}^{T}\rE_{\pi^{e}}[r_t|\ch_{a_k}]\right\}\right\}t(\mathcal{J}_{a_k}) \right] \\
    &= \rE \left[\left\{\left \{\lambda_{k}\sum_{t=k}^{T}\rE_{\pi^{e}}[r_t|\ch_{a_k}]\right\}-\left \{\lambda_{k}\sum_{t=k}^{T}\rE_{\pi^{e}}[r_t|\ch_{a_k}]\right\}\right\}t(\mathcal{J}_{a_k}) \right] =0.
\end{align*}

The influence function is orthogonal to $C'_k$: for $t(\mathcal{J}_{r_k}) \in C'_k$;
\begin{align*}
    &\rE \left[\left \{-\rho^{\pi^{e}}+\sum_{c=0}^{T}\lambda_{c}(\ch_{a_c})r_c-\left \{\lambda_{c}(\ch_{a_c})\sum_{t=c}^{T}\rE_{\pi^{e}}[r_t|\ch_{a_c}]-\lambda_{c-1}(\ch_{a_{c-1}})\sum_{t=c}^{T}\rE_{\pi^{e}}[r_t|\ch_{s_c}]\right \}\right\}t(\mathcal{J}_{r_k})\right] \\
    &=\rE \left[\left\{\sum_{c=k}^{T}\lambda_{c}(\ch_{a_c})r_c\right \}t(\mathcal{J}_{r_k})\right] \\
     &= \rE \left[\left\{\left \{\lambda_{k-1}\sum_{t=k}^{T}\rE_{\pi^{e}}[r_t|\mathcal{J}_{r_k}]\right\}\right\}t(\mathcal{J}_{r_k})\right]= \rE \left[\left \{\lambda_{k-1}\sum_{t=k}^{T}\rE_{\pi^{e}}[r_t|\ch_{a_k},r_k]\right\}t(\mathcal{J}_{r_k})\right] \\
    &= \rE \left[\left\{\left \{\lambda_{k-1}\sum_{t=k}^{T}\rE_{\pi^{e}}[r_t|\ch_{a_{k}},r_k]\right\}\right\}\rE[t(\mathcal{J}_{r_k})|\ch_{a_{k}},r_k]\right]=0.\qedhere
\end{align*}
\end{proof}

This concludes the proof for $\cm_1$.

\paragraph*{Efficient influence function under $\cm_{1b}$.}

Next, we show that the efficiency bound is still the same even if we know the target policy. To show that, we derive an orthogonal space of the tangent space of the regular parametric submodel:
\begin{align*}
    \{p_{\theta}(s_0)p(a_0|s_0)p_{\theta}(r_0|\ch_{a_0})p_{\theta}(s_1|\ch_{r_0})p(a_1|\ch_{s_1})p_{\theta}(r_1|\ch_{a_1})\cdots p_{\theta}(r_T|\ch_{a_T})\},
\end{align*}
where $p(a_t|\ch_{s_t})$ is fixed at $\pi^{b}_t$. This is equal to 
\begin{align}\label{eq:orth2}
\bigoplus_{0\leq t\leq T}( A'_t \bigoplus B''_t \bigoplus C'_t)
\end{align}
This space \cref{eq:orth2} is orthogonal to the obtained efficient influence function under $\cm_1$. Therefore, the efficient influence function under $\cm_{1b}$ is the same as the one under $\cm_{b}$.

\paragraph*{Efficiency bound.}

We use a law of total variance \citep{CliveG.Bowsher2012Isov} to compute the variance of the efficient influence function. 
\begin{align*}
&\mathrm{var}\left[ \sum_{t=0}^{T}\left(\lambda_{t}r_{t}-\left(\lambda_{t}q_t-\lambda_{t-1}v_t \right) \right)\right]\\
&=\sum_{t=0}^{T+1}\mathrm{E}\left[\mathrm{var}\left(\mathrm{E}\left[\lambda_{t-1}r_{t-1}+\sum_{k=0}^{T}\left(\lambda_{k}r_{k}-\left\{\lambda_{k}q_k-\lambda_{k-1}v_k \right\} \right) |\mathcal{J}_{a_{t}}\right]|\mathcal{J}_{a_{t-1}}\right)\right]\\
&=\sum_{t=0}^{T+1}\mathrm{E}\left[\mathrm{var}\left(\mathrm{E}\left[\lambda_{t-1}r_{t-1}+\sum_{k=t}^{T}\left(\lambda_{k}r_{k}-\left\{\lambda_{k}q_k-\lambda_{k-1}v_k \right\} \right) |\mathcal{J}_{a_{t}}\right]|\mathcal{J}_{a_{t-1}}\right)\right] \\
&=\sum_{t=0}^{T+1}\mathrm{E}\left[\mathrm{var}\left(\mathrm{E}\left[\lambda_{t-1}r_{t-1}+\left(\sum_{k=t}^{T}\lambda_{k}r_{k} \right)-\{\lambda_{t}q_t-\lambda_{t-1}v_t \} |\mathcal{J}_{a_{t}}\right]|\mathcal{J}_{a_{t-1}}\right]\right) \\
&=\sum_{t=0}^{T+1}\mathrm{E}\left[\lambda_{t-1}^{2}\mathrm{var}\left(r_{t-1}+v_t(\ch_{s_t})  |\ch_{a_{t-1}}\right)\right].
\end{align*}
Here, we used $\rE[\sum_{k=t}^{T} \lambda_{k}r_{k}|\mathcal{J}_{a_k}]=\lambda_k q_k$. 
\end{proof}

\begin{proof}[Proof of \cref{thm:fin_nnonpara2}]
~
\paragraph*{Efficient influence function under $\cm_{2}$.}

The entire regular parametric submodel is 
\begin{align*}
    \{p_{\theta}(s_0)p_{\theta}(a_0|s_0)p_{\theta}(r_0|s_0,a_0)p_{\theta}(s_1|s_0,a_0)p_{\theta}(a_1|s_1)p_{\theta}(r_1|s_1,a_1)\cdots p_{\theta}(r_T|s_T,a_T)\}. 
\end{align*}
{\blockedit The score function of the parametric submodel is 
\begin{align*}
   g(\mathcal{J}_{s_T})&=\sum_{k=0}^{T}\nabla_{\theta} \log p_{\theta}(s_k\mid s_{k-1},a_{k-1})+\nabla_{\theta} \log p_{\theta}(a_{k+1}\mid s_{k})+\nabla_{\theta} \log p_{\theta}(r_k\mid s_{k},a_{k})\\
   &=\sum_{k=0}^{T}g_{s_{k}|s_{k-1},a_{k-1}}+\sum_{k=0}^{T}g_{a_{k+1}|s_k}+\sum_{k=0}^{T}g_{r_k|s_k,a_k}. 
\end{align*}
}

\edit{We first calculate a derivative of the target functional w.r.t the model $\cm_2$.} Note that this derivative is not only derivative.
We have 
{\blockedit
\begin{align*}
    &\nabla_{\theta} \mathrm{E}_{\pi^{e}}[\sum_{t=0}^{T}r_t] \\
    & =\nabla_{\theta} \int \sum_{t=0}^{T}r_t \left\{\prod_{t=0}^{T}p_{\theta}(s_k|a_{k-1},s_{k-1})p_{\pi^{e}_k}(a_k|s_k)p_{\theta}(r_k|a_k,s_k) \right\}\mathrm{d}\mu(\mathcal{J}_{s_T})   \\
    &=\sum_{c=0}^{T}\{\rE_{\pi^{e}}\left[(\rE_{\pi^{e}}[r_c|s_0]-\rE_{\pi^{e}}[r_c])g_{s_0}\right]+\rE_{\pi^{e}}[(r_c-\rE_{\pi^{e}}[r_c|s_c,a_c]) g_{r_c|s_c,a_c}] \\ 
    &+\rE_{\pi^{e}}\left[\left(\rE_{\pi^{e}}[\sum_{c=t+1}^{T} r_t|s_{c+1}]-\rE_{\pi^{e}}[\sum_{c=t+1}^{T} r_t|s_{c},a_{c}]\right)g_{s_{c+1}|s_{c},a_{c}}\right]\} \\
    &= \sum_{c=0}^{T}\{\rE[(\rE[r_c|s_0]-\rE_{\pi^{e}}[r_c])g]+\rE\left[\frac{p_{\pi^{e}}(s_c,a_c)}{p_{\pi^{b}}(s_c,a_c)}(r_c-\rE[r_c|s_c,a_c])g\right] \\
    & +\rE\left[\frac{p_{\pi^{e}}(s_c,a_c)}{p_{\pi^{b}}(s_c,a_c)}(\rE[\sum_{t=c+1}^{T} r_t|s_{c+1}]-\rE[\sum_{t=c+1}^{T} r_t|s_{c},a_{c}])g\right]\} \\
    &= \rE \left[\left[-\rho^{\pi^{e}}+\sum_{c=0}^{T}\left \{\frac{p_{\pi^{e}}(s_c,a_c)}{p_{\pi^{b}}(s_c,a_c)}r_c-\frac{p_{\pi^{e}}(s_c,a_{c})}{p_{\pi^{b}}(s_c,a_{c})}\sum_{t=c}^{T}\rE_{\pi_e}[r_t|s_c,a_c]+\frac{p_{\pi^{e}}(s_{c-1},a_{c-1})}{p_{\pi^{b}}(s_{c-1},a_{c-1})}\sum_{t=c}^{T}\rE_{\pi_e}[r_t|s_c]\right \}\right]g(\mathcal{J}_{s_T})\right]
\end{align*}
}

Therefore, the following function is a derivative: 
\begin{align}
\label{eq:eff_m2}
    -\rho^{\pi^{e}_{c}}+\sum_{c=0}^{T}\frac{p_{\pi^{e}_{c}}(s_c,a_c)}{p_{\pi^{b}_{c}}(s_c,a_c)}r_c-\left \{\frac{p_{\pi^{e}_{c}}(s_c,a_{c})}{p_{\pi^{b}_{c}}(s_c,a_{c})}\sum_{t=c}^{T}\rE_{\pi_e}[r_t|s_c,a_c]-\frac{p_{\pi^{e}_{c}}(s_{c-1},a_{c-1})}{p_{\pi^{b}_{c}}(s_{c-1},a_{c-1})}\sum_{t=c}^{T}\rE_{\pi_e}[r_t|s_c]\right \}. 
\end{align}
We will show this derivative is the efficient influence function. 

In order to show this, we calculate the tangent space of model $\cm_2$. The tangent space of the model $\cm_2$ is the product space;
\begin{align*}
   & \bigoplus_{0\leq t\leq T} (A_t \bigoplus B_t \bigoplus C_t),\\
A_t &= \{q(s_{t},s_{t-1},a_{t-1}); \rE[q(s_{t},s_{t-1},a_{t-1})|s_{t-1},a_{t-1}]=0,\,q\in L^{2}\}, \\
B_t &= \{q(a_{t},s_t); \rE[q(a_{t},s_t)|s_t]=0,\,q\in L^{2}\}, \\
C_t &= \{q(r_t,s_t,a_t); \rE[q(r_t,s_t,a_t)|s_t,a_t]=0,\,q\in L^{2}\}.  \\
\end{align*}
The orthogonal space of the tangent space is the product of 
\begin{align}\label{eq:orth}
\bigoplus_{0\leq t\leq T} (A'_t \bigoplus B'_t \bigoplus C'_t),
\end{align}
such that 
\begin{align*}
A'_t \bigoplus A_t &= A''_t  ,\, A''_t = \{q(\mathcal{J}_{s_{t}}); \rE[q(\mathcal{J}_{s_{t}})|\mathcal{J}_{r_{t-1}}]=0,\,q\in L^{2}\},   \\
    B'_t \bigoplus B_t &= B''_t  ,\, B''_t = \{q(\mathcal{J}_{a_{t}}); \rE[q(\mathcal{J}_{a_{t}})|\mathcal{J}_{s_{t}}]=0,\,q\in L^{2}\},   \\
    C'_t \bigoplus C_t &= C''_t,\, C''_t = \{q(\mathcal{J}_{r_{t}}); \rE[q(\mathcal{J}_{r_{t}})|\mathcal{J}_{a_{t}}]=0,\,q\in L^{2}\}.
\end{align*}
More specifically, the orthogonal tangent space is represented as 
\begin{align*}
A'_t &= \left \{q(\mathcal{J}_{s_{t}})-\rE[ q(\mathcal{J}_{s_{t}})|s_t,a_{t-1},s_{t-1}]; \rE[q(\mathcal{J}_{s_{t}})|\mathcal{J}_{r_{t-1}}]=0,\,q\in L^{2} \right\},  \\
B'_t &= \left \{q(\mathcal{J}_{a_{t}})-\rE[ q(\mathcal{J}_{r_{t}})|s_t,a_t]; \rE[q(\mathcal{J}_{a_{t}})|\mathcal{J}_{s_{t}}]=0,\,q\in L^{2} \right\}, \\
C'_t &= \left \{q(\mathcal{J}_{r_{t}})-\rE[ q(\mathcal{J}_{r_{t}})|r_t,s_t,a_t]; \rE[q(\mathcal{J}_{r_{t}})|\mathcal{J}_{a_{t}}]=0,\,q\in L^{2} \right\}.  \\
\end{align*}
If we can prove that the derivative \cref{eq:eff_m2} is orthogonal to the orthogonal tangent space \cref{eq:orth}, we can see that the above derivative is actually the efficient influence function under the model $\cm_2$. This fact is shown as follows. 

\begin{lemma}
\label{lem:orthgonal2}
The derivative \cref{eq:eff_m2} is orthogonal to $\{A'_t\}_{t=0}^{T+1},\{B''_t\}_{t=0}^{T},\{C'_t\}_{t=0}^{T}$. 
\end{lemma}
\begin{proof}

First, the influence function \cref{eq:eff_m2} is orthogonal to $A'_k$; for $t(\mathcal{J}_{s_k})\in A'_k$
\begin{align*}
  &\rE \left[\left \{v_{0}+\sum_{t=0}^{T}\mu_t(s_t,a_t)(r_{t}+v_{t+1}-q_{t})\right \}t(\mathcal{J}_{s_k})\right] \\
&=\rE \left[\left \{\sum_{t=k-1}^{T}\mu_t(s_t,a_t)(r_{t}+v_{t+1}-q_{t})\right \}t(\mathcal{J}_{s_k})\right] \\
&=\rE \left[\mu_{k-1}(s_{k-1},a_{k-1})(r_{k-1}+v_{k}-q_{k-1})t(\mathcal{J}_{s_k})\right] \\
&=\rE \left[\mu_{k-1}(s_{k-1},a_{k-1})v_{k}t(\mathcal{J}_{s_k})\right] \\
&=\rE \left[\mu_{k-1}(s_{k-1},a_{k-1})v_{k}\rE[t(\mathcal{J}_{s_k})|s_k,a_{k-1},s_{k-1}]\right]=0 
\end{align*}

Second, the influence function \cref{eq:eff_m2} is orthogonal to $B''_k$; for $t(\mathcal{J}_{a_k})\in B''_k$
\begin{align*}
  &\rE \left[\left \{v_{0}+\sum_{t=0}^{T}\mu_t(s_t,a_t)(r_{t}+v_{t+1}-q_{t})\right \}t(\mathcal{J}_{a_k})\right] \\
&=\rE \left[\left \{\sum_{t=k}^{T}\mu_t(s_t,a_t)(r_{t}+v_{t+1}-q_{t})\right \}t(\mathcal{J}_{a_k})\right]=0.
\end{align*}

Third, the influence function \cref{eq:eff_m2} is orthogonal to $C'_k$; for $t(\mathcal{J}_{r_k})\in C'_k$
\begin{align*}
  &\rE \left[\left \{v_{0}+\sum_{t=0}^{T}\mu_t(s_t,a_t)(r_{t}+v_{t+1}-q_{t})\right \}t(\mathcal{J}_{r_k})\right] \\
&=\rE \left[\left \{\sum_{t=k}^{T}\mu_t(s_t,a_t)(r_{t}+v_{t+1}-q_{t})\right \}t(\mathcal{J}_{r_k})\right] \\ 
&=\rE \left[\left \{\mu_k(s_k,a_k)(r_{k}+v_{k+1}-q_{k})\right \}t(\mathcal{J}_{r_k})\right] \\
&=\rE \left[\left \{\mu_k(s_k,a_k)(\rE[r_k+\rE[v_{k+1}|\mathcal{J}_{r,k}]-q_{k})\right \}t(\mathcal{J}_{r_k})\right] \\
&=\rE \left[\left \{\mu_k(s_k,a_k)(r_{k}+\rE[v_{k+1}|s_k]-q_{k})\right \}\rE[t(\mathcal{J}_{r_k})|s_k,a_k,r_k]\right]= 0.\qedhere 
\end{align*}

\end{proof}

\paragraph*{Efficient influence function under $\cm_{2b}$.}

In \cref{lem:orthgonal2}, we check that the $\phi^{\cm2}_{\mathrm{eff}}$ is orthogonal to $B"_t$. This concludes the proof noting that the orthogonal tangent space of $\cm_{2b}$ is
\begin{align*}
\bigoplus_{0\leq t\leq T} (A'_t \bigoplus B"_t \bigoplus C'_t). 
\end{align*}

\paragraph*{Efficiency bound.}

To show an efficiency bound, we use a law of total variance \citep{CliveG.Bowsher2012Isov}. Recall that 
we can also easily derive this variance form using another equivalent form of efficient influence function.
\begin{align*}
&\mathrm{var}\left[v_{0}+\sum_{t=0}^{T}\mu_t(s_t,a_t)(r_{t}+v_{t+1}-q_{t})\right]\\
&=\sum_{t=0}^{T+1}\rE\left[\mathrm{var}\left[ \mathrm{E}\left[v_{0}+\sum_{t=0}^{T}\mu_k(s_k,a_k)(r_{k}+v_{k+1}-q_{k})|\mathcal{J}_{a_t}\right]|\mathcal{J}_{a_{t-1}}\right]\right]\\
&=\sum_{t=0}^{T+1}\rE\left[\mathrm{var}\left[ \mathrm{E}\left[\sum_{k=t-1}^{T}\mu_k(s_k,a_k)(r_{k}+v_{k+1}-q_{k})|\mathcal{J}_{a_t}\right]|\mathcal{J}_{a_{t-1}}\right]\right]\\
&=\sum_{t=0}^{T+1}\rE\left[\mathrm{var}\left[ \rE\left[\mu_{t-1}(s_{t-1},a_{t-1})(r_{t-1}+v_{t}-q_{t-1})|\mathcal{J}_{a_t}\right]|\mathcal{J}_{a_{t-1}}\right]\right]\\
&=\sum_{t=0}^{T+1}\rE\left[\mathrm{var}\left[ \mu_{t-1}(s_{t-1},a_{t-1})(r_{t-1}+v_{t}-q_{t-1})|\mathcal{J}_{a_{t-1}}\right]\right]\\
&=\sum_{t=0}^{T+1}\rE\left[\mu^{2}_{t-1}(s_{t-1},a_{t-1})\mathrm{var}\left[ (r_{t-1}+v_{t}(s_t))|\mathcal{J}_{a_{t-1}}\right]\right].
\end{align*}

\end{proof}

\begin{proof}[Proof of \cref{cor:jensen}]

From Jensen's inequality, 
\begin{align*}
    &\sum_{t=0}^{T+1}\mathrm{E}\left[\lambda^{2}_{t-1}\mathrm{var}\left \{r_t+v_t(s_t)  |s_{t-1},a_{t-1}\right\}\right]=  \sum_{t=0}^{T+1}\mathrm{E}\left[\rE(\lambda^{2}_{t-1}|s_{t-1},a_{t-1})\mathrm{var}\left\{r_t+v_t(s_t)|s_{t-1},a_{t-1}\right \}\right] \\
    &\geq \sum_{t=0}^{T+1}\mathrm{E}\left[\rE(\lambda_{t-1}|s_{t-1},a_{t-1})^{2}\mathrm{var}\left\{r_t+v_t(s_t)|s_{t-1},a_{t-1}\right \}\right]= \sum_{t=0}^{T+1}\mathrm{E}\left[\mu^{2}_{t-1}\mathrm{var}\left\{r_t+v_t(s_t)|s_{t-1},a_{t-1}\right \}\right]. 
\end{align*}
When $\lambda^{2}_{t-1}$ is not constant given $s_{t-1},a_{t-1}$ and $\mathrm{var}\left\{r_t+v_t(s_t)|s_{t-1},a_{t-1}\right \}\neq0$, the inequality is strict. 
\end{proof}

{\blockedit
\begin{proof}[Proof of \cref{thm:horizon}]

By changing the limits of summation and letting $r_{-1}=0,\,\lambda_0=1$, we can write
the efficiency bound under NMDP as 
\begin{align*}
   \sum_{t=0}^{T+1}\mathrm{E}\left[\lambda^{2}_{t-1}\mathrm{var}\left\{r_{t-1}+v_t(\ch_{s_t})\mid \ch_{a_{t-1}}\right \}\right] & \leq   C^{T+1}\sum_{t=0}^{T+1}\mathrm{E}\left[\lambda_{t-1}\mathrm{var}\left\{r_{t-1}+v_t(\ch_{s_t})|\ch_{a_{t-1}}\right \}\right]  \\
  & = C^{T+1}\sum_{t=0}^{T+1}\mathrm{E}_{{\epol}}\left[\mathrm{var}_{{\epol}}\left\{r_{t-1}+v_t(\ch_{s_t})|\ch_{a_{t-1}}\right \}\right]  \\
  & = C^{T+1}\sum_{t=0}^{T+1}\mathrm{E}_{{\epol}}\left[\mathrm{var}\left\{r_{t-1}+v_t(\ch_{s_t})|\ch_{a_{t-1}}\right \}\right]  \\
  & = C^{T+1}\var[  \sum_{t=0}^{T+1}r_{t-1}]  \\
  & \leq C^{T+1}(T+1)^2R^2_{\max}. 
\end{align*}
The last equality follows by the law of total variance.

Similarly, the efficiency bound under MDP is 
\begin{align*}
   \sum_{t=0}^{T+1}\mathrm{E}\left[\mu^{2}_{t-1}\mathrm{var}\left\{r_{t-1}+v_t(s_t)\mid s_{t-1},a_{t-1}\right \}\right] & \leq   C'\sum_{t=0}^{T+1}\mathrm{E}\left[\mu_{t-1}\mathrm{var}\left\{r_{t-1}+v_t(s_t)\mid s_{t-1},a_{t-1}\right \}\right]  \\
  & = C'\sum_{t=0}^{T+1}\mathrm{E}_{{\epol}}\left[\mathrm{var}\left\{r_{t-1}+v_t(s_t)\mid s_{t-1},a_{t-1}\right \}\right]  \\
  & = C'\sum_{t=0}^{T+1}\mathrm{E}_{{\epol}}\left[\mathrm{var}_{{\epol}}\left\{r_{t-1}+v_t(s_t)\mid s_{t-1},a_{t-1}\right \}\right]  \\
  & = C'\var[\sum_{t=0}^{T+1}r_{t-1}]  \\
  & \leq C'(T+1)^2R^2_{\max}. 
\end{align*}
The last equality again follows by the law of total variance.

Finally, for the NMDP lower bound we have by Jensen's inequality
\begin{align*}
   \sum_{t=0}^{T+1}\mathrm{E}\left[\lambda^{2}_{t-1}\mathrm{var}\left\{r_{t-1}+v_t(\ch_{s_t})\mid \ch_{a_{t-1}}\right \}\right] & = \sum_{t=0}^{T+1}\mathrm{E}_{\epol}\left[\lambda_{t-1}\mathrm{var}\left\{r_{t-1}+v_t(\ch_{s_t})\mid \ch_{a_{t-1}}\right \}\right]\\
   &\geq
   \sum_{t=0}^{T+1}\exp\mathrm{E}_{\epol}\left[\log(\lambda_{t-1}\mathrm{var}\left\{r_{t-1}+v_t(\ch_{s_t})\mid \ch_{a_{t-1}}\right \})\right]
   \\
   & \geq \sum_{t=0}^{T+1}\exp(\mathrm{E}_{\epol}\left[\log(\lambda_{t-1})\right]+\mathrm{E}_{\epol}\left[\mathrm{var}\left\{r_{t-1}+v_t(\ch_{s_t})\mid \ch_{a_{t-1}}\right \})\right]) \\
   &\geq
   \sum_{t=0}^{T+1}\exp(t\log C_{\min}+\log V_{\min}^2)
   \\&\geq V_{\min}^2C_{\min}^{T+1}.
\end{align*}
\end{proof}
}

\begin{proof}[Proof of \cref{thm:double_m1}]
{\blockedit
Without loss of generality, we consider the case $K=2$. 
}
Define $\phi(\{\hat{\lambda}_{k}\},\{\hat{q}_k\})$ as:
\begin{align*}
    \sum_{k=0}^{T}\hat{\lambda}_{k}r_k-\{\hat{\lambda}_{k} \hat{q}_k-\hat{\lambda}_{k-1}\rE_{\pi^{e}}[\hat{q}_k(\ch_{a_k})|\ch_{s_k}]\}. 
\end{align*}
The estimator $\hat{\rho}^{\cm1}_{\mathrm{DRL}(\cm_1)}$ is given by
\begin{align*}
\frac{n_1}{n}\P_{n_1}\phi(\{\hat{\lambda}^{(1)}_{k}\},\{\hat{q}^{(1)}_k\})+\frac{n_2}{n}\P_{n_2}\phi(\{\hat{\lambda}^{(2)}_{k}\},\{\hat{q}^{(2)}_k\}),
\end{align*}
{\blockedit
where $\P_{n_1}$ is an empirical approximation based on a set of samples such that $i \in \mathcal{D}_1$, $\P_{n_2}$ is an empirical approximation based on a set of samples such that $i \in \mathcal{D}_2$. 
}
Then, we have 
\begin{align}
    \sqrt{n}(\P_{n_1}\phi(\{\hat{\lambda}^{(1)}_{k}\},\{\hat{q}^{(1)}_k\})-\rho^{\pi^{e}})&= \sqrt{n/n_1}\bG_{n_1}[\phi(\{\hat{\lambda}^{(1)}_{k}\},\{\hat{q}^{(1)}_k\})-\phi(\{\lambda_{k}\},\{q_k\})]
    \label{eq:term1_drl}
    \\
    &+\sqrt{n/n_1}\bG_{n_1}[\phi(\{\lambda^{(1)}_{k}\},\{q^{(1)}_k\}) ] 
       \label{eq:term1_dr2}
    \\
    &+\sqrt{n}(\rE[\phi(\{\hat{\lambda}^{(1)}_{k}\},\{\hat{q}^{(1)}_k\})|\{\hat{\lambda}^{(1)}_{k}\},\{\hat{q}^{(1)}_k\} ]-\rho^{\pi^{e}} ).    \label{eq:term1_dr3}
\end{align}
We analyze each term. To do that, we use the following relation: {\blockedit 
\begin{align*}
&\phi(\{\hat{\lambda}_{k}\},\{\hat{q}_k\})-\phi(\{\lambda_{k}\},\{{q}_k\})=D_1+D_2+D_3,\quad\text{where}\\
&D_1 =\sum_{k=0}^{T}(\hat{\lambda}_{k}-\lambda_{k})(-\hat{q}_k+q_k)+(\hat{\lambda}_{k-1}-\lambda_{k-1})(\hat{v}_k-v_k), \\
&D_2 =\sum_{k=0}^{T}\lambda_{k}(-\hat{q}_k+q_k)+\lambda_{k-1}(\hat{v}_k-v_k), \\
&D_3 =\sum_{k=0}^{T}(\hat{\lambda}_{k}-\lambda_{k} )(r_k-q_k+v_{k+1}).
\end{align*}
}

First, we show the term \cref{eq:term1_drl} is $\smallo_{p}(1)$. 
\begin{lemma}\label{lem:op11}
The term \cref{eq:term1_drl} is $\smallo_{p}(1)$.
\end{lemma}
\begin{proof}
If we can show that for any $\epsilon>0$, 
\begin{align}
\label{eq:process_m1}
\lim_{n\to \infty}\sqrt{n_1}P[&\P_{n_1}[\phi(\{\hat{\lambda}^{(1)}_{k}\},\{\hat{q}^{(1)}_k\})-\phi(\{\lambda_{k}\},\{{q}_k\})]\\&\notag- 
 \rE[\phi(\{\hat{\lambda}^{(1)}_{k}\},\{\hat{q}^{(1)}_k\})-\phi(\{\lambda_{k}\},\{{q}_k\})|\{\hat{\lambda}^{(1)}_{k}\},\{\hat{q}^{(1)}_k\} ]>\epsilon|\mathcal{D}_2]=0, %
\end{align}
Then, by bounded convergence theorem, we would have 
\begin{align*}
  \lim_{n\to \infty}\sqrt{n_1}P[&\P_{n_1}[\phi(\{\hat{\lambda}^{(1)}_{k}\},\{\hat{q}^{(1)}_k\})-
  \phi(\{\lambda_{k}\},\{{q}_k\})]\\&-\rE[\phi(\{\hat{\lambda}^{(1)}_{k}\},\{\hat{q}^{(1)}_k\})-\phi(\{\lambda_{k}\},\{{q}_k\})|\{\hat{\lambda}^{(1)}_{k}\},\{\hat{q}^{(1)}_k\}]>\epsilon]=0,
\end{align*}
yielding the statement. 

To show \cref{eq:process_m1}, we show that the conditional mean is $0$ and conditional variance is $\smallo_p(1)$. 
The conditional mean is 
\begin{align*}
&\rE[\P_{n_1}[\phi(\{\hat{\lambda}^{(1)}_{k}\},\{\hat{q}^{(1)}_k\})-\phi(\{\lambda_{k}\},\{{q}_k\})|\{\hat{\lambda}^{(1)}_{k}\},\{\hat{q}^{(1)}_k\}]\\&\qquad\qquad\qquad\qquad\qquad-\P[\phi(\{\hat{\lambda}^{(1)}_{k}\},\{\hat{q}^{(1)}_k\})-\phi(\{\lambda_{k}\},\{{q}_k\})]|\mathcal{D}_2]
=0.
\end{align*}
Here, we leveraged the sample splitting construction, that is, $\hat{\lambda}^{(1)}_{k}$ and $\hat{q}^{(1)}_{k}$ only depend on $\mathcal{D}_2$. The conditional variance is 
\begin{align*}
&\mathrm{var}[\sqrt{n_1}\P_{n_1}[\phi(\{\hat{\lambda}^{(1)}_{k}\},\{\hat{q}^{(1)}_k\})-\phi(\{\lambda_{k}\},\{{q}_k\})]|\mathcal{D}_2]\\
&\qquad\qquad= \rE[\rE[D^{2}_1+D^{2}_2+D^{2}_3+2D_{1}D_{2}+2D_{2}D_{3}+2D_{2}D_{3} |\{\hat{q}^{(1)}_k\},\{\lambda^{(1)}_{k}\} ]|\mathcal{D}_2] \\ 
&\qquad\qquad=\smallo_p(1).
\end{align*}
Here, we used the convergence rate assumption and the relation $\|\hat{v}^{(1)}_{k}-v_{k}\|_{2}\leq\|\hat{q}^{(1)}_{k}-q_{k}\|_{2}$ arising from the fact that the former is the marginalization of the latter over $\pi^e_k$. 
Then, from Chebyshev's inequality:
\begin{align*}
&\sqrt{n_1}P[\P_{n_1}[\phi(\{\hat{\lambda}^{(1)}_{k}\},\{\hat{q}^{(1)}_k\})-\phi(\{\lambda_{k}\},\{{q}_k\})]-\rE[\phi(\{\hat{\lambda}^{(1)}_{k}\},\{\hat{q}^{(1)}_k\})\\
&\qquad\qquad\qquad-\phi(\{\lambda_{k}\},\{{q}_k\})|\{\hat{\lambda}^{(1)}_{k}\},\{\hat{q}^{(1)}_k\} ]>\epsilon|\mathcal{D}_2] \\
&\leq \frac{1}{\epsilon^{2}}\mathrm{var}[\sqrt{n_1}\P_{n_1}[\phi(\{\hat{\lambda}^{(1)}_{k}\},\{\hat{q}^{(1)}_k\})-\phi(\{\lambda_{k}\},\{{q}_k\})]|\mathcal{D}_2]=\smallo_{p}(1).\qedhere
\end{align*}
\end{proof}

\begin{lemma}\label{lem:op12}
The term \cref{eq:term1_dr3} is $\smallo_{p}(1)$.
\end{lemma}
\begin{proof}
\begin{align*}
&\sqrt{n}\rE[\phi(\{\hat{\lambda}^{(1)}_{k}\},\{\hat{q}^{(1)}_k\})-\rE[\phi(\{\lambda_{k}\},\{{q}_k\})]|\{\hat{\lambda}^{(1)}_{k}\},\{\hat{q}^{(1)}_k\}]\\ 
&=\sqrt{n}\rE[\sum_{k=0}^{T}(\hat{\lambda}^{(1)}_{k}-\lambda_{k} )(-\hat{q}^{(1)}_k+q_k)+(\hat{\lambda}^{(1)}_{k-1}-\lambda_{k-1})(-\hat{v}^{(1)}_k+v_k)|\{\hat{\lambda}^{(1)}_{k}\},\{\hat{q}^{(1)}_k\}]\\
    &+\sqrt{n}\rE[\sum_{k=0}^{T}\lambda_{k}(-\hat{q}^{(1)}_k+q_k)+\lambda_{k-1}(\hat{v}^{(1)}_k-v_k)|\{\hat{\lambda}^{(1)}_{k}\},\{\hat{q}^{(1)}_k\}] \\
    &+\sqrt{n}\rE[\sum_{k=0}^{T}(\hat{\lambda}^{(1)}_{k}-\lambda_{k} )(r_k-q_k+v_{k+1})|\{\hat{\lambda}^{(1)}_{k}\},\{\hat{q}^{(1)}_k\}] \\
&=\sqrt{n}\rE[\sum_{k=0}^{T}(\hat{\lambda}^{(1)}_{k}-\lambda_{k} )(-\hat{q}^{(1)}_k+q_k)+(\hat{\lambda}^{(1)}_{k-1}-\lambda_{k-1})(-\hat{v}_k+v_k)|\{\hat{\lambda}^{(1)}_{k}\},\{\hat{q}^{(1)}_k\}] \\
&= \sqrt{n}\sum_{k=0}^{T} \bigO(\|\hat{\lambda}^{(1)}_{k}-\lambda_{k}\|_{2}\|\hat{q}^{(1)}_{k}-q_{k}\|_{2}) =\sqrt{n}\sum_{k=0}^{T}\smallo_{p}(n^{-1/2})=\smallo_{p}(1). \qedhere
\end{align*}
\end{proof}

Finally, we get 
\begin{align*}
    \sqrt{n}(\P_{n_1}\phi(\{\hat{\lambda}^{(1)}_{k}\},\{\hat{q}^{(1)}_k\})-\rho^{\pi^{e}})=\sqrt{n/n_1}\bG_{n_1}[\phi(\{\lambda_{k}\},\{q_k\}) ]+\smallo_{p}(1).
\end{align*}
Therefore, 
\begin{align*}
    &\sqrt{n}(\hat{\rho}^{\pi^e}_{\mathrm{DRL}(\cm_1)}-\rho^{\pi^{e}})\\
    &=  n_1/n\sqrt{n}\phi(\{\hat{\lambda}^{(1)}_{k}\},\{\hat{q}^{(1)}_k\})-\rho^{\pi^{e}})+ n_2/n\sqrt{n}(\P_{n_2}\phi(\{\hat{\lambda}^{(2)}_{k}\},\{\hat{q}^{(2)}_k\})-\rho^{\pi^{e}}) \\
    &=\sqrt{n_1/n}\bG_{n_1}[\phi(\{\lambda_{k}\},\{q_k\})  ]+\sqrt{n_2/n}\bG_{n_2}[\phi(\{\lambda_{k}\},\{q_k\}) ]+\smallo_{p}(1)\\
    &=\bG_{n}[\phi(\{\lambda_{k}\},\{q_k\}) ] +\smallo_{p}(1),
\end{align*}
concluding the proof by showing the influence function of $\hat{\rho}^{\pi^e}_{\mathrm{DRL}(\cm_1)}$ is the efficient one.
\end{proof}

\begin{proof}[Proof of \cref{thm:effcor_m1}]
Here, $a \lessapprox b$ means there exists constant $Q$ not depending on $T,R_{\mathrm{max}},C,n,C_1,C_2$ such that $a<Qb$. 

Define $\phi(\{\hat{\lambda}_{k}\},\{\hat{q}_k\})$ as:
\begin{align*}
    \sum_{k=0}^{T}\hat{\lambda}_{k}r_k-\{\hat{\lambda}_{k} \hat{q}_k-\hat{\lambda}_{k-1}\rE_{\pi^{e}}[\hat{q}_k(\ch_{a_k})|\ch_{s_k}]\}. 
\end{align*}
The estimator $\hat{\rho}^{\cm1}_{\mathrm{DRL}(\cm_1)}$ is given by
\begin{align*}
\frac{n_1}{n}\P_{n_1}\phi(\{\hat{\lambda}^{(1)}_{k}\},\{\hat{q}^{(1)}_k\})+\frac{n_2}{n}\P_{n_2}\phi(\{\hat{\lambda}^{(2)}_{k}\},\{\hat{q}^{(2)}_k\}),
\end{align*}
{\blockedit
where $\P_{n_1}$ is an empirical approximation based on a set of samples such that $i \in \mathcal{D}_1$, $\P_{n_2}$ is an empirical approximation based on a set of samples such that $i \in \mathcal{D}_2$. 
}

Then, we have 
\begin{align}
(\P_{n_1}\phi(\{\hat{\lambda}^{(1)}_{k}\},\{\hat{q}^{(1)}_k\})-\rho^{\pi^{e}})&= \sqrt{1/n_1}\bG_{n_1}[\phi(\{\hat{\lambda}^{(1)}_{k}\},\{\hat{q}^{(1)}_k\})-\phi(\{\lambda_{k}\},\{q_k\})]
    \label{eq:term1_finite_drl}
    \\
    &+\sqrt{1/n_1}\bG_{n_1}[\phi(\{\lambda_{k}\},\{q_k\}) ] 
       \label{eq:term1_finite_dr2}
    \\
    &+(\rE[\phi(\{\hat{\lambda}^{(1)}_{k}\},\{\hat{q}^{(1)}_k\})|\{\hat{\lambda}^{(1)}_{k}\},\{\hat{q}^{(1)}_k\} ]-\rho^{\pi^{e}} ).    \label{eq:term1_finite_dr3}
\end{align}
We analyze each term. To do that, we use the following relation: {\blockedit 
\begin{align*}
&\phi(\{\hat{\lambda}_{k}\},\{\hat{q}_k\})-\phi(\{\lambda_{k}\},\{{q}_k\})=D_1+D_2+D_3,\quad\text{where}\\
&D_1 =\sum_{k=0}^{T}(\hat{\lambda}_{k}-\lambda_{k})(-\hat{q}_k+q_k)+(\hat{\lambda}_{k-1}-\lambda_{k-1})(\hat{v}_k-v_k), \\
&D_2 =\sum_{k=0}^{T}\lambda_{k}(-\hat{q}_k+q_k)+\lambda_{k-1}(\hat{v}_k-v_k), \\
&D_3 =\sum_{k=0}^{T}(\hat{\lambda}_{k}-\lambda_{k} )(r_k-q_k+v_{k+1}).
\end{align*}
}

\begin{lemma}\label{lem:op11_3}
With probability $1-2\delta$, the absolute value of the term \cref{eq:term1_finite_drl} is bounded by 
\begin{align*}
\sqrt{\frac{\log(2/\delta)T^{2}(TR_{\mathrm{max}}C^{T+1}\sqrt{\kappa_1\kappa_2}+\kappa_1T^2R^2_{\mathrm{max}}+\kappa_2C^{2(T+1)})}{n}}+    \frac{\log(2/\delta)TR_{\mathrm{max}}C^{T+1}}{n},
\end{align*}
up to some constant independent of $n,C,R_{\mathrm{max}},T$.
\end{lemma}
\begin{proof}
From Bernstein inequality, with probability $1-\delta$, 
\begin{align*}
P[&|\P_{n_1}[\phi(\{\hat{\lambda}^{(1)}_{k}\},\{\hat{q}^{(1)}_k\})-
  \phi(\{\lambda_{k}\},\{{q}_k\})]\\&-\rE[\phi(\{\hat{\lambda}^{(1)}_{k}\},\{\hat{q}^{(1)}_k\})-\phi(\{\lambda_{k}\},\{{q}_k\})|\mathcal D_1]|>\epsilon]|\mathcal{D}_2] \\ 
  &\leq 2\exp\left (-\frac{0.5n\epsilon^2}{\mathrm{E}[\{\phi(\{\hat{\lambda}^{(1)}_{k}\},\{\hat{q}^{(1)}_k\})-\phi(\{\lambda_{k}\},\{{q}_k\})\}^{2}|\mathcal{D}_2]  +Q_1TR_{\mathrm{max}}C^{T+1}\epsilon } \right),
\end{align*}
noting the conditional mean is $0$; 
\begin{align*}
&\rE[\P_{n_1}[\phi(\{\hat{\lambda}^{(1)}_{k}\},\{\hat{q}^{(1)}_k\})-\phi(\{\lambda_{k}\},\{{q}_k\})|\mathcal D_1]\\&\qquad\qquad\qquad\qquad\qquad-\P[\phi(\{\hat{\lambda}^{(1)}_{k}\},\{\hat{q}^{(1)}_k\})-\phi(\{\lambda_{k}\},\{{q}_k\})]|\mathcal{D}_2]
=0.
\end{align*}
Here, $Q_1$ is some constant independent of $n,C,R_{\mathrm{max}},T,\delta$. 

With probability $1-\delta$, the conditional variance is 
\begin{align*}
&\mathrm{E}[\{\phi(\{\hat{\lambda}^{(1)}_{k}\},\{\hat{q}^{(1)}_k\})-\phi(\{\lambda_{k}\},\{{q}_k\})\}^{2}|\mathcal{D}_2]  \\
&=\rE[D^{2}_1+D^{2}_2+D^{2}_3+2D_{1}D_{2}+2D_{2}D_{3}+2D_{2}D_{3} |\mathcal{D}_2] \\ 
& \lessapprox T^{2}(TR_{\mathrm{max}}C^{T+1}\sqrt{\kappa_1\kappa_2}+\kappa_1(TR_{\mathrm{max}})^{2}+\kappa_2C^{2(T+1)}).
\end{align*}
Therefore, with probability $1-2\delta$, 
\begin{align*}
P[&|\P_{n_1}[\phi(\{\hat{\lambda}^{(1)}_{k}\},\{\hat{q}^{(1)}_k\})-
  \phi(\{\lambda_{k}\},\{{q}_k\})]\\&-\rE[\phi(\{\hat{\lambda}^{(1)}_{k}\},\{\hat{q}^{(1)}_k\})-\phi(\{\lambda_{k}\},\{{q}_k\})|\mathcal D_1]|>\epsilon]|\mathcal{D}_2] \\ 
  &\leq 2\exp\left (-\frac{0.5n_1\epsilon^2}{Q_2T^{2}(TR_{\mathrm{max}}C^{T+1}\sqrt{\kappa_1\kappa_2}+\kappa_1T^2R^2_{max}+\kappa_2C^{2(T+1)}) +Q_1TR_{\mathrm{max}}C^{T+1}\epsilon } \right).     
\end{align*}
Here, $Q_2$ is some constant independent of $n,C,R_{\mathrm{max}},T,\delta$. Then, by the law of total probability and some algebra, the statement is concluded. 
\end{proof}

\begin{lemma}\label{lem:op12_3}
With probability $1-\delta$, the absolute value of the term \cref{eq:term1_finite_dr3} is bounded by 
\begin{align*}
T\sqrt{\kappa_1\kappa_2}. 
\end{align*}
up to some constant independent of $n,C,R_{\mathrm{max}},T,\delta$.
\end{lemma}
\begin{proof}
Here, we have 
\begin{align*}
\rE[\phi(\{\hat{\lambda}^{(1)}_{k}\},\{\hat{q}^{(1)}_k\})-\rE[\phi(\{\lambda_{k}\},\{{q}_k\})]|\{\hat{\lambda}^{(1)}_{k}\},\{\hat{q}^{(1)}_k\}]\leq \sum_{k=0}^{T} 2\|\hat{\lambda}^{(1)}_{k}-\lambda_{k}\|_{2}\|\hat{q}^{(1)}_{k}-q_{k}\|_{2}.
\end{align*}
Then, with probability $1-\delta$, this is bounded by $T\sqrt{\kappa_1\kappa_2}$. 
\end{proof}

Combining all results so far, with probability $1-6\delta$, we have 
\begin{align*}
    &\left|\frac{n_1}{n}\P_{n_1}\phi(\{\hat{\lambda}^{(1)}_{k}\},\{\hat{q}^{(1)}_k\})+\frac{n_2}{n}\P_{n_2}\phi(\{\hat{\lambda}^{(2)}_{k}\},\{\hat{q}^{(2)}_k\})-\rho^{\pi_e}\right|\\
    &< Q_3T\sqrt{\kappa_1\kappa_2}+Q_2\sqrt{\frac{\log(2/\delta)T^{2}(TR_{\mathrm{max}}C^{T+1}\sqrt{\kappa_1\kappa_2}+\kappa_1T^2R^2_{\mathrm{max}}+\kappa_2C^{2(T+1)})}{n}}+    \frac{Q_1\log(2/\delta)TR_{\mathrm{max}}C^{T+1}}{n} \\
    &+\P_{n_1}[\phi(\{\lambda_{k}\},\{q_k\}) ] +\P_{n_2}[\phi(\{\lambda_{k}\},\{q_k\}) ]-\rho^{\pi_e} . 
\end{align*}
Here, $Q_3$ is some constant independent of $n,C,R_{\mathrm{max}},T$.
Noting $|\P_{n}[\phi(\{\lambda_{k}\},\{q_k\})]-\rho^{\pi_e}|$ is bounded by 
\begin{align*}
    \sqrt{\frac{2\log(2/\delta)\mathrm{Effbd(\cm_1)}}{n}}+\frac{Q_1\log(2/\delta)TR_{\mathrm{max}}C^{T+1}}{n}
\end{align*}
with probability $1-\delta$, the statement is concluded. 
\end{proof}

\begin{proof}[Proof of \cref{thm:double_m1_sam}]

We define $\phi(\{\hat{\lambda}_{k}\},\{\hat{q}_k\})$ as:
\begin{align*}
    \sum_{k=0}^{T}\hat{\lambda}_{k}r_k-\hat{\lambda}_{k-1}\{ \hat{\eta}_{k}\hat{q}_k-\rE_{\pi^{e}}[\hat{q}_k(\ch_{a_k})|\ch_{s_k}]\}. 
\end{align*}
The estimator $\hat{\rho}^{\pi^e}_{\mathrm{DR}}$ is given by $
\P_{n}\phi(\{\hat{\lambda}_{k}\},\{\hat{q}_k\})$.     
Then, we have 
\begin{align}
    \sqrt{n}(\P_{n}\phi(\{\hat{\lambda}_{k}\},\{\hat{q}_k\})-\rho^{\pi^{e}})&= \bG_{n}[\phi(\{\hat{\lambda}_{k}\},\{\hat{q}_k\})-\phi(\{\lambda_{k}\},\{q_k\})]
    \label{eq:term1_drl_sam1}
    \\
    &+\bG_{n}[\phi(\{\lambda_{k}\},\{q_k\}) ] 
       \label{eq:term1_dr1_sam2}
    \\
    &+\sqrt{n}(\rE[\phi(\{\hat{\lambda}_{k}\},\{\hat{q}_k\})|\{\hat{\lambda}_{k}\},\{\hat{q}_k\} ]-\rho^{\pi^{e}} ).    \label{eq:term1_dr1_sam3}
\end{align}
If we can prove that the term \cref{eq:term1_drl_sam1} is $\op(1)$, the statement is concluded as in the proof of Theorem \ref{thm:double_m1}. We proceed to prove this.

First, we show $\phi(\{\hat{\lambda}_{k}\},\{\hat{q}_t\})-\phi(\{\lambda_{t}\},\{q_t\})$ belongs to a Donsker class. The transformation
\begin{align*}
   (\{\lambda_{k}\},\{q_k\} ) \mapsto    \sum_{k=0}^{T}\lambda_{k}r_k-\{ \lambda_{k}q_k-\lambda_{k-1}\rE_{\pi^{e}}[q_k(\ch_{a_k})|\ch_{s_k}]\}
\end{align*}
is a Lipschitz function. Therefore, by Example 19.20 in \citet{VaartA.W.vander1998As}, $\phi(\{\hat{\lambda}_{k}\},\{\hat{q}_k\})-\phi(\{\lambda_{k}\},\{q_k\})$ is an also Donsker class. In addition, we can also show that 
\begin{align*}
    \|\phi(\{\hat{\lambda}_{k}\},\{\hat{q}_k\})-\phi(\{\lambda_{k}\},\{q_k\})\|_{2}=\op(1),
\end{align*}
as in Lemma \ref{lem:op11}. 
Therefore, from Lemma 19.24 in \cite{VaartA.W.vander1998As}, the term \cref{eq:term1_drl_sam1} is $\op(1)$, concluding the proof.
\end{proof}

\begin{proof}[Proof of \cref{thm:doublerobust_m1}]

{\blockedit
Without loss of generality, we consider the case $K=2$. 
}

We use the following doubly robust structure
\begin{align*}
    &\rE\left [\sum_{k=0}^{T}\lambda_{k}r_k-\{\lambda_{k} q_k-\lambda_{k-1}\rE_{\pi^{e}}(q_k|\mathcal{H}_{s_k}) \}\right] \\
    &=  \rE\{\rE_{\pi^{e}}(q_0|s_0)\} +\rE\left [\sum_{k=0}^{T}\lambda_{k}\{r_k- q_k+\rE_{\pi^{e}}(q_k|\mathcal{H}_{s_{k+1}}) \}\right]=\rho^{\pi^{e}}. 
\end{align*}

Then, as in the proof of \cref{thm:double_m1}, 
\begin{align*}
    &\sqrt{n}(\P_{n_1}\phi(\{\hat{\lambda}^{(1)}_{k}\},\{\hat{q}^{(1)}_k\})-\rho^{\pi^{e}})\\
    &= \sqrt{n/n_1}\bG_{n_1}[\phi(\{\hat{\lambda}^{(1)}_{k}\},\{\hat{q}^{(1)}_k\})-\phi(\{\lambda^{\dagger}_{k}\},\{q^{\dagger}_k\})]+\sqrt{n/n_1}\bG_{n_1}[\phi(\{\lambda^{\dagger}_{k}\},\{q^{\dagger}_k\}) ] \\
    &+\sqrt{n/n_1}(\rE[\phi(\{\hat{\lambda}^{(1)}_{k}\},\{\hat{q}^{(1)}_k\})|\{\hat{\lambda}^{(1)}_{k}\},\{\hat{q}^{(1)}_k\}]-\rE[\phi(\{\lambda^{\dagger}_{k}\},\{q^{\dagger}_k\}) ])  +\sqrt{n}(\rE[\phi(\{\lambda^{\dagger}_{k}\},\{q^{\dagger}_k\}) ]-\rho^{\pi^{e}} )\\
    &=\sqrt{n/n_1}\bG_{n_1}[\phi(\{\lambda^{\dagger}_{k}\},\{q^{\dagger}_k\}) ]+\sqrt{n}(\rE[\phi(\{\lambda^{\dagger}_{k}\},\{q^{\dagger}_k\}) ]-\rho^{\pi^{e}} )+\smallo_{p}(1). 
\end{align*}
Here, we used 
\begin{align*}
 \sqrt{n/n_1}\bG_{n_1}[\phi(\{\hat{\lambda}^{(1)}_{k}\},\{\hat{q}^{(1)}_k\})-\phi(\{\lambda^{\dagger}_{k}\},\{q^{\dagger}_k\})]=\op(1)   
\end{align*}
from \cref{lem:op11} and 
\begin{align*}
 \sqrt{n/n_1}(\rE[\phi(\{\hat{\lambda}^{(1)}_{k}\},\{\hat{q}^{(1)}_k\})|\{\hat{\lambda}^{(1)}_{k}\},\{\hat{q}^{(1)}_k\}]-\rE[\phi(\{\lambda^{\dagger}_{k}\},\{q^{\dagger}_k\}) ])=\op(1),
 \end{align*}
 which we prove below as in \cref{lem:op12}.
\begin{lemma}
\begin{align*}
 \sqrt{n/n_1}(\rE[\phi(\{\hat{\lambda}^{(1)}_{k}\},\{\hat{q}^{(1)}_k\})|\{\hat{\lambda}^{(1)}_{k}\},\{\hat{q}^{(1)}_k\}]-\rE[\phi(\{\lambda^{\dagger}_{k}\},\{q^{\dagger}_k\}) ])=\op(1).
 \end{align*}
\end{lemma}

\begin{proof}
First, consider the case where $\lambda_k=\lambda^{\dagger}_k$. 

\begin{align*}
&\sqrt{n}\rE[\phi(\{\hat{\lambda}^{(1)}_{k}\},\{\hat{q}^{(1)}_k\})-\rE[\phi(\{\lambda_{k}\},\{{q}^{\dagger}_k\})]|\{\hat{\lambda}^{(1)}_{k}\},\{\hat{q}^{(1)}_k\}]\\ 
&=\sqrt{n}\rE[\sum_{k=0}^{T}(\hat{\lambda}^{(1)}_{k}-\lambda_{k} )(-\hat{q}^{(1)}_k+q^{\dagger}_k)+(\hat{\lambda}^{(1)}_{k-1}-\lambda_{k-1})(-\hat{v}_k+v^{\dagger}_k)|\{\hat{\lambda}^{(1)}_{k}\},\{\hat{q}^{(1)}_k\}]\\
    &+\sqrt{n}\rE[\sum_{k=0}^{T}\lambda_{k}(\hat{q}^{(1)}_k-q^{\dagger}_k)+\lambda_{k-1}(\hat{v}_k-v^{\dagger}_k)|\{\hat{\lambda}^{(1)}_{k}\},\{\hat{q}^{(1)}_k\}] \\
    &+\sqrt{n}\rE[\sum_{k=0}^{T}(\hat{\lambda}^{(1)}_{k}-\lambda_{k} )(r_k-q^{\dagger}_k+v^{\dagger}_{k+1})|\{\hat{\lambda}^{(1)}_{k}\},\{\hat{q}^{(1)}_k\}] \\
&=\sqrt{n}\rE[\sum_{k=0}^{T}(\hat{\lambda}^{(1)}_{k}-\lambda_{k} )(-\hat{q}^{(1)}_k+q^{\dagger}_k)+(\hat{\lambda}^{(1)}_{k-1}-\lambda_{k-1})(-\hat{v}_k+v^{\dagger}_k)|\{\hat{\lambda}^{(1)}_{k}\},\{\hat{q}^{(1)}_k\}] \\
&+\sqrt{n}\rE[\sum_{k=0}^{T}(\hat{\lambda}^{(1)}_{k}-\lambda_{k} )(r_k-q^{\dagger}_k+v^{\dagger}_{k+1})|\{\hat{\lambda}^{(1)}_{k}\},\{\hat{q}^{(1)}_k\}] \\
&= \sqrt{n}\sum_{k=0}^{T} \bigO(\|\hat{\lambda}^{(1)}_{k}-\lambda_{k}\|_{2}\|\hat{q}^{(1)}_{k}-q^{\dagger}_{k}\|_{2}+\|\hat{\lambda}^{(1)}_{k}-\lambda_{k}\|_2) \\ &=\sqrt{n}\sum_{k=0}^{T}\left\{\bigO_{p}(n^{-\alpha_{1}})\bigO_{p}(n^{-\alpha_{2}})+\bigO_{p}(n^{-\alpha_{1}})\right\}=\bigO_{p}(1). 
\end{align*}

Next, consider the case where $q_k=q^{\dagger}_k$:
\begin{align*}
&\sqrt{n}\rE[\phi(\{\hat{\lambda}^{(1)}_{k}\},\{\hat{q}^{(1)}_k\})-\rE[\phi(\{\lambda^{\dagger}_{k}\},\{{q}_k\})]|\{\hat{\lambda}^{(1)}_{k}\},\{\hat{q}^{(1)}_k\}]\\ 
&=\sqrt{n}\rE[\sum_{k=0}^{T}(\hat{\lambda}^{(1)}_{k}-\lambda^{\dagger}_{k} )(-\hat{q}^{(1)}_k+q_k)+(\hat{\lambda}^{(1)}_{k-1}-\lambda^{\dagger}_{k-1})(-\hat{v}_k+v_k)|\{\hat{\lambda}^{(1)}_{k}\},\{\hat{q}^{(1)}_k\}]\\
    &+\sqrt{n}\rE[\sum_{k=0}^{T}\lambda^{\dagger}_{k}(\hat{q}^{(1)}_k-q_k)+\lambda^{\dagger}_{k-1}(\hat{v}_k-v_k)|\{\hat{\lambda}^{(1)}_{k}\},\{\hat{q}^{(1)}_k\}] \\
&= \sqrt{n}\sum_{k=0}^{T} \bigO(\|\hat{\lambda}^{(1)}_{k}-\lambda^{\dagger}_{k}\|_{2}\|\hat{q}^{(1)}_{k}-q_{k}\|_{2}+\|\hat{q}^{(1)}_{k}-q_{k}\|_2) \\ &=\sqrt{n}\sum_{k=0}^{T}\left\{\bigO_{p}(n^{-\alpha_{1}})\bigO_{p}(n^{-\alpha_{2}})+\bigO_{p}(n^{-\alpha_{2}})\right\}=\bigO_{p}(1). \qedhere
\end{align*}
\end{proof}

Using the above result, we prove the statement for each case below. 

{\blockedit
\paragraph*{$\lambda$-model is well-specified.}

First, consider the case when $\lambda^{\dagger}_{k}=\lambda_{k}$:
\begin{align*}
\rE[\phi(\{\lambda_{k}\},\{q^{\dagger}_k\}) ]&=\rE[\sum_{k=0}^{T}[\lambda_{k}r_k-\{\lambda_{k}q^{\dagger}_k(\mathcal{H}_{a_k})-\lambda_{k-1}\rE_{\pi^{e}}[q^{\dagger}_{k}(\mathcal{H}_{a_k})|s_k]] \\
&= \rE[\sum_{k=0}^{T}\lambda_{k}r_k]=\rho^{\pi^{e}}.
\end{align*}
Then,
\begin{align*}
    \sqrt{n}(\P_{n_1}\phi(\{\hat{\lambda}^{(1)}_{k}\},\{\hat{q}^{(1)}_k\})-\rho^{\pi^{e}})&=\sqrt{n/n_1}\bG_{n_1}[\phi(\{\lambda_{k}\},\{q^{\dagger}_k\}) ]+\bigO_{p}(1).
\end{align*}
Therefore, 
\begin{align*}
    \sqrt{n}(\hat{\rho}^{\pi^e}_{\mathrm{DRL}(\cm_1)}-\rho^{\pi^{e}})&=\sqrt{n_1/n}\bG_{n_1}[\phi(\{\lambda_{k}\},\{q^{\dagger}_k\}) ]+\sqrt{n_2/n}\bG_{n_2}[\phi(\{\lambda_{k}\},\{q^{\dagger}_k\}) ]+\bigO_{p}(1)\\
    &=\bG_{n}[\phi(\{\lambda_{k}\},\{q^{\dagger}_k\}) ] +\bigO_{p}(1)=\bigO_{p}(1),
\end{align*}
which shows $\hat{\rho}^{\pi^e}_{\mathrm{DRL}(\cm_1)}$ is $\sqrt{n}$-consistent around $\rho^{\pi^{e}}$ when the model for the behavior policy is well-specified. 

\paragraph*{$q$-model is well-specified.}

Next, consider the case where $q^{\dagger}_k=q_k$.
\begin{align*}
&\rE[\phi(\{\lambda^{\dagger}_{k}\},\{q_k\})]=\rE\left[\rE_{\pi^{e}}[q_{0}(\mathcal{H}_{a_0})|s_0] + \sum_{k=0}^{T}\lambda^{\dagger}_{k}\{r_k- q_k(\mathcal{H}_{a_k})+\rE_{\pi^{e}}[q_k(\mathcal{H}_{a_k})|s_{k+1}] \}\right]  \\
    &=\rE\left[\rE_{\pi^{e}}[q_{0}(\mathcal{H}_{a_0})|s_0]\right]= \rho^{\pi^{e}}.
\end{align*}
Then, 
\begin{align*}
    \sqrt{n}(\P_{n_1}\phi(\{\hat{\lambda}^{(1)}_{k}\},\{\hat{q}^{(1)}_k\})-\rho^{\pi^{e}})
    =\sqrt{n/n_1}\bG_{n_1}[\phi(\{\lambda^{\dagger}_{k}\},\{q_k\}) ]+\bigO_{p}(1). 
\end{align*}
Therefore, 
\begin{align*}
    \sqrt{n}(\hat{\rho}^{\pi^e}_{\mathrm{DRL}(\cm_1)}-\rho^{\pi^{e}})&=\sqrt{n_1/n}\bG_{n_1}[\phi(\{\lambda^{\dagger}_{k}\},\{q_k\}) ]+\sqrt{n_2/n}\bG_{n_2}[\phi(\{\lambda^{\dagger}_{k}\},\{q_k\}) ]+\bigO_{p}(1)\\
    &=\bG_{n}[\phi(\{\lambda^{\dagger}_{k}\},\{q_k\}) ] +\bigO_{p}(1)=\bigO_{p}(1).
\end{align*}
which shows $\hat{\rho}^{\pi^e}_{\mathrm{DRL}(\cm_1)}$ is $\sqrt{n}$-consistent around $\rho^{\pi^{e}}$ when the model for the $q$-function is well-specified. 
}
\end{proof}

\begin{proof}[Proof of \cref{cor:db1}]

{\blockedit
Without loss of generality, we consider the case $K=2$. 
}

{\blockedit
Then, as in the proof of \cref{thm:double_m1}, 
\begin{align*}
    &\P_{n_1}\phi(\{\hat{\lambda}^{(1)}_{k}\},\{\hat{q}^{(1)}_k\})-\rho^{\pi^{e}}\\
    &= \sqrt{1/n_1}\bG_{n_1}[\phi(\{\hat{\lambda}^{(1)}_{k}\},\{\hat{q}^{(1)}_k\})-\phi(\{\lambda^{\dagger}_{k}\},\{q^{\dagger}_k\})]+\sqrt{1/n_1}\bG_{n_1}[\phi(\{\lambda^{\dagger}_{k}\},\{q^{\dagger}_k\}) ] \\
    &+\sqrt{1/n_1}(\rE[\phi(\{\hat{\lambda}^{(1)}_{k}\},\{\hat{q}^{(1)}_k\})|\{\hat{\lambda}^{(1)}_{k}\},\{\hat{q}^{(1)}_k\}]-\rE[\phi(\{\lambda^{\dagger}_{k}\},\{q^{\dagger}_k\}) ])  +(\rE[\phi(\{\lambda^{\dagger}_{k}\},\{q^{\dagger}_k\}) ]-\rho^{\pi^{e}} )\\
    &=\sqrt{1/n_1}\bG_{n_1}[\phi(\{\lambda^{\dagger}_{k}\},\{q^{\dagger}_k\}) ]+(\rE[\phi(\{\lambda^{\dagger}_{k}\},\{q^{\dagger}_k\}) ]-\rho^{\pi^{e}} )+\smallo_{p}(1)\\
    &=(\P_{n_1}-\P)[\phi(\{\lambda^{\dagger}_{k}\},\{q^{\dagger}_k\}) ]+\op(1). 
\end{align*}
Here, under the assumption, we use  the following equations:
\begin{align*}
 \sqrt{1/n_1}\bG_{n_1}[\phi(\{\hat{\lambda}^{(1)}_{k}\},\{\hat{q}^{(1)}_k\})-\phi(\{\lambda^{\dagger}_{k}\},\{q^{\dagger}_k\})]  & =\op(1)   \\ 
(\rE[\phi(\{\hat{\lambda}^{(1)}_{k}\},\{\hat{q}^{(1)}_k\})|\{\hat{\lambda}^{(1)}_{k}\},\{\hat{q}^{(1)}_k\}]-\rE[\phi(\{\lambda^{\dagger}_{k}\},\{q^{\dagger}_k\}) ])&=\op(1) \\ 
(\rE[\phi(\{\lambda^{\dagger}_{k}\},\{q^{\dagger}_k\}) ]-\rho^{\pi^{e}} ) &= 0. 
 \end{align*}
 These equations are proved as in the proof of Theorem \ref{thm:doublerobust_m1}. 
Then, 
\begin{align*}
(\hat{\rho}^{\pi^e}_{\mathrm{DRL}(\cm_1)}-\rho^{\pi^{e}})&= (\P_{n_1}-\P)[\phi(\{\lambda^{\dagger}_{k}\},\{q^{\dagger}_k\}) ]+(\P_{n_2}-\P)[\phi(\{\lambda^{\dagger}_{k}\},\{q^{\dagger}_k\}) ]+\op(1)\\ 
&= (\P_{n}-\P)[\phi(\{\lambda^{\dagger}_{k}\},\{q^{\dagger}_k\})]+\op(1). 
\end{align*}
From the law of large numbers, the statement is concluded. 
}
 
\end{proof}

\begin{proof}[Proof of \cref{lem:denstiy_estimation}]
\edit{
Let any $\hat{\pi}^{b}_{t}$ be given. Due to boundedness away from zero, we have 
\begin{align*}
   & \left \|\prod_{0=t}^{k} \frac{\pi^{e}_{t}}{\hat{\pi}^{b}_{t}}-\lambda_{k}\right\|_{2}\leq\left \|\sum_{i=0}^{T}\left (\prod_{t=0}^{i} \frac{\pi^{e}_{t}}{\hat{\pi}^{b}_{t}}\prod_{t=i}^{k}\frac{\pi^{e}_{t}}{\pi^{b}_{t}} -\prod_{t=0}^{i-1} \frac{\pi^{e}_{t}}{\hat{\pi}^{b}_{t}}\prod_{t=i}^{k}\frac{\pi^{e}_{t}}{\pi^{b}_{t}} \right)\right\|_{2} \\
   & \leq\sum_{t=0}^{k}{\op(\| 1/\hat{\pi}^{b}_{t}-1/{\pi}^{b}_{t}\|_{2})}\\
   & \leq\sum_{t=0}^{k}{\op(\| \hat{\pi}^{b}_{t}-{\pi}^{b}_{t}\|_{2})}
   \\
   &\leq \op(n^{-\alpha}). 
\end{align*}}
\end{proof}

\begin{proof}[Proof of \cref{thm:double_m2}]

{\blockedit
Without loss of generality, we consider the case $K=2$. 
}

Define $\phi(\{\hat{\mu}_{k}\},\{\hat{q}_k\})$ as:
\begin{align*}
    \sum_{k=0}^{T}\hat{\mu}_{k}\{r_k-\hat{q}_k\}-\hat{\mu}_{k-1}\rE_{\pi^{e}}[\hat{q}_k(\ch_{a_k})|\ch_{s_k}]\}. 
\end{align*}

The estimator $\hat{\rho}_{\mathrm{DRL}(\cm_2)}$ is given by
\begin{align*}
\frac{n_1}{n}\P_{n_1}\phi(\{\hat{\mu}^{(1)}_{k}\},\{\hat{q}^{(1)}_k\})+\frac{n_2}{n}\P_{n_2}\phi(\{\hat{\mu}^{(2)}_{k}\},\{\hat{q}^{(2)}_k\}).     
\end{align*}
Then, we have 
\begin{align}
    \sqrt{n}(\P_{n_1}\phi(\{\hat{\mu}^{(1)}_{k}\},\{\hat{q}^{(1)}_k\})-\rho^{\pi^{e}})&= \sqrt{n/n_1}\bG_{n_1}[\phi(\{\hat{\mu}^{(1)}_{k}\},\{\hat{q}^{(1)}_k\})-\phi(\{\mu_{k}\},\{q_k\})]
    \label{eq:term1_drl_2}
    \\
    &+\sqrt{n/n_1}\bG_{n_1}[\phi(\{\mu_{k}\},\{q_k\}) ] 
       \label{eq:term1_dr2_2}
    \\
    &+\sqrt{n}(\rE[\phi(\{\hat{\mu}^{(1)}_{k}\},\{\hat{q}^{(1)}_k\})|\{\hat{\mu}^{(1)}_{k}\},\{\hat{q}^{(1)}_k\} ]-\rho^{\pi^{e}} ).    \label{eq:term1_dr3_2}
\end{align}
We analyze each term. To do that, we use the following relation;  {\blockedit 
\begin{align*}
&\phi(\{\hat{\mu}^{(1)}_{k}\},\{\hat{q}^{(1)}_k\})-\phi(\{\mu_{k}\},\{{q}_k\})=D_1+D_2+D_3,\quad\text{where}\\
&D_1 =\sum_{k=0}^{T}(\hat{\mu}^{(1)}_{k}-\mu_{k})(-\hat{q}^{(1)}_k+q_k)+(\hat{\mu}^{(1)}_{k-1}-\mu_{k-1})(\hat{v}_k-v_k), \\
&D_2 =\sum_{k=0}^{T}\mu_{k}(-\hat{q}^{(1)}_k+q_k)+\mu_{k-1}(\hat{v}^{(1)}_k-v_k), \\
&D_3 =\sum_{k=0}^{T}(\hat{\mu}^{(1)}_{k}-\mu_{k} )(r_k-q_k+v_{k+1}).
\end{align*}
}
First, we show the term \cref{eq:term1_drl_2} is $\smallo_{p}(1)$. 
\begin{lemma}\label{lem:op11_2}
The term \cref{eq:term1_drl_2} is $\smallo_{p}(1)$.
\end{lemma}
\begin{proof}
If we can show that for any $\epsilon>0$, 
\begin{align}
\label{eq:process_m1_2}
\lim_{n\to \infty}\sqrt{n_1}P[&\P_{n_1}[\phi(\{\hat{\mu}^{(1)}_{k}\},\{\hat{q}^{(1)}_k\})-\phi(\{\mu^{(1)}_{k}\},\{{q}^{(1)}_k\})]\\&\notag- 
 \rE[\phi(\{\hat{\mu}^{(1)}_{k}\},\{\hat{q}^{(1)}_k\})-\phi(\{\mu_{k}\},\{{q}_k\})|\{\hat{\mu}^{(1)}_{k}\},\{\hat{q}^{(1)}_k\} ]>\epsilon|\mathcal{D}_2]=0. %
\end{align}
Then, by bounded convergence theorem, we have 
\begin{align*}
  \lim_{n\to \infty}\sqrt{n_1}P[&\P_{n_1}[\phi(\{\hat{\mu}^{(1)}_{k}\},\{\hat{q}^{(1)}_k\})-
  \phi(\{\mu_{k}\},\{{q}_k\})]\\&-\rE[\phi(\{\hat{\mu}^{(1)}_{k}\},\{\hat{q}^{(1)}_k\})-\phi(\{\mu_{k}\},\{{q}_k\})|\{\hat{\mu}^{(1)}_{k}\},\{\hat{q}^{(1)}_k\}]>\epsilon]=0,
\end{align*}
yielding the statement. 

To show \cref{eq:process_m1_2}, we show that the conditional mean is $0$ and conditional variance is $\smallo_p(1)$. 
The conditional mean is 
\begin{align*}
&\rE[\P_{n_1}[\phi(\{\hat{\mu}^{(1)}_{k}\},\{\hat{q}^{(1)}_k\})-\phi(\{\mu_{k}\},\{{q}_k\})|\{\hat{\mu}^{(1)}_{k}\},\{\hat{q}^{(1)}_k\}]-\\&\qquad\qquad\qquad\qquad\qquad\P[\phi(\{\hat{\mu}^{(1)}_{k}\},\{\hat{q}^{(1)}_k\})-\phi(\{\mu_{k}\},\{{q}_k\})]|\mathcal{D}_2]=0.
\end{align*}
Here, we used a sample splitting construction, that is, $\hat{\mu}^{(1)}_{k}$ and $\hat{q}^{(1)}_{k}$ only depend on $\mathcal{D}_2$. The conditional variance is 
\begin{align*}
&\mathrm{var}[\sqrt{n_1}\P_{n_1}[\phi(\{\hat{\mu}^{(1)}_{k}\},\{\hat{q}^{(1)}_k\})-\phi(\{\mu_{k}\},\{{q}_k\})]|\mathcal{D}_2]\\
&= \qquad\qquad\rE[\rE[D^{2}_1+D^{2}_2+D^{2}_3+2D_{1}D_{2}+2D_{2}D_{3}+2D_{2}D_{3} |\{\hat{q}^{(1)}_k\},\{\mu^{(1)}_{k}\} ] |\mathcal{D}_2]\\ 
&=\qquad\qquad\smallo_p(1).
\end{align*}
Here, we used the convergence rate assumption and the relation $\|\hat{v}_{k}^{(1)}-v_{k}\|_{2}\leq\|\hat{q}^{(1)}_{k}-q_{k}\|_{2}$. 
Then, from Chebyshev's inequality;
\begin{align*}
&\sqrt{n_1}P[\P_{n_1}[\phi(\{\hat{\mu}^{(1)}_{k}\},\{\hat{q}^{(1)}_k\})-\phi(\{\mu_{k}\},\{{q}_k\})]-\rE[\phi(\{\hat{\mu}^{(1)}_{k}\},\{\hat{q}^{(1)}_k\})-\\
&\phi(\{\mu_{k}\},\{{q}_k\})|\{\hat{\mu}^{(1)}_{k}\},\{\hat{q}^{(1)}_k\} ]>\epsilon|\mathcal{D}_2] \\
&\leq \frac{1}{\epsilon^{2}}\mathrm{var}[\sqrt{n_1}\P_{n_1}[\phi(\{\hat{\mu}^{(1)}_{k}\},\{\hat{q}^{(1)}_k\})-\phi(\{\mu_{k}\},\{{q}_k\})]|\mathcal{D}_2]=\smallo_{p}(1).\qedhere
\end{align*}
\end{proof}

\begin{lemma}\label{lem:op12_2}
The term \cref{eq:term1_dr3_2} is $\smallo_{p}(1)$.
\end{lemma}
\begin{proof}
\begin{align*}
&\sqrt{n}\rE[\phi(\{\hat{\mu}^{(1)}_{k}\},\{\hat{q}^{(1)}_k\})-\rE[\phi(\{\mu_{k}\},\{{q}_k\})]|\{\hat{\mu}^{(1)}_{k}\},\{\hat{q}^{(1)}_k\}]\\ 
&=\sqrt{n}\rE[\sum_{k=0}^{T}(\hat{\mu}^{(1)}_{k}-\mu_{k} )(-\hat{q}^{(1)}_k+q_k)+(\hat{\mu}^{(1)}_{k-1}-\mu_{k-1})(\hat{v}^{(1)}_k-v_k)|\{\hat{\mu}^{(1)}_{k}\},\{\hat{q}^{(1)}_k\}]\\
    &+\sqrt{n}\rE[\sum_{k=0}^{T}\mu_{k}(-\hat{q}^{(1)}_k+q_k)+\mu_{k-1}(\hat{v}^{(1)}_k-v_k)|\{\hat{\mu}^{(1)}_{k}\},\{\hat{q}^{(1)}_k\}] \\
    &+\sqrt{n}\rE[\sum_{k=0}^{T}(\hat{\mu}^{(1)}_{k}-\mu_{k} )(r_k-q_k+v_{k+1})|\{\hat{\mu}^{(1)}_{k}\},\{\hat{q}^{(1)}_k\}] \\
&=\sqrt{n}\rE[\sum_{k=0}^{T}(\hat{\mu}^{(1)}_{k}-\mu_{k} )(-\hat{q}^{(1)}_k+q_k)+(\hat{\mu}^{(1)}_{k-1}-\mu_{k-1})(-\hat{v}^{(1)}_k+v_k)|\{\hat{\mu}^{(1)}_{k}\},\{\hat{q}^{(1)}_k\}] \\
&= \sqrt{n}\sum_{k=0}^{T} \bigO(\|\hat{\mu}^{(1)}_{k}-\mu_{k}\|_{2}\|\hat{q}^{(1)}_{k}-q_{k}\|_{2}) =\sqrt{n}\sum_{t=0}^{T}\smallo_{p}(n^{-1/2})=\smallo_{p}(1). 
\qedhere\end{align*}
\end{proof}

Finally, we get 
\begin{align*}
    \sqrt{n}(\P_{n_1}\phi(\{\hat{\mu}^{(1)}_{k}\},\{\hat{q}^{(1)}_k\})-\rho^{\pi^{e}})=\sqrt{n/n_1}\bG_{n_1}[\phi(\{\mu_{k}\},\{q_k\}) ]+\smallo_{p}(1).
\end{align*}
Therefore, 
\begin{align*}
    &\sqrt{n}(\hat{\rho}^{\pi^e}_{\mathrm{DRL}(\cm_2)}-\rho^{\pi^{e}})\\
    &=  n_1/n\sqrt{n}\phi(\{\hat{\mu}^{(1)}_{k}\},\{\hat{q}^{(1)}_k\})-\rho^{\pi^{e}})+ n_2/n\sqrt{n}(\P_{n_2}\phi(\{\hat{\mu}^{(2)}_{k}\},\{\hat{q}^{(2)}_k\})-\rho^{\pi^{e}}) \\
    &=\sqrt{n_1/n}\bG_{n_1}[\phi(\{\mu_{k}\},\{q_k\})  ]+\sqrt{n_2/n}\bG_{n_2}[\phi(\{\mu_{k}\},\{q_k\}) ]+\smallo_{p}(1)\\
    &=\bG_{n}[\phi(\{\mu_{k}\},\{q_k\}) ] +\smallo_{p}(1),
\end{align*}
concluding the proof by showing the influence function of $\hat{\rho}^{\pi^e}_{\mathrm{DRL}(\cm_2)}$ is the efficient one.
\end{proof}

\begin{proof}[Proof of \cref{thm:effcor_m2}]
\edit{Almost same as the proof of \cref{thm:effcor_m1}. The differences are replacing $\lambda_t$ with $\mu_t$, and $C^{T+1}$ with $C'$}
\end{proof}

\begin{proof}[Proof of \cref{thm:double_m2_sam}]

We define $\phi(\{\hat{\mu}_{k}\},\{\hat{q}_k\})$ as:
\begin{align*}
    \sum_{k=0}^{T}\hat{\mu}_{k}\{r_k-\hat{q}_k\}-\hat{\mu}_{k-1}\rE_{\pi^{e}}[\hat{q}_k(\ch_{a_k})|\ch_{s_k}]\}. 
\end{align*}
The estimator $\hat\rho^{\pi^{e}}_{\mathrm{DRL}(\cm_2),\,\mathrm{adaptive}}$ is given by $
\P_{n}\phi(\{\hat{\mu}_{k}\},\{\hat{q}_k\})$.     
Then, we have 
\begin{align}
    \sqrt{n}(\P_{n}\phi(\{\hat{\mu}_{k}\},\{\hat{q}_k\})-\rho^{\pi^{e}})&= \bG_{n}[\phi(\{\hat{\mu}_{k}\},\{\hat{q}_k\})-\phi(\{\mu_{k}\},\{q_k\})]
    \label{eq:term1_drl_sam1_m2}
    \\
    &+\bG_{n}[\phi(\{\mu_{k}\},\{q_k\}) ] 
       \label{eq:term1_dr1_sam2_m2}
    \\
    &+\sqrt{n}(\rE[\phi(\{\hat{\mu}_{k}\},\{\hat{q}_k\})|\{\hat{\mu}_{k}\},\{\hat{q}_k\} ]-\rho^{\pi^{e}} ).    \label{eq:term1_dr1_sam3_m2}
\end{align}
If we can prove that the term \cref{eq:term1_drl_sam1_m2} is $\op(1)$, the statement is concluded as in the proof of Theorem \ref{thm:double_m2}. We proceed to prove this.

First, we show $\phi(\{\hat{\mu}_{k}\},\{\hat{q}_t\})-\phi(\{\mu_{t}\},\{q_t\})$ belongs to a Donsker class. The transformation
\begin{align*}
   (\{\mu_{k}\},\{q_k\} ) \mapsto    \sum_{k=0}^{T}\mu_{k}r_k-\{ \mu_{k}q_k-\mu_{k-1}\rE_{\pi^{e}}[q_k(\ch_{a_k})|\ch_{s_k}]\}
\end{align*}
is a Lipschitz function. Therefore, by Example 19.20 in \citet{VaartA.W.vander1998As}, $\phi(\{\hat{\mu}_{k}\},\{\hat{q}_k\})-\phi(\{\mu_{k}\},\{q_k\})$ is an also Donsker class. In addition, we can also show that 
\begin{align*}
    \|\phi(\{\hat{\mu}_{k}\},\{\hat{q}_k\})-\phi(\{\mu_{k}\},\{q_k\})\|_{2}=\op(1),
\end{align*}
as in Lemma \ref{lem:op11}. 
Therefore, from Lemma 19.24 in \cite{VaartA.W.vander1998As}, the term \cref{eq:term1_drl_sam1_m2} is $\op(1)$, concluding the proof.
\end{proof}

\begin{proof}[Proof of \cref{thm:doublerobust_m2}]

{\blockedit
Without loss of generality, we consider the case $K=2$. 
}

We use the following doubly robust structure
\begin{align*}
    &\rE\left [\sum_{k=0}^{T}\mu_{k}r_k-\{\mu_{k} q_k-\mu_{k-1}\rE_{\pi^{e}}(q_k|s_k) \}\right] \\
    &=  \rE[\rE_{\pi^{e}}(q_0|s_0)]] +\rE\left [\sum_{k=0}^{T}\mu_{k}\{r_k- q_k+\rE_{\pi^{e}}(q_k|s_{k+1}) \}\right]=\rho^{\pi^{e}}. 
\end{align*}

Then, as in the proof of \cref{thm:double_m1}, 
\begin{align*}
    &\sqrt{n}(\P_{n_1}\phi(\{\hat{\mu}^{(1)}_{k}\},\{\hat{q}^{(1)}_k\})-\rho^{\pi^{e}})\\
    &= \sqrt{n/n_1}\bG_{n_1}[\phi(\{\hat{\mu}^{(1)}_{k}\},\{\hat{q}^{(1)}_k\})-\phi(\{\mu^{\dagger}_{k}\},\{q^{\dagger}_k\})]+\sqrt{n/n_1}\bG_{n_1}[\phi(\{\mu^{\dagger}_{k}\},\{q^{\dagger}_k\}) ] \\
    &\phantom{=}+\sqrt{n/n_1}(\rE[\phi(\{\hat{\mu}^{(1)}_{k}\},\{\hat{q}^{(1)}_k\}) ;\{\hat{\mu}^{(1)}_{k}\},\{\hat{q}^{(1)}_k\}]\\&\phantom{=}-\rE[\phi(\{\mu^{\dagger}_{k}\},\{q^{\dagger}_k\})] )+\sqrt{n}(\rE[\phi(\{\mu^{\dagger}_{k}\},\{q^{\dagger}_k\}) ]-\rho^{\pi^{e}} )\\
    &=\sqrt{n/n_1}\bG_{n_1}[\phi(\{\mu^{\dagger}_{k}\},\{q^{\dagger}_k\}) ]+\sqrt{n}(\rE[\phi(\{\mu^{\dagger}_{k}\},\{q^{\dagger}_k\}) ]-\rho^{\pi^{e}} )+\bigO_{p}(1). 
\end{align*}
We proceed by considering each case.

{\blockedit
\paragraph*{$\mu$-model is well-specified.}

First, consider the case when $\mu^{\dagger}_{k}=\mu_{k}$:
\begin{align*}
&\rE[\phi(\{\mu_{k}\},\{q^{\dagger}_k\}) ]=\rE[\sum_{k=0}^{T}[\mu_{k}r_k-\{\mu_{k}q^{\dagger}_k(s_k,a_k)-\mu_{k-1}\rE_{\pi^{e}}[q^{\dagger}_{k}(s_k,a_k)|s_k]] \\
&= \rE[\sum_{k=0}^{T}\mu_{k}r_k]=\rho^{\pi^{e}}.
\end{align*}
Then, 
\begin{align*}
    \sqrt{n}(\P_{n_1}\phi(\{\hat{\mu}^{(1)}_{k}\},\{\hat{q}^{(1)}_k\})-\rho^{\pi^{e}})=\sqrt{n/n_1}\bG_{n_1}[\phi(\{\mu_{k}\},\{q^{\dagger}_k\}) ]+\bigO_{p}(1).
\end{align*}
Therefore, 
\begin{align*}
    \sqrt{n}(\hat{\rho}^{\pi^e}_{\mathrm{DRL}(\cm_2)}-\rho^{\pi^{e}})&=\sqrt{n_1/n}\bG_{n_1}[\phi(\{\mu_{k}\},\{q^{\dagger}_k\}) ]+\sqrt{n_2/n}\bG_{n_2}[\phi(\{\mu_{k}\},\{q^{\dagger}_k\}) ]+\bigO_{p}(1)\\
    &=\bG_{n}[\phi(\{\mu_{k}\},\{q^{\dagger}_k\}) ] +\bigO_{p}(1)=\bigO_{p}(1),
\end{align*}
which shows $\hat{\rho}^{\pi^e}_{\mathrm{DRL}(\cm_2)}$ is $\sqrt{n}$-consistent around $\rho^{\pi^{e}}$ when the model for the $\mu$-function is well-specified. 

\paragraph*{$q$-model is well-specified.}

Next, consider the case where $q^{\dagger}_k=q_k$: 
\begin{align*}
&\rE[\phi(\{\mu^{\dagger}_{k}\},\{q_k\})]=\rE\left[\rE_{\pi^{e}}[q(s_k,a_k)|s_0] + \sum_{k=0}^{T}\mu^{\dagger}_{k}\{r_k- q_k(s_k,a_k)+\rE_{\pi^{e}}[q_k(s_k,a_k)|{s_{k+1}}] \}\right]  \\
    &=\rE\left[\rE_{\pi^{e}}[q_{0}(s_0,a_0)|s_0]\right]= \rho^{\pi^{e}}.
\end{align*}
We have 
\begin{align*}
    \sqrt{n}(\P_{n_1}\phi(\{\hat{\mu}^{(1)}_{k}\},\{\hat{q}^{(1)}_k\})-\rho^{\pi^{e}})
    =\sqrt{n/n_1}\bG_{n_1}[\phi(\{\mu^{\dagger}_{k}\},\{q_k\}) ]+\bigO_{p}(1). 
\end{align*}
Therefore, 
\begin{align*}
    \sqrt{n}(\hat{\rho}^{\pi^e}_{\mathrm{DRL}(\cm_2)}-\rho^{\pi^{e}})&=\sqrt{n/n_1}\bG_{n_1}[\phi(\{\mu^{\dagger}_{k}\},\{q_k\}) ]+\sqrt{n/n_2}\bG_{n_2}[\phi(\{\mu^{\dagger}_{k}\},\{q_k\}) ]+\bigO_{p}(1)\\
    &=\bG_{n}[\phi(\{\mu^{\dagger}_{k}\},\{q_k\}) ] +\bigO_{p}(1)=\bigO_{p}(1),
\end{align*}
which shows $\hat{\rho}^{\pi^e}_{\mathrm{DRL}(\cm_2)}$ is $\sqrt{n}$-consistent around $\rho^{\pi^{e}}$ when the model for the $q$-function is well-specified. 
}
\end{proof}

\begin{proof}[Proof of \cref{cor:db2}]
\edit{Almost the same as the proof of \cref{cor:db1}}
\end{proof}

\begin{proof}[Proof of Theorem \ref{thm:ipw_m2_finite}]

We first prove  
\begin{align}\label{eq:mis_finite_infl}
    \mathbb{P}_n[\hat{w}_{t}(s_t)\eta_{t}r_t]= \mathbb{P}_n[w_{t}(s_t)\eta_{t}r_t]+\mathbb{P}_n[(\lambda_{t-1}-w_t(s_t))\rE_{\pi_e}[r_t|s_t]]+\op(n^{-1/2}). 
\end{align}

Noting 
\begin{align*}
   \left \| \frac{1}{n}\sum_{i=1}^n\mathbb I\bracks{s_t^{(i)}=s_t}\lambda_{t-1}- p_{\pi^{e}_t}(s_t) \right \|_{\infty} &=\op(n^{-1/4}), \\
   \left \| \frac{1}{n}\sum_{i=1}^n\mathbb I\bracks{s_t^{(i)}=s_t}- p_{\pi^{b}_t}(s_t) \right \|_{\infty} &=\op(n^{-1/4}), \\
   \hat{a}/\hat{b}=b^{-1}\{1-\hat{b}^{-1}(\hat{b}-b)\}\{(\hat{a}-a)-a/b(\hat{b}-b) \},
\end{align*}
we have 
\begin{align*}
    &\hat{w}_t(s_t)-w_t(s_t)+\op(n^{-1/2})  \\
    &=\frac{1}{p_{\pi^{b}_t}(s_t)}\left[\left \{\frac{1}{n}\sum_{i=1}^n\mathbb I\bracks{s_t^{(i)}=s_t}\lambda_{t-1}- p_{\pi^{e}_t}(s_t) \right\} - w_t(s_t)\left \{ \frac{1}{n}\sum_{i=1}^n\mathbb I\bracks{s_t^{(i)}=s_t}-p_{\pi^{b}_t}(s_t)\right \}  \right]. 
\end{align*}
Then, 
\begin{align*}
    \mathbb{P}_n[\hat{w}_{t}(s_t)\eta_{t}r_t] =& \frac{1}{n}\sum_{i=1}^{n}w_{t}(s^{(i)}_{t})\eta^{(i)}_{t}r^{(i)}_t + \\
    +& \frac{1}{n^{2}}\sum_{i=1}^{n}\frac{\eta^{(i)}_{t}r^{(i)}_t }{p_{\pi^{b}_t}(s^{(i)}_t)}\left  \{\sum_{j=1}^n\mathbb I\bracks{s_t^{(j)}=s^{(i)}_t}\lambda^{(i)}_{t-1}-w_{t}(s^{(i)}_t)\sum_{j=1}^n\mathbb I\bracks{s_t^{(j)}=s^{(i)}_t}  \right\} \\
    =& \frac{2}{n(n-1)}\sum_{i<j}0.5(a_{ij}+a_{ji})+\op(n^{-1/2}),
\end{align*}
where 
\begin{align*}
a_{ij}=\frac{\eta^{(i)}_{t}r^{(i)}_t \mathbb I\bracks{s_t^{(j)}=s^{(i)}_t}}{ p_{\pi^{b}_t}(s^{(i)}_t) }\left( \lambda^{(j)}_{t-1}-w_{t}(s^{(i)}_t)\right). 
\end{align*}
From $U$-statistics theory, by defining $b_{ij}(\ch^{(i)},\ch^{(j)})=0.5(a_{ij}+a_{ji})$, we have 
\begin{align*}
    \frac{2}{n(n-1)}\sum_{i<j}b_{ij}(\ch^{(i)},\ch^{(j)})=\frac{2}{n}\sum_{i=1}^{n}\rE[ b_{ij}(\ch^{(i)},\ch^{(j)})| \ch^{(i)}]+\op(n^{-1/2}). 
\end{align*}
In addition, 
\begin{align*}
    \rE[a_{i,j}|\ch^{(i)}]&= \eta^{(i)}_{t}r^{(i)}_t \{w_{t}(s^{(i)}_t)-w_{t}(s^{(i)}_t) \}=0 , \\
     \rE[a_{j,i}|\ch^{(i)}] &= \rE \left[\frac{\eta^{(j)}_{t}r^{(j)}_t \mathbb I\bracks{s_t^{(j)}=s^{(i)}_t}}{ p_{\pi^{b}_t}(s^{(j)}_t) }\left( \lambda^{(i)}_{t-1}-w_{t}(s^{(j)}_t)\right) |\ch^{(i)}\right ]  \\
    &= (\lambda^{(i)}_{t-1}-w^{(i)}_{t})\rE[\eta^{(i)}_{t}r^{(i)}_t | s^{(i)}_{t}]
\end{align*}
Therefore, we have shown \cref{eq:mis_finite_infl}.
Summing over $t$ yields
\begin{align*}
   \mathbb{P}_{n}\bracks{\sum_{t=0}^{T}\hat{w}_t(s_t)\eta_{t}r_t}=\mathbb{P}_{n}\bracks{\sum_{t=0}^{T}\left \{w_t(s_t)\eta_{t}r_t+\lambda_{t-1}\rE_{\pi_e}[r_t|s_t]-w_{t} \rE_{\pi_e}[r_t|s_t]\right \} }+\op(n^{-1/2}),
\end{align*}
which concludes the proof by establishing the influence function for $\hat{\rho}_{\mathrm{MIS}}$.
\end{proof}

\begin{proof}[Proof of Theorem \ref{thm:ipw_inefficient}]

The difference of the influence functions belongs to the orthogonal tangent space. Therefore, the difference of variances is the variance of the difference of the influence functions. 
This is equal to 
\begin{align*}
& \mathrm{var}\left[v_{0}+\sum_{t=0}^{T}-\mu_t q_t+\mu_{t}v_{t+1}-\{\lambda_{t-1}-w_{t}(s_t) \}\rE_{\pi_e}[r_t|s_t]\right] \\ 
& =\mathrm{var}[v_{0}]+\sum^{T+1}_{t=1}\rE \left[ \mathrm{var}\left(\rE\left[\sum_{k=0}^{T}-\mu_k q_k+\mu_{k}v_{k+1}-\{\lambda_{k-1}-w_{k}(s_k) \}\rE_{\pi_e}[r_k|s_k]|\mathcal{J}_{s_t}\right]|\mathcal{J}_{s_{t-1}}  \right)\right]\\ 
& =\mathrm{var}[v_{0}]+\sum^{T+1}_{t=1}\rE \left[ \mathrm{var}\left(\rE\left[\mu_{t-1}v_t+\sum_{k=t}^{T}-\mu_{k}r_k-\{\lambda_{k-1}-w_{k}(s_k) \}\rE_{\pi_e}[r_k|s_k]|\mathcal{J}_{s_t}\right]|\mathcal{J}_{s_{t-1}}  \right)\right]\\ 
& =\mathrm{var}[v_{0}]+\sum^{T+1}_{t=1}\rE \left[ \mathrm{var}\left(\rE\left[\mu_{t-1}v_t-\sum_{k=t}^{T}\lambda_{k-1} \rE_{\pi_e}[r_k|s_k]|\mathcal{J}_{s_t}\right]|\mathcal{J}_{s_{t-1}}  \right)\right]\\ & =\mathrm{var}[v_{0}]+\sum^{T+1}_{t=1}\rE \left[ \mathrm{var}\left(\rE\left[\mu_{t-1}v_t-\lambda_{t-1} v_t(s_t)|\mathcal{J}_{s_t}\right]|\mathcal{J}_{s_{t-1}}  \right)\right]\\
& =\mathrm{var}[v_{0}]+\sum^{T+1}_{t=1}\rE \left[ \mathrm{var}\left(\{\mu_{t-1}-\lambda_{t-1}\} v_t(s_t)|\mathcal{J}_{s_{t-1}} 
\right)\right] \\
&=\mathrm{var}[v_{0}]+\sum^{T}_{t=1}\rE \left[\{w_{t-1}-\lambda_{t-1}\}^{2} \mathrm{var}\left(\eta_{t-1}v_t(s_t)|s_{t-1}\right)\right].\qedhere
\end{align*}
\end{proof}

\begin{proof}[Proof of \cref{thm:ipw_m2}]

We have 
\begin{align*}
&\sqrt{n}\{\P[\sum_{c=0}^{T}\hat{w}_{c}\eta_c(s_c,a_c)r_c]-\rho^{\pi^{e}}\} \\
&=\bG_{n}[\sum_{c=0}^{T}\hat{w}_{c}\eta_c r_c-\sum_{c=0}^{T}w_{c}\eta_cr_c ] 
+ \bG_{n}[\sum_{c=0}^{T}w_{c}\eta_c r_c ]\\
&+ \sqrt{n}\left\{\rE\{\sum_{c=0}^{T}\hat{w}_{c}\eta_c r_c|\{\hat{\mu}_{c}\}\}-\rho^{\pi^{e}}\right\}\\
&=\bG_{n}[\sum_{c=0}^{T}w_{c}\eta_c r_c ]+ \sqrt{n}\left\{\rE\{\sum_{c=0}^{T}\hat{w}_{c}\eta_c r_c|\{\hat{\mu}_{c}\}\}-\rho^{\pi^{e}}\right\}+\smallo_{p}(1)\\
&=\smallo_{p}(1)+\bG_{n}\{\sum_{c=0}^{T}w_{c}\eta_c r_c \}+\bG_{n}[\sum_{c=0}^{T}\{\lambda_{c-1}-w_c(s_c)\}\rE_{\pi_e}[r_c|s_c] ].
\end{align*}
From the second line to the third line, we used a fact that the stochastic equicontinuity term is $\smallo_{p}(1)$ as in the proof of \cref{thm:double_m1_sam}  because $\{\hat{w}_c\eta_c r_c\}$ belongs to a Donsker class and the convergence rate condition holds. This fact is confirmed by the fact that a \Holder\, class with $\alpha>d_{\ch_{s_c}}/2$ is a Donsker class (Example 19.9 in \citealp{VaartA.W.vander1998As}). 

From the third line to the fourth line, we have used a result of Example 1(a) in Section 8 of \citet{ShenXiaotong1997OMoS}. More specifically, the functional derivative of the loss function with respect to $\hat{w}_c$ is 
\begin{align*}
    g(s_c) \to \{w_c(s_c)-\lambda_{c-1}\}g(s_c), 
\end{align*}
and the induced metric from the loss function is $L^{2}$-metric with respect to $p_{\pi^{b}}(s_c)$. The functional derivative of the target function with respect to $\hat{\mu}_c$ is 
\begin{align*}
    g(s_c)\to \rE[g(s_c)\eta_{c}r_c]=\rE[g(s_c) \rE_{\pi_e}[r_c|s_c]],
\end{align*}
and Riesz representation of the Hilbert space with the induced $L^{2}$-metric with respect to $p_{\pi^{b}}(s_c)$ is $\rE_{\pi_e}[r_c|s_c]$. Therefore, from Theorem 1 in \cite{ShenXiaotong1997OMoS}, 
\[
    \rE[\sum_{c=0}^{T}\hat{w}_{c}\eta_c r_c|\{\hat{\mu}_{c}\}]- \rho^{\pi^{e}}=(\P_{n}-\P)\sum_{c=0}^{T}(\lambda_{c-1}-w_{c}(s_c))\rE_{\pi_e}[r_c|s_c]+\smallo_{p}(n^{-1/2}). \qedhere
\]
\end{proof}

\begin{proof}[Proof of \cref{thm:ipw_m2_2}]

The proof is done as in the proof of \cref{thm:ipw_m2}. We study the following drift term:
\begin{align*}
    \rE\bracks{\sum_{c=0}^{T}\hat{\mu}_c r_c \mid \hat{\mu}_c}. 
\end{align*}
Here, the functional derivative of the loss function with respect to $\hat{\mu}_c(s_c,a_c)$ is 
\begin{align*}
    g(s_c,a_c) \to \rE[g(s_c,a_c)(\mu_c-\lambda_c) ]. 
\end{align*}
and Riesz representation of the Hilbert space with the induced $L^{2}$-metric with respect to $p_{\pi_{b}}(s_c,a_c)$ is $\rE[r_c|s_c,a_c]$. On the other hand, the functional derivative of the target function with respect to $\hat{\mu}_c$ is
\begin{align*}
    g(s_c,a_c) \to \rE[g(s_c,a_c) r_c]=  \rE[g(s_c,a_c) \rE[r_c|s_c,a_c]],
\end{align*} 
From Theorem 1 in \cite{ShenXiaotong1997OMoS}, 
\[
    \rE\bracks{\sum_{c=0}^{T}\hat{\mu}_c r_c \mid \hat{\mu}_c }=\P_n \bracks{\sum_{c=0}^{T} ( \lambda_c-\mu_c)\rE[r_c|s_c,a_c] }+\op(n^{-1/2}). \qedhere
\]

\end{proof}

\begin{proof}[Proof of \cref{thm:bound_mq_rl}]
We use the general framework developed in \citet{ChamberlainGary1992CSMR} for establishing the efficiency bounds. For the current problem, noting that the orthogonal moment condition
\begin{align*}
    \rE[e_{q,k+i}e_{q,k}]=\rE[\rE[e_{q,k+i}|\ch_{a_{k+i}}]e_{q,k}]]=0,\,(0\leq k <k+i\leq T)
\end{align*}
holds, the efficiency bound for $\beta$ is represented as 
\begin{align*}
    \left\{\sum_{k=0}^{T}\nabla_{\beta}m_k(\ch_{a_k};\beta^*)\Sigma^{-1}_k(\ch_{a_k}) \nabla_{\beta}^{\top} m_k(\ch_{a_k};\beta^*)\right\}^{-1},
\end{align*}
where $\Sigma_k(\ch_{a_k})=\mathrm{var}[e_{q,k}|\ch_{a_k}]$.
The statement of the theorem for the efficiency bound for $\beta$ is arrived at by algebraic simplification of the above. The efficiency bound of $\rho^{\pi^{e}}$ is calculated similarly. 
\end{proof}

\begin{proof}[Proof of \cref{lem:distance}]
Note $\mathrm{var}[e_{q,k}|\ch_{a_k}]$ and $\mathrm{var}[e_{q,k-1}|\ch_{a_{k-1}}]$ are upper and lower bounded by some constants by assumption. From Jensen's inequality, we also know 
\begin{align*}
    \rE[g^{2}]=\rE[\rE[g^{2}|\ch_{s_k}]]\geq \rE[\rE[g|\ch_{s_k}]^{2}]. 
\end{align*}
This conludes that there exists some constant $C_1,C_2$ such that 
\[
    C_1\|g \|_{2}\leq \|g\|_{F,k} \leq C_2\|g \|_{2}. \qedhere
\]
\end{proof}

\section{Additional Details from \cref{sec:exp-aigym}}\label{sec:experiment-ape}

\paragraph*{Cliff Walking.}

This RL task is detailed in Example  6.6 in \cite{SuttonRichardS1998Rl:a}. We consider a board of size $4 \times 12$. 
The horizon was set to $T=400$.
Each time step incurs $-1$ reward until the goal is reached, at which point it is $0$, and stepping off the cliff incurs $-100$ reward and a reset to the start. 

\paragraph*{Mountain Car.}

The RL task is as follows: a car is between two hills in the interval $[-0.7,0.5]$ and the agent must move back and forth to gain enough power to reach the top of the right hill. The state space comprises position and velocity. There are three discrete actions: (1) forward, (2) backward, and (3) stay-still. The horizon was set to $T=200$. The reward for each step is $-1$ until the position $0.5$ is reached, at which point it is $0$.
The state space was continuous; thus, we obtained a $400$-dimensional feature expansion using a radial basis function kernel as mentioned. 

\paragraph*{The Policy $\pi_d$.}

We construct the policy $\pi_d$ using standard $q$-learning \citep{SuttonRichardS1998Rl:a}. For Cliff Walking, we use a $q$-learning in a tabular manner. Regarding a Mountain Car, we use $q$-learning based on the same feature expansion as above. We use $4000$ sample to learn an optimal policy. 

\end{document}